\documentclass[conference]{IEEEtran}
\let\labelindent\relax
\DeclareUnicodeCharacter{0301}{\hspace{-1ex}\'{ }}
\usepackage{hyperref}
\usepackage{tikz}
\usepackage{amsmath}
\usepackage{filecontents}
\usepackage{footnote}
\usepackage{listings}

\usepackage{subfigure}
\usepackage{hyperref}       % hyperlinks
\usepackage{url}            % simple URL typesetting
\usepackage{booktabs}       % professional-quality tables
\usepackage{microtype}      % microtypography
\usepackage{graphicx}%
\usepackage{subcaption}
\usepackage{float}
\usepackage{soul}
\usepackage{caption}
\usepackage{multirow}
\usepackage{bbm}
\usepackage{bm}
\usepackage{mathrsfs}
\usepackage{xcolor}
\usepackage{xspace}
\usepackage{tikz}
\usepackage[framemethod=tikz]{mdframed}

\usepackage{thmtools,thm-restate}
\usepackage{dsfont}
\usepackage[ruled,vlined]{algorithm2e}
\usepackage{soul}
\usepackage{algpseudocode}  
\usepackage{pifont}
\usepackage[flushleft]{threeparttable}
\usepackage{xcolor}
\usepackage{makecell}
\usepackage{lipsum}
\usepackage{balance}
\usepackage{caption} 
\usepackage{microtype}
\usepackage{graphicx}
\usepackage{subfigure}
\usepackage{booktabs} % for professional tables
\usepackage{hyperref}
\usepackage{amsmath}
\usepackage{amssymb}
\usepackage{mathtools}
\usepackage{amsthm}
\usepackage{amsfonts}       % blackboard math symbols
\usepackage{thmtools,thm-restate}
\usepackage{dsfont}
\usepackage{tabularx}
\usepackage{enumitem}
\setlist{noitemsep}

\usepackage{url}

\renewcommand{\tabcolsep}{5pt}

\definecolor{firebrick}{rgb}{0.7, 0.13, 0.13}
\definecolor{darkblue}{rgb}{0,0,0.55}
\definecolor{grey}{rgb}{0.8,0.8,0.8}



\def\eg{\emph{e.g.,}\xspace}

\def\ie{\emph{i.e.,}\xspace}

\def\eqref#1{equation~\ref{#1}}
\def\1{\bm{1}}

\DeclareMathAlphabet{\mathsfit}{\encodingdefault}{\sfdefault}{m}{sl}
\SetMathAlphabet{\mathsfit}{bold}{\encodingdefault}{\sfdefault}{bx}{n}
\newcommand{\tens}[1]{\bm{\mathsfit{#1}}}

\def\tL{{\tens{L}}}

\def\gD{{\mathcal{D}}}

\def\gL{{\mathcal{L}}}

\def\gX{{\mathcal{X}}}
\def\gY{{\mathcal{Y}}}

\def\sD{{\mathbb{D}}}
\newcommand{\E}{\mathbb{E}}

\newcommand{\R}{\mathbb{R}}

\newcommand{\N}{\mathcal{N}}

\DeclareMathOperator*{\argmax}{arg\,max}

\DeclareMathOperator{\sign}{sign}

\def\eg{\emph{e.g.,}\xspace}

\def\ie{\emph{i.e.,}\xspace}

\graphicspath{{./Figures/}}

\begin{document}
\title{Provably Unlearnable Data Examples}

\author{
\IEEEauthorblockN{
Derui Wang\IEEEauthorrefmark{1}\IEEEauthorrefmark{2},
Minhui Xue\IEEEauthorrefmark{1}\IEEEauthorrefmark{2},
Bo Li\IEEEauthorrefmark{3}, 
Seyit Camtepe\IEEEauthorrefmark{1}\IEEEauthorrefmark{2} and
Liming Zhu\IEEEauthorrefmark{1}
}
\IEEEauthorblockA{
\IEEEauthorrefmark{1}CSIRO's Data61, Australia\\
\IEEEauthorrefmark{2}Cyber Security Cooperative Research Centre, Australia\\
\IEEEauthorrefmark{3}University of Chicago, USA}
}

\IEEEoverridecommandlockouts
\makeatletter\def\@IEEEpubidpullup{6.5\baselineskip}\makeatother
\IEEEpubid{\parbox{\columnwidth}{
    Network and Distributed System Security (NDSS) Symposium 2025\\
    24 - 28 February 2025, San Diego, CA, USA\\
    ISBN 979-8-9894372-8-3\\
    https://dx.doi.org/10.14722/ndss.2025.24886\\
    www.ndss-symposium.org
}
\hspace{\columnsep}\makebox[\columnwidth]{}}

\maketitle

\begin{abstract}
The exploitation of publicly accessible data has led to escalating concerns regarding data privacy and intellectual property (IP) breaches in the age of artificial intelligence. To safeguard both data privacy and IP-related domain knowledge, efforts have been undertaken to render shared data unlearnable for unauthorized models in the wild. Existing methods apply empirically optimized perturbations to the data in the hope of disrupting the correlation between the inputs and the corresponding labels such that the data samples are converted into Unlearnable Examples (UEs). Nevertheless, the absence of mechanisms to verify the robustness of UEs against uncertainty in unauthorized models and their training procedures engenders several under-explored challenges. First, it is hard to quantify the unlearnability of UEs against unauthorized adversaries from different runs of training, leaving the soundness of the defense in obscurity. Particularly, as a prevailing evaluation metric, empirical test accuracy faces generalization errors and may not plausibly represent the quality of UEs. This also leaves room for attackers, as there is no rigid guarantee of the maximal test accuracy achievable by attackers. Furthermore, we find that a simple recovery attack can restore the clean-task performance of the classifiers trained on UEs by slightly perturbing the learned weights. To mitigate the aforementioned problems, in this paper, we propose a mechanism for certifying the so-called $(q, \eta)$-Learnability of an unlearnable dataset via parametric smoothing. A lower certified $(q, \eta)$-Learnability indicates a more robust and effective protection over the dataset. Concretely, we 1) improve the tightness of certified $(q, \eta)$-Learnability and 2) design Provably Unlearnable Examples (PUEs) which have reduced $(q, \eta)$-Learnability. According to experimental results, PUEs demonstrate both decreased certified $(q, \eta)$-Learnability and enhanced empirical robustness compared to existing UEs. Compared to the competitors on classifiers with uncertainty in parameters, PUEs reduce at most $18.9\%$ of certified $(q, \eta)$-Learnability on ImageNet and $54.4\%$ of the empirical test accuracy score on CIFAR-100. Our source code is available at \href{https://github.com/NeuralSec/certified-data-learnability}{https://github.com/NeuralSec/certified-data-learnability}.
\end{abstract}


% \fi
%%\IEEEpeerreviewmaketitle

%%%%%%%%%%%%%%%%%%%%%%%%%%%%%%%%%%%%%
\begin{figure}[t]
    \centering
    \includegraphics[width=.9 \columnwidth]{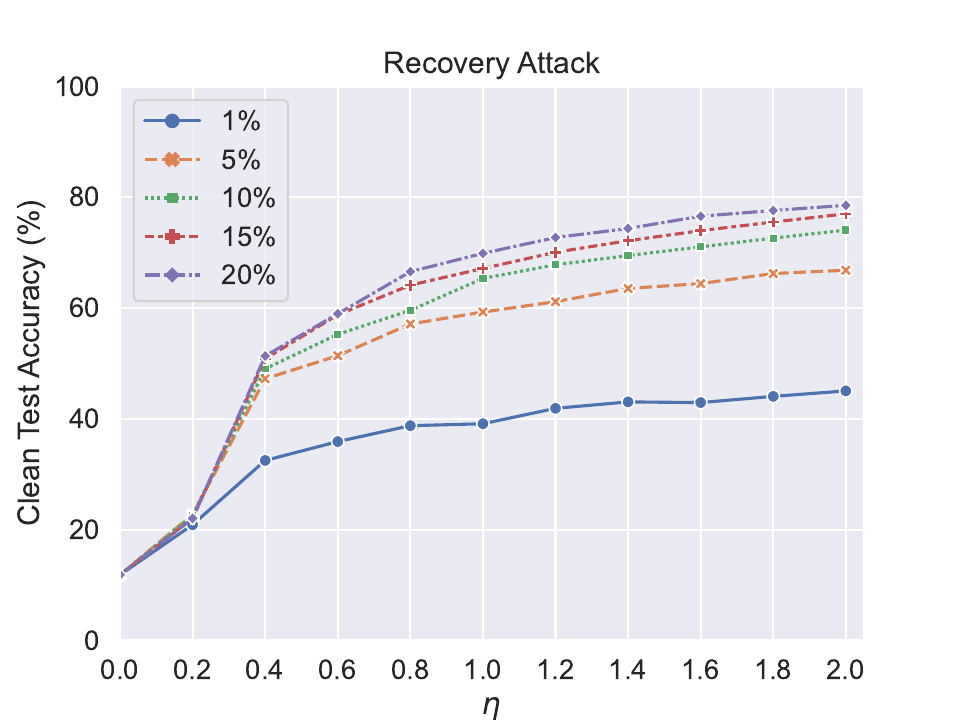}
    \caption{Recovery attacks using a small portion ($1\%$-$20\%$) of the CIFAR10 training set. Points on the curves trace the clean testing accuracy of classifiers whose weights are perturbed away from the poisoned classifier. For each fine-tuned classifier, the $\ell_2$ norm of the weight perturbation is capped by $\eta$. 
    This special adversary reveals that current UEs are not robust against uncertainty in classifier parameters, and their reliability cannot be guaranteed.
    Therefore, a mechanism for certifying UE performance is pivotal.
    %\bo{the message here is not clear. it's better to connect what's the relationship with your method. how your certification helps here}
    }
    \label{fig:recovery_attack_demo}
\end{figure}
%%%%%%%%%%%%%%%%%%%%%%%%%%%%%%%%%%%%%

%%%%%%%%%%%%%%%%%%%%%%%%%%%%%%%%%%%%%%%%%%%%%%%%%%%%%%%%%%%%%%%%%%%%%%%%%%%%%%%%%%%%%%%%%%%%%%%%%%%%%%%%%%%%%%%%%%%%%%%%%%
\section{Introduction}
Data privacy and intellectual property (IP) breaches have become major concerns in recent legislations and regulations serving for building responsible Artificial Intelligence (AI), such as GDPR~\cite{EuropeanParliament2016a}, CCPA~\cite{CAprivacy2018}, the European Union AI Act~\cite{EuropeanCommission2021}, and the recent US Executive Order on the Safe, Secure, and Trustworthy Development and Use of AI~\cite{us2023excutiveorder}.
These regulations have emphasized that consumers have the right to limit the use and disclosure of data collected from them.
Despite the laws and regulations, technologies for enhancing data privacy and protecting IP face increasingly severe challenges due to rapidly evolving machine learning algorithms.
Particularly, data published by individual users on content-sharing social media platforms is exposed to the threat of domain exploitation by unauthorized machine learning models.
For instance, a pretrained Stable Diffusion can be fine-tuned on a small set of paintings to generate new images mimicking the style of the paintings~\cite{heikkila2022art,gal2022image}. 
As another example, classifiers trained on publicly accessible data can be used to launch membership inference attacks exposing confidential personal information~\cite{shokri2017membership}.
As data becomes the new oil, these threats highlight the imminent need to implement effective measures for Data Availability Control (DAC).

As a response, a series of attempts have been made to make the published data unlearnable. 
One pragmatic and effective way to protect image datasets against unauthorized classifiers is to perturb the pixels before publishing the images.
Existing defenses apply anti-learning perturbations that minimize the training loss with respect to either correct or target labels on the input data.~\cite{fowl2021adversarial,huangunlearnable2021}.
In lieu of learning functions mapping the input data to the corresponding labels, the defenses encourage unauthorized models to learn a strong correlation (\ie a shortcut) between the perturbations and the labels. 
As a consequence, despite that the trained models can achieve low error rates on the data with the anti-learning perturbations, they generalize poorly to datasets from the clean data distribution. 
This line of work is also referred to as Perturbative Availability Poison (PAP) in the literature~\cite{feng2019learning,tao2021better,fowl2021adversarial,liu2023image}.

%%%%%%%%%%%%%%%%%%%%%%%%%%%%%%%%%%%%%%%%%%%%%%%%%%%%%%%%%%%%%%%%%%%%%%%%%%%%
%However, 
\noindent \textbf{Motivation.~} 
Existing PAP methods empirically search for the anti-learning perturbations based on a finite amount of training samples and models~\cite{huangunlearnable2021,furobust2022,tao2021better,fowl2021adversarial,yuan2021neural,chen2022self, he2024sharpness}. 
Henceforth, the generated perturbations may face the cross-model generalization problem brought by unknown illegal classifiers in the wild. 
%Moreover, it is discovered that Unlearnable Examples (UEs) crafted by these PAP methods can be made invalid by prevailing train-time techniques such as data augmentations, adversarial training, and data compression algorithms~\cite{huangunlearnable2021, furobust2022,liu2023image}. 
Moreover, there is no rigid guarantee of the maximally attainable learning results for adversaries on the UEs. 
The widely adopted empirical test accuracy, as a metric, is prone to uncertainty in the data, model selection, and training process.
This indicates that empirical test accuracy is not a sound metric for evaluating UEs.
UEs resulting in lower test accuracy scores on certain models and datasets do not necessarily evince their superiority over others and may leave room for attackers.
%due to the randomness, such as initialization and mini-batch selection, involved in the training process. 
Let us consider a type of recovery attack in which an attacker can fine-tune a poisoned classifier on a small set of clean data to restore its clean-task performance without significantly changing the classifier parameters.
%Some real-world cases in which restricting the weight change is crucial come from Federated Learning (FL) based on weight aggregation with Byzantine robustness~\cite{zhu2023byzantine}. 
This kind of adversary may exist in federated learning systems based on weight aggregation with Byzantine robustness~\cite{zhu2023byzantine}.
For example, an adversarial client training its unauthorized classifier on UEs can perform a recovery attack and upload the parameters to the server without alerting outlier detectors, rendering the global classifier to violate corresponding DAC protocols.
We randomly collect $1\%$-$20\%$ of CIFAR10 training samples to fine-tune a ResNet-18 trained on the error-minimizing PAP~\cite{huangunlearnable2021}.
We use projected stochastic gradient descent (SGD) in the fine-tuning so the $\ell_2$ norm of the weight differences is restricted within a small range. 
Surprisingly, we find that the accuracy can be restored from around $0.1$ to near $0.8$ by seemingly minuscule modifications in the weights using $5\%$-$10\%$ of CIFAR10 samples (Figure~\ref{fig:recovery_attack_demo}).
This result reflects that the robustness of UEs towards weight perturbations and uncertainties in classifier parameters is neither rigidly gauged nor carefully pondered.

The aforementioned problems hinder the assessment of UEs and constitute additional risks to be exploited by deliberate adaptive adversaries.
Therefore, it is crucial to develop a mechanism to ensure that a set of UEs can take effect without being ravaged by the problems of training stochasticity, cross-model generalization, and adaptive attackers.
The mechanism can resort to \textit{finding classifiers that can achieve the best possible clean-task accuracy when trained on UEs and restraining this best-case accuracy.}
We are thus motivated to ask the question:

\begin{mdframed}[backgroundcolor=grey!10,rightline=true,leftline=true,topline=true,bottomline=true,roundcorner=1mm,everyline=true,nobreak=false]  
\textit{Is it possible to go beyond heuristics and provide an upper bound towards the clean-data performance of undivulged models trained on a set of UEs?}
\end{mdframed}
The question is daunting to be thoroughly answered.
However, we will show in this work that, it is possible to derive a clean-data accuracy upper bound guaranteed with a high probability, for unauthorized classifiers under some constraints.
On the other hand, this threshold also serves as an upper bound of clean accuracy when facing the recovery attack.
It can henceforth be a tool for gauging the effectiveness and robustness of UEs.

%%%%%%%%%%%%%%%%%%%%%%%%%%%%%%%%%%%%%%%%%%%%%%%%%%%%%%%%%%%%%%%%%%%%%%%%%%%%
\noindent \textbf{Our method and its broader impact.~}
In this paper, we propose a mechanism for certifying the best clean-task accuracy retained by unauthorized classifiers trained on a set of UEs.
The mechanism can be applied to any empirically made UEs to derive a \textit{certified $(q, \eta)$-Learnability}.
%of the UEs.
$(q, \eta)$-Learnability gauges the effectiveness and robustness of UEs towards arbitrary classifiers sampled from a certain parameter space.
%UEs with higher $(q, \eta)$-Learnability are less robust, though they may produce good learning shortcuts in particular cases.
Note that the certification pertains to the parameter space rather than the input space, alleviating the need for knowledge of the training strategies or learning algorithms used by the attackers.
We also proposed certifying a generalization $(q, \eta)$-Learnability, which removes the requirement for holding private test datasets in the evaluation of UEs.
%The first reason is that the certification problem is attributed to finding the PAP-trained classifiers that can achieve the best possible clean-task utility.
%Second, certifying in the input space usually requires special classifiers on which certain statistics (\eg mean, median) of the predictions can be computed, bounded, and thus certified.
%This requirement is unrealistic for attackers in the wild since defenders generally have zero control over the classifiers and learning algorithms of the attackers.
To make the certification more sound, the paper also proposes a simple method to expand the certifiable parameter space and reduce the gap between the certified $(q, \eta)$-Learnability and the True Learnability (see Definition~\ref{def:learnabilitydef}).
Moreover, we produce \emph{Provably Unlearnable Examples} (PUEs) which achieve lower $(q, \eta)$-Learnability scores and be more robust than existing PAP methods.

The $(q, \eta)$-Learnability metric serves as a rigid guarantee of data availability for machine learning models.
This guarantee can operate in parallel with other techniques, such as deep watermarking~\cite{wang2024must} and differential privacy (DP)~\cite{hu2023sok}, to deter illegal data exploitation and support responsible AI regulations and standards (e.g., CCPA, GDPR, and Australia's AI Safety Standard~\cite{AustraliaAIsafety2023}).
Instead of post-event ownership auditing through watermarks and protection of sensitive information with DP, $(q, \eta)$-Learnability proactively controls access to domain knowledge through data, safeguarding both data IP and privacy.
The contributions of this paper are as follows:
\begin{itemize}[leftmargin=*]
    \item We formally derive the \emph{certified $(q, \eta)$-Learnability} of an unlearnable dataset. The certified $(q, \eta)$-Learnability serves as a guaranteed upper bound of the clean test accuracy that can be achieved by particular classifiers trained on the unlearnable dataset. To the best of our knowledge, this is the \textit{first} attempt towards guaranteed effectiveness of PAPs.
    \item As the \textit{first} step towards DAC with provable guarantees, we propose a simple method to narrow the gap between certified $(q, \eta)$-Learnability and True Learnability. Moreover, we design PUEs that can suppress the certified $(q, \eta)$-Learnability.
    \item PUEs not only achieved lower certified $(q, \eta)$-Learnability but also demonstrated better empirical robustness in scenarios where unauthorized classifiers cannot be certified.
    \item Our source code provides a tool for measuring the provable effectiveness of UEs and producing PUEs.
\end{itemize}

%Our method functions in parallel with other techniques, such as deep watermarking~\cite{wang2024must} and differential privacy~\cite{hu2023sok}, to deter illegal data exploitation and facilitate responsible AI regulations and standards.

%We will first introduce the required background knowledge of our study in the following section. 

%%%%%%%%%%%%%%%%%%%%%%%%%%%%%%%%%%%%%%%%%%%%%%%%%%%%%%%%%%%%%%%%%%%%%%%%%%%%%%%%%%%%%%%%%%%%%%%%%%%%%%%%%%%%%%%%%%%%%%%%%%
\section{Background}
In this section, we briefly introduce the concepts of UEs and smoothed classifiers.

%%%%%%%%%%%%%%%%%%%%%%%%%%%%%%%%%%%%%%%%%%%%%%%%%%%%%%%%%%%%%%%%%%%%%%%%%%%%
\subsection{Unlearnable Examples}
Let a set $\sD_s :=\{(x_i, y_i)|(x_i,y_i)\in \gX \times \gY\}_{i=1}^N$ be a source dataset containing $N$ labeled samples. 
$\gX$ and $\gY$ are the $I$-dimensional input space and the $K$-dimensional label space, respectively.
$\gX \times \gY$ is the Cartesian product of the two spaces.
%Let $\sX_s$ be the set of input samples in $\sD_s$, that is $\sX_s=\{x_1,...,x_N\}$.
The aim of the adversary is to train a classifier $f: \gX\subset\R^I \rightarrow \gY\subset\R^K$ by empirically minimizing a training loss $\gL(\cdot)$ over $\sD_s$ as follows:
\begin{equation}
    \min_{f} \frac{1}{N} \sum_{i=1}^{N} \gL (f(x_i), y_i).
\end{equation}
Suppose $\gD$ is the clean distribution from which $\sD_s$ is drawn.
The trained $f$ is supposed to perform well on samples from $\gD$.
Once $\gD$ is a data distribution from which other datasets with privacy concerns or registered intellectual property are drawn, it is crucial to make sure $\sD_s$ cannot be exploited in the unauthorized training of models which will later be applied to infer other private datasets drawn from $\gD$.
Instead of maximizing the training error, the common goal of the defender is to find a perturbation $\delta\in\gX$, such that
\begin{equation}
 \begin{aligned}
    \min_{\hat{f}}\,& \E_{(x,y)\sim\sD_s} \min_{\delta} \gL (\hat{f}(x+\delta),\, y)\\
              & s.t.\,\|\delta\|_p \leq \xi,
 \end{aligned}
\end{equation}
where $\hat{f}$ is a surrogate model used for searching $\delta$ and $\|\delta\|_p$ is the $\ell_p$-norm of the perturbation bounded by $\xi$. 
The $\ell_{\infty}$ norm is commonly used in the previous literature.
Such $x+\delta$ is a UE and $\delta$ is the PAP noise inserting a noise-label shortcut to minimize the training loss.
However, since the loss is empirically calculated, there is no rigid guarantee on the performance of the noise $\delta$.

%%%%%%%%%%%%%%%%%%%%%%%%%%%%%%%%%%%%%%%%%%%%%%%%%%%%%%%%%%%%%%%%%%%%%%%%%%%%
\subsection{Smoothed Classifiers} 
A smoothed classifier $g(x)$ predicts by returning the majority vote of predictions from a $K$-way base classifier $f_{\theta}:\R^I\rightarrow\{1,2,...,K\}$ parameterized by $\theta$ over random noises applied to either the input $x\in\R^I$ or the parameters during the test time. 
That is,
\begin{equation}
\begin{aligned}
    g(x)= & \argmax_{y\in\gY}{\Pr_{\epsilon\sim\pi(x)}[f_{\theta}(x+\epsilon)=y]}, \hspace{1em} or \\
    g(x)= & \argmax_{y\in\gY}{\Pr_{\epsilon\sim\pi(\theta)}[f_{\theta+\epsilon}(x)=y]}.
\end{aligned}
\end{equation}
$\pi(x)$ is the distribution of the noise (\ie smoothing distribution) centered at $x$ (or $\theta$) and $y$ is a predicted label from the label space $\gY$. 
The output from the smoothed classifier can be statistically bounded to offer certifiable robustness with respect to changes in the smoothed variables.
Cohen et al. derived a tight $\ell_2$ robustness bound by smoothing classifier inputs with noise drawn from Gaussian distributions and certifying based on Neyman-Pearson Lemma~\cite{cohen2019certified}.
Subsequently, the certified robustness is extended to various smoothing noises and certification frameworks~\cite{fischer2020certified,li2021tss, hao2022gsmooth,sukenik2022intriguing,cullen2022double,dvijotham2020framework,zhang2020black,salman2020denoised, yang2020randomized}. 
Such smoothed classifier $g(x)$ can also be extended to randomized learning functions which in addition smooth the training dataset $\sD_s$, resulting
\begin{equation}\label{eq:random_lf}
     g(\sD_s, x) = \argmax_{y\in\gY} \Pr_{\hat{\sD}_s\sim\pi(\sD_s)} [f_{\theta}(\hat{\sD}_s,\ x) = y].
\end{equation} 
Once we get the background information of UEs and randomized classifiers, we will move to define the research problem and the threat model of this paper in the following section. 

%%%%%%%%%%%%%%%%%%%%%%%%%%%%
\begin{figure*}[t]
    \centering
    \includegraphics[width=.9\textwidth]{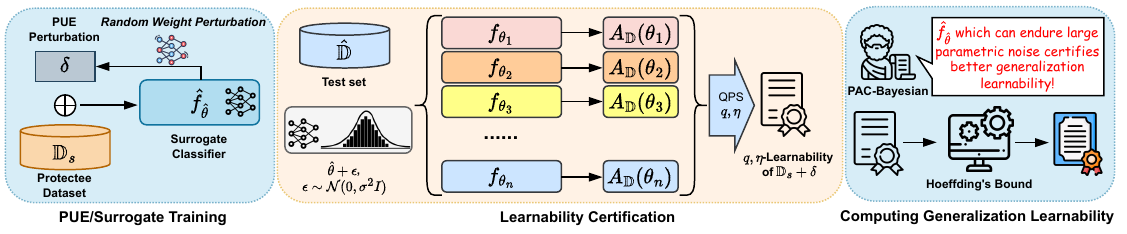}
    \caption{An overview of the certification and PUE crafting framework. A dataset $\sD_s$ is perturbed into $\sD_s \oplus \delta$ before being released to the public. The $(q, \eta)$-Learnability of $\sD_s \oplus  \delta$ can be certified to ensure that any unauthorized classifier trained on $\sD_s \oplus  \delta$ has a provable upper bound on its performance on any test set $\hat{\sD}$ (or $\sD$) in the same domain with $\sD_s \oplus \delta$, as long as the parameters of the unauthorized classifier are within a certified parameter set. Generalization learnability can be computed using Hoeffding's bound, and PAC-Bayesian theory suggests that a certification surrogate with low prediction variance under large parametric noise can improve certified learnability. Optimized PUEs lead to lower $(q, \eta)$-Learnability.}
    \label{fig:framework_overview}
\end{figure*}
%%%%%%%%%%%%%%%%%%%%%%%%%%%%

%%%%%%%%%%%%%%%%%%%%%%%%%%%%%%%%%%%%%%%%%%%%%%%%%%%%%%%%%%%%%%%%%%%%%%%%%%%%%%%%%%%%%%%%%%%%%%
\section{Problem Statement}
%%%%%%%%%%%%%%%%%%%%%%%%%%%%%%%%%%%%%%%%%%%%%%%%%%%%%%%%%%%%%%%%%%%%%%%%%%%%
\noindent \textbf{Problem definition.~}
Inheriting from the previous definitions, consider a to-be-published dataset $\sD_s$ (e.g., portraits to be posted on Facebook) with $N$ data points. 
Additionally, there is a private dataset $\sD$ (e.g., patient records with photos in a hospital database) sharing the same distribution $\gD$ with $\sD_s$.
The defender's goal is to render $\sD_s$ unlearnable against unknown learning algorithms and classifiers (e.g., deep neural networks) by perturbing it into $\sD_s \oplus \delta$, where $\delta$ is the anti-learning perturbation and $\oplus$ is the perturbing operator.
We specifically consider additive perturbations $\delta=\{\delta_i\}^{|\sD_s|}_{i=1}$ applied to each sample of $\sD_s$ in this paper.
Thus, $\oplus$ represents the element-wise addition.
``Unlearnable'' here means that a classifier trained on $\sD_s \oplus \delta$ would exhibit a high classification error on $\sD$.
Furthermore, the defender wants to verify how robust $\sD_s \oplus \delta$ is by having an upper bound on the accuracy that can be achieved by unknown classifiers stochastically trained on $\sD_s \oplus \delta$.
More profoundly, the defender wants to determine the best possible accuracy that classifiers trained on $\sD_s \oplus \delta$ can achieve on $\sD$. The defender does not desire an elevated best possible accuracy since it implies $\sD_s \oplus \delta$ is exploitable.
Therefore, the research goals are as follows.
\begin{mdframed}[backgroundcolor=grey!10,rightline=true,leftline=true,topline=true,bottomline=true,roundcorner=1mm,everyline=true,nobreak=false]  
\begin{itemize}[leftmargin=*]
    \item A certification mechanism should be developed to verify the maximum accuracy achievable by any classifiers and learning algorithms on $\sD$ through training on $\sD_s \oplus \delta$.
    \item A better noise $\delta$ should be designed such that $\sD_s \oplus \delta$ can effectively decrease the verified best accuracy. 
\end{itemize}
\end{mdframed}

To better formulate the problem, we first formally define the \emph{True Learnability} of $\sD_s \oplus \delta$ for classification tasks.
\begin{restatable}
[\textit{True Learnability of a perturbed dataset}]{definition}{learnabilitydef}\label{def:learnabilitydef}
Given a $\delta$-perturbed dataset $\sD_s \oplus \delta$, let $f_{\hat{\theta}}$ be a hypothesis with parameters $\hat{\theta}\in\R^d$ selected by arbitrary learning algorithm $\Gamma(\cdot)$, such that $\hat{\theta}=\Gamma(\sD_s \oplus \delta)$.
$\sD$ and $\sD_s$ come from the same data distribution $\gD$.
Suppose there are sample pairs $(x,y)\in\sD$. 
The learnability of the perturbed training dataset $\sD_s \oplus \delta$ is
\begin{equation}\label{eq:learnability_def}
\begin{aligned}
    \tL(\Theta; \sD_s\oplus\delta) := & \max_{\hat{\theta} \in \Theta}\, \E_{(x,y)\sim\sD} \mathds{1}[f_{\hat{\theta}}(x) = y],\\
\end{aligned}
\end{equation}
where $\mathds{1}[\cdot]$ is the indicator function and $\Theta$ is the space of all possible parameters that can be selected.
\end{restatable}
\noindent
Herein, $\hat{\theta}$ can be viewed as a function of $D_s\oplus\delta$.
In other words, $\tL(\Theta; \sD_s\oplus\delta)$ is the \emph{best possible} testing accuracy a classifier from the space $\Theta$ can achieve when evaluated on $\sD$.
Note that the learnability defined here differs from PAC learnability~\cite{valiant2013probably}.
$\tL(\cdot)$ deterministically reflects the lowest generalization error any classifier in the space $\Theta$ can achieve.
A smaller $\tL(\cdot)$ indicates that the learned classifier is harder to be generalized to the test-time dataset.
After having the definition of learnability, we can formulate the aim of an anti-learning perturbation $\delta$ as follows.
\begin{equation}\label{eq:unlearnable_objective}
    \begin{aligned}
        \min_{\delta}& \ \tL(\Theta; \sD_s\oplus\delta),\\
        s.t.& \ \forall\ \delta_i\in\delta,\ \|\delta_i\|_p \leq \xi.
    \end{aligned}
\end{equation}
Herein, $\xi$ is a fixed $\ell_p$ perturbation budget for noises added to each data point.
%$\hat{f}_{\hat{\theta}}$ is a surrogate classifier whose parameters form a set $\hat{\theta}$.
%$c$ is a constant equal to the observed clean test accuracy of $\hat{f}_{\hat{\theta}}$ over $\sD$.
An optimal anti-learning noise is the minimizer of the learnability.

In order to find the $\delta$, we need to first compute $\tL(\Theta; \sD_s\oplus\delta)$.
Solving $\tL(\Theta; \sD_s\oplus\delta)$ in an unknown space $\Theta$ is intractable.
However, we will show that a learnability score guaranteed with a high probability can be found if we restrict the space of possible hypothesis by only considering $\hat{\Theta}:=\{\theta\ |\ \theta\sim\N(\hat{\theta}+\upsilon, \sigma^2I),\ \|\upsilon\| \leq \eta\}$, where $\|\cdot\|$ is the $\ell_2$ norm, $\upsilon$ is a parametric perturbation, $\eta$ and $\sigma$ are two constants, and $\hat{\theta}$ is a set of surrogate parameters selected by the defender based on the dataset $\sD_s\oplus\delta$ to form the hypothesis $f_{\hat{\theta}}$.
If a set of PAP noises can reduce such a learnability score, they could achieve theoretically more robust protection of datasets against the uncertainty in classifiers and learning algorithms.

In essence, the problem to be addressed in this paper is two-fold. 
First, we certify the learnability of a dataset towards arbitrary classifiers trained on it as long as the trained parameters are from $\hat{\Theta}$.
This setting non-trivially appears in data exploitation cases through \textit{model retraining}, \textit{transfer learning}, \textit{model fine-tuning}, and \textit{recovery attacks}.  
Second, we strive to reduce the learnability of the protected dataset by designing a proper $\delta$. 
We will present in the following part a threat model of the paper. 

%%%%%%%%%%%%%%%%%%%%%%%%%%%%%%%%%%%%%%%%%%%%%%%%%%%%%%%%%%%%%%%%%%%%%%%%%%%%
\noindent \textbf{Threat model.~}
The producer of UEs (\ie defender) has white-box access to a surrogate classifier $f_{\hat{\theta}}$ and a protectee dataset $\sD_s\sim\gD$.
Additionally, either a clean test dataset $\sD$ or the domain $\gD$ of $\sD_s$ is revealed to the defender.
Note that it is \textit{not} necessary for the defender to access $\sD$ if it is privately held by a third party.
Instead, the defender can sample a dataset $\hat{\sD}$ from $\gD$ to compute a certified generalization learnability which serves as a generalization upper bound on the certified learnability.
The defender can alter the parameters of $f_{\hat{\theta}}$ and modify $\sD_s$ at will.
The defender crafts an unlearnable version $\sD_s \oplus \delta$ of $\sD_s$ and releases the unlearnable dataset rather than the unprotected one to the public.
Besides, defenders can either train surrogate classifiers on generated $\sD_s \oplus \delta$ or simultaneously with $\delta$ on $\sD_s$. 
However, the defender has zero access to any potential unauthorized classifiers and is unaware of the training procedure for the unauthorized classifiers.

As the adversaries in our threat model, the unauthorized classifiers can obtain labeled $\sD_s \oplus \delta$ as their training dataset. 
The trained unauthorized classifiers will be used to infer the labels of samples in $\gD$.
Importantly, the exact anti-learning perturbation $\delta$ is hidden from the adversaries.
Our defense tackles three levels of adversaries.
\begin{itemize}[leftmargin=*]
    \item \textbf{General Adversary (GA)}: Unauthorized classifier whose trained parameters are in $\mathbb{R}^d$, where $d$ is the parameter size.
    \item \textbf{Certifiable Adversary (CA)}: Unauthorized classifier whose trained parameters fall in a certified parameter set $\hat{\Theta}$.
    \item \textbf{Special Adversary (SA)}: Recovery attacker which changes the model parameters within an $\ell_2$ norm bound.
\end{itemize}
Among these adversaries, GA encompasses almost all attackers in real-world scenarios. 
CA denotes the subset of attackers from GA that can be provably mitigated by the defense. 
SA serves as a tool to gauge the robustness of different UEs against GAs who cannot be provably invalidated, as well as to assess the tightness of our defense against CAs.
The defender aims to expand $\hat{\Theta}$ to improve the proportion of CAs within GAs.
We will introduce in Section~\ref{sec:pue} a method for augmenting $\hat{\Theta}$.

%On the other hand, as the adversaries in our threat model, the unauthorized classifiers can obtain labelled $\sD_s\oplus\delta$ as their training dataset. 
%The trained unauthorized classifiers will be used to infer the labels of samples in $\sD$.
%Importantly, the exact anti-learning perturbation $\delta$ is hidden from the adversaries.
%Nevertheless, the adversary may perform a recovery attack by fine-tuning the poisoned classifier on a small clean dataset.

%%%%%%%%%%%%%%%%%%%%%%%%%%%%%%%%%%%%%%%%%%%%%%%%%%%%%%%%%%%%%%%%%%%%%%%%%%%%%%%%%%%%%%%%%%%%%%%%%%%%%%%%%%%%%%%%%%%%%%%%%%
\section{Overview}\label{sec:overview}
% state in this section how can we convert the certification problem from the dataset to a model trained on the dataset.
We will provide a brief overview over the crux of learnability certification and PUE generation in this section.
Our first desideratum is measuring how well an unlearnable dataset can sabotage those illegal but unknown classifiers by certifying a possibly best learning outcome on this dataset.
In a nutshell, a surrogate classifier selected based on the unlearnable set $\sD_s\oplus\delta$ can be viewed as a hypothesis of the data distribution. 
Intuitively, if the distribution of the unlearnable set is highly disjoint from that of clean data, the classifier would demonstrate a significant generalization error on the clean test dataset.
In contrast, if the unlearnable set is suboptimal, the classifier might still be able to generalize to the 
%mapping relationships in the 
clean dataset and thus induce a lower generalization error.
Herein, we convert the certification problem to certifying a surrogate classifier appropriately modeling the distribution of UEs.

Second, we design a simple method for selecting the surrogate classifier used in the certification.
Importantly, we emphasize that the two entities, the UE and the surrogate classifier, serve different objectives.
On the side of UEs, their purpose is upfront --- to minimize the learnability by lowering the clean testing accuracy of any classifiers trained on them.
%However, the design goal of the surrogate classifier diverges from that of UEs.
Instead, the surrogate should \textit{1) resemble the distributional characteristics of UEs} and \textit{2) have a high chance of discovering better clean test accuracy when its parameters are perturbed.}
In lieu of directly using classifiers trained by UEs as surrogates, we design a better way to construct surrogates meeting these two criteria.

In the following sections, we divide our method into two major parts.
In the first part, we will introduce how to certify the learnability of a dataset through a surrogate classifier.
Subsequently, we will show how PAP noises and a surrogate can be constructed to make more sound learnability certification in the UE regime.
At the same time, an overview of our framework is summarized in Figure~\ref{fig:framework_overview}.

%%%%%%%%%%%%%%%%%%%%%%%%%%%%%%%%%%%%%%%%%%%%%%%%%%%%%%%%%%%%%%%%%%%%%%%%%%%%%%%%%%%%%%%%%%%%%%%%%%%%%%%%%%%%%%%%%%%%%%%%%%
\section{Learnability Certification}\label{sec:lcert}
In this section, we will first describe the learnability certification process in detail. 
Certifying the learnability of a dataset over an unknown space of classifiers is a daunting task. 
However, it is possible to certify the learnability of a dataset towards classifiers whose parameters are in a particular set. 

%%%%%%%%%%%%%%%%%%%%%%%%%%%%%%%%%%%%%%%%%%%%%%%%%%%%%%%%%%%%%%%%%%%%%%%%%%%%
\subsection{Quantile Parametric Smoothing}
To effectively analyze classifiers in the space $\Theta$, we first introduce a Quantile Parametric Smoothing (QPS) function in the following definition.
\begin{restatable}
[\textit{Quantile Parametric Smoothing function}]{definition}{percsmooth}\label{def:perc_smooth}
Given a dataset $\sD$ from the space $\gX\times\gY$, an $\sD$-parameterized function $A_{\sD}:\Theta \rightarrow [0,1]$ with an input $\theta\in\Theta$, and a parametric smoothing noise $\epsilon\sim\N(0,\sigma^2I)$ under a standard deviation of $\sigma$, a Quantile Parametric Smoothing function $h_q(\theta)$ is defined as:
\begin{equation}\label{eq:discrete_QPS}
\begin{aligned}
    h_q(\theta) = \inf\ \{ t\ |\ \Pr_{\epsilon}[A_{\sD}(\theta + \epsilon)\leq t]\geq q\},
    % & \underline{h}_q(\theta) = \sup\ \{ t\ |\ \Pr_{\epsilon}[A_{\sD}(\theta + \epsilon)\leq t]\leq q\} \\
    % &  \overline{h}_q(\theta)  = \inf\ \{ t\ |\ \Pr_{\epsilon}[A_{\sD}(\theta + \epsilon)\leq t]\geq q\}, \\
\end{aligned}
\end{equation}
where $q\in[0, 1]$ is a probability.
% where $\underline{h}_q(\theta)$ and $\overline{h}_q(\theta)$ are the left and the right edges of ${h}_q(\theta)$, respectively.
\end{restatable}
\noindent
The definition connects to percentile smoothing used in certified robustness~\cite{chiang2020detection, bansal2022certified}.
In lieu of randomizing the inputs of classifiers, QPS takes the parameters (\ie weights) of classifiers as its variables and randomizes the parameters.
% Furthermore, the function $A_{\sD}$ is defined and parameterized over a dataset.
% The QPS function $h_q(\theta)$ may not have a closed form due to discrete $A_{\sD}(\theta + \epsilon)$ values from sampling, in which case some $q$ values cannot be reached.
% However, $\underline{h}_q(\theta)$ is equivalent to $\overline{h}_q(\theta)$ when $A_{\sD}$ is a continuous function, which reduces the definition to
% \begin{equation}\label{eq:continous_QPS}
%     h_q(\theta) = \inf\ \{ t\ |\ \Pr_{\epsilon}[A_{\sD}(\theta + \epsilon)\leq t]\geq q\}.
% \end{equation}

The QPS function has several nice properties for the task of learnability certification.
First, it returns the accuracy score in the $q$-th quantile of the samples/population.
By tuning the value $q$, QPS can naturally suit the need for computing the possibly highest accuracy that can be obtained by hypotheses, such that the learnability can be approximated.
Second, QPS considers all hypotheses whose parameters are defined on the support of the Gaussian smoothing distribution (\ie the set $\R^d$ of real numbers, where $d$ is the dimensionality of $\theta$). 
This set of hypotheses covers almost all real-world classifiers with parameters as real numbers. 
However, most of the Gaussian probability mass concentrates around the mean, making it nearly impossible in practice to sample parameters far away from $\theta$.
We are thus motivated to consider the QPS function with varying mean values such that parameters from a more extensive space can be sampled.

%%%%%%%%%%%%%%%%%%%%%%%%%%%%%%%%%%%%%%%%%%%%%%%%%%%%%%%%%%%%%%%%%%%%%%%%%%%%
\subsection{Certified Learnability}
To apply the QPS function to obtain a learnability certification, we explicitly define $A_{\sD}$ as a performance metric in this section.
Without loss of generality, such $A_{\sD}(\hat{\theta})$ can be the top-1 accuracy function of a classifier ${f}_{\hat{\theta}}$ mapping its input $x$ to a label $y$, where $(x,y)\in\sD$.
That is
\begin{equation}\label{def:learnability_function}
    A_{\sD}(\hat{\theta}) := \E_{(x,y)\sim\sD}\mathds{1}[{f}_{\hat{\theta}}(x)=y].
\end{equation}
Note that the parameters of the classifier are the variables of $A_{\sD}$.
Subsequently, we can obtain ${h}_q(\hat{\theta})$ by substituting such $A_{\sD}(\hat{\theta})$ into Equation~\ref{eq:discrete_QPS}. 
% By setting a proper $q$, the QPS function returns an accuracy $t$ in the $q$-th quantile of the population of accuracy scores.
An empirical value of $h_q(\hat{\theta})$ can be approximately calculated by evaluating $A_{\sD}(\hat{\theta}+\epsilon)$ multiple times given sampled $\epsilon$ values.
Suppose $f_{\hat{\theta}}$ is a surrogate classifier trained on an unlearnable dataset $\sD_s\oplus\delta$ by the defender, ${h}_q(\hat{\theta})$ indicates that, with probability at least $q$, there exists an accuracy upper bound for all classifiers $f_{\hat{\theta}+\epsilon}$ whose parameters are from the support of $\N(\hat{\theta},\sigma^2I)$.

We further move to the case in which the parameters are sampled from Gaussians with varying means. 
Specifically, we sample parameters from $\N(\hat{\theta}+\upsilon,\sigma^2I)$, where $\|\upsilon\|\leq \eta$.
When $\eta$ gets larger, we become more capable of sampling classifiers with diverse parameters.
% when an unauthorized classifier trained on $\sD_s\oplus\delta$ has parameters that differ from that of the surrogate by an $\ell_2$ norm up to $\eta$ or 2) a recovery attack is performed),
Given a sufficiently large $\eta$, if there still exists an upper bound of the accuracy scores, then this upper bound can be used as a certificate of the best clean test accuracy.
In other words, we want to certify an upper bound of ${h}_{{q}}(\hat{\theta}+\upsilon)$, $\forall \upsilon: \|\upsilon\|\leq \eta$. 
Such upper bound is called the \textit{certified $(q, \eta)$-Learnability} of $\sD_s\oplus\delta$. 
$\eta$ is a \textit{certified parametric radius of mean} describing the parameter space in which $(q, \eta)$-Learnability of $\sD_s\oplus\delta$ can be guaranteed with probability at least $q$.
For simplicity, we refer to $\eta$ as the \textit{certified parametric radius} in the following paper.
To this point, the task of learnability certification becomes clear.

%%%%%%%%%%%%%%%%%%%%%%%%%%%%%%%%%%%%%
\begin{figure}[t]
    \centering
    \includegraphics[width=.93\linewidth]{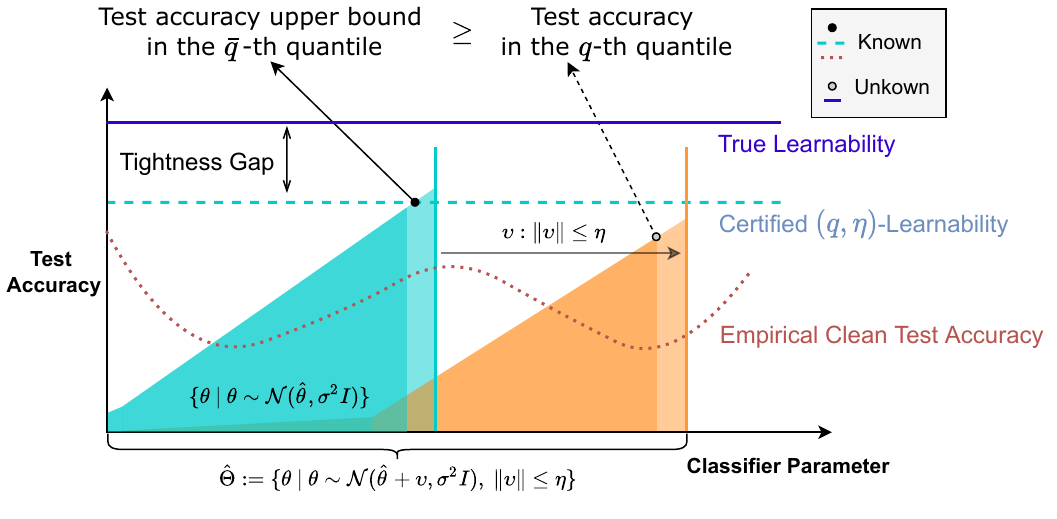}
    \caption{An illustration of the certified $(q, \eta)$-Learnability.}
    \label{fig:cert_demo}
\end{figure}
%%%%%%%%%%%%%%%%%%%%%%%%%%%%%%%%%%%%%

Next, we aim to construct such certified $(q, \eta)$-Learnability with a closed form so that it can be computed. 
Importantly, the certification of $(q, \eta)$-Learnability is provided through the following theorem:
%%%%%%%%%%%%%%%%%%%%%%%%%%%%%%%%%%%%%
\begin{restatable}
[\textit{Perturbation bound on QPS}]{theorem}{lcert}\label{theorem:cert_learnability}
Let $\Gamma: \gX\times\gY \rightarrow \hat{\theta}\in\Theta$ be a learning function selecting $\hat{\theta}$ from the parameter space $\Theta$ based on a dataset defined in $\gX\times\gY$.
Given an target dataset $\sD$ and a quantile smoothed function $h_q(\hat{\theta})$ centered at a Gaussian $\N(\hat{\theta}, \sigma^2I)$, then there exists an upper bound for ${h}_{q}(\hat{\theta}+\upsilon)$. 
Specifically,
\begin{equation}\small\label{eq:QPS_bound}
    \begin{aligned}
       & {h}_{q}(\hat{\theta}+\upsilon) \leq \inf\ \{t\ |\ \Pr_{\epsilon}[A_{\sD}(\hat{\theta}+\epsilon) \leq t] \geq \overline{q}\},\ \forall\ \|\upsilon\|\leq\eta,
    \end{aligned}
\end{equation}
where $\overline{q}:=\Phi( \Phi^{-1}(q) + \frac{\eta}{\sigma})$.
$\Phi(\cdot)$ is the standard Gaussian CDF and $\Phi^{-1}(\cdot)$ is the inverse of the CDF.
$\|\upsilon\|$ is the $\ell_2$ norm of the parameter shift $\upsilon$ from $\hat{\theta}$.
\end{restatable}
%%%%%%%%%%%%%%%%%%%%%%%%%%%%%%%%%%%%%
\noindent
We defer the proof of Theorem~\ref{theorem:cert_learnability} to Appendix~\ref{append:proof}.
The theorem states that the return of the QPS function is bounded with respect to the perturbation $\upsilon$ in its input.
Recall that ${h}_{q}(\hat{\theta}+\upsilon)$ is actually an upper bound of the clean test accuracy for classifiers whose parameters $\theta\sim\N(\hat{\theta}+\upsilon, \sigma^2I)$, with probability at least $q$.
Therefore, the right-hand side of Inequality~\ref{eq:QPS_bound} actually guarantees the best possible clean test accuracy for classifiers with parameters from the subspace $\hat{\Theta}:=\{\theta\ |\ \theta\sim\N(\hat{\theta}+\upsilon, \sigma^2I),\ \|\upsilon\| \leq \eta\}$.
We name such $\hat{\Theta}$ as a \textit{certified parameter set}.
Though it seems that $\hat{\Theta}$ covers all parameters in $\R^d$, we only consider those $\theta$ having a sufficiently large probability of being sampled.

Based on Theorem~\ref{theorem:cert_learnability}, we have the formal definition of $(q, \eta)$-Learnability as follows.
\begin{restatable}
[\textit{$(q, \eta)$-Learnability}]{definition}{qetadef}\label{definition:q_eta_learnability}
Suppose a learning function $\Gamma$ selects $\hat{\theta}$ based on an unlearnable dataset $\sD_s\oplus\delta$. The certified $(q, \eta)$-Learnability of $\sD_s\oplus\delta$ is
\begin{equation}\label{eq:learnability}
\begin{aligned}
    l_{(q,\eta)}(\hat{\Theta}; \sD_s\oplus\delta) = \inf\ \{t | \Pr_{\epsilon}[A_{\sD}(\hat{\theta}+\epsilon) \leq t] \geq \overline{q}\},
\end{aligned}
\end{equation}
where $\overline{q}=\Phi( \Phi^{-1}(q) + \frac{\eta}{\sigma})$.
For any $\theta^*$ drawn from the certified parameter set $\hat{\Theta}:=\{\theta\ |\ \theta\sim\N(\hat{\theta}+\upsilon, \sigma^2I),\ \|\upsilon\| \leq \eta\}$, there is $A_{\sD}(\theta^*) \leq l_{(q,\eta)}(\hat{\Theta}; \sD_s\oplus\delta)$ with probability no less than $q$.
\end{restatable}

It should be noted that, though $\sD_s\oplus\delta$ does not explicitly appear in the function, $\hat{\theta}$ is selected by a learning function $\hat{\theta} = \Gamma(\sD_s\oplus\delta)$ given $\sD_s\oplus\delta$.
Henceforth, $\hat{\theta}$ can be viewed as a function of $\sD_s\oplus\delta$ in this case.
In practice, we can sample $A_{\sD}(\hat{\theta}+\epsilon)$ using Monte Carlo and obtain the accuracy in the empirical $\overline{q}$-th quantile.
Next, a Binomial confidence interval upper bound of the accuracy can be adopted as $l_{(q,\eta)}(\hat{\Theta}; \sD_s\oplus\delta)$ with a confidence level of $1-\alpha$.
Detailed steps of computing the confidence interval upper bound are in Appendix~\ref{append:algorithm_details}.

Since $l_{(q,\eta)}(\hat{\Theta}; \sD_s\oplus\delta)$ is an upper bound only for accuracy obtained within $\hat{\Theta}$, there may exist a gap between such $l_{(q,\eta)}(\hat{\Theta}; \sD_s\oplus\delta)$ and the True Learnability.
Specifically, 
\begin{equation}
    \bigtriangleup l = \tL(\Theta; \sD_s\oplus\delta) - l_{(q,\eta)}(\hat{\Theta}; \sD_s\oplus\delta).
\end{equation}
The gap $\bigtriangleup l \in [0,1]$ represents the \textit{tightness} of the certified $(q, \eta)$-Learnability. 
In order to make sense of using the certified $(q, \eta)$-Learnability as a measurement for UEs, the tightness gap should be minimized.
Considering the concentration of the Gaussian probability mass, $\hat{\Theta}$ with a large $\eta$ can have higher probability measures at locations far from $\hat{\theta}$, which helps reduce the tightness gap.
A sketch describing the certified $(q, \eta)$-Learnability of a one-parameter classifier is illustrated in Figure~\ref{fig:cert_demo}.
We will introduce the practical certification algorithm in the next section.

%%%%%%%%%%%%%%%%%%%%%%%%%%%%%%%%%%%%%%%%%%%%%%%%%%%%%%%%%%%%%%%%%%%%%%%%%%%%
\subsection{Certification Algorithm}
We summarize the certification algorithm in this section.
We first train a surrogate classifier $f_{\hat{\theta}}$ of which the parameters $\hat{\theta}$ will be randomized for $n$ times over a Gaussian distribution to calculate the QPS function.
For each randomization, the resulting classifier with parameters $\hat{\theta}+\epsilon$ is evaluated on the clean test set $\sD$ such that $A_{\sD}(\hat{\theta}+\epsilon)$ is obtained.
Next, $\overline{q}$ can be theoretically computed based on Theorem~\ref{theorem:cert_learnability}.
Finally, we compute the confidence interval of the $A_{\sD}(\hat{\theta}+\epsilon)$ in the $\overline{q}$-th quantile to obtain the upper bound of the interval as $l_{(q,\eta)}(\hat{\Theta}; \sD_s\oplus\delta)$.
The detailed certification process is depicted in Algorithm~\ref{alg:quantile_estimate} and Algorithm~\ref{alg:cert}.
In addition, a meticulous description is in Appendix~\ref{append:algorithm_details}.

There are a few important things to be noticed.
% First, since the support of the parametric smoothing noise is the set $\R$ of real numbers, the certification covers classifiers whose parameters are from $\R$.
A higher $q$ is recommended for obtaining a learnability upper bound with high probability, and a larger $\eta$ helps establish a more inclusive certified parameter set.
In practice, both parameters can be optimized through grid search.
Moreover, the $\sigma$ of the parametric smoothing noise should be selected with caution. 
There is an intuition behind randomizing the surrogate weights and computing the QPS function --- 
\textit{some of the randomized classifiers can outperform the surrogate classifier on the clean test set}.
Notwithstanding, we find that, when an improperly large $\sigma$ is employed, the number of randomized classifiers having a higher accuracy than that of the surrogate drops significantly. 
The possible reason for this phenomenon could be that the shortcuts added to the UEs cannot fully conceal the mapping information between input samples and their labels.
Hence, the surrogate trained on the unlearnable dataset, though not adequately, is somehow fit to the clean data distribution as well.  
Large noises added to the parameters thus not only shift the surrogate away from the distribution of UEs but also move it off the clean data manifold.
This assumption is further verified by the recovery attack, which discovers classifiers with satisfactory performance on clean data in the vicinity of the poisoned classifier.
Therefore, we will further introduce the methods for constructing sound surrogates for tighter certification.

%%%%%%%%%%%%%%%%%%%%%%%%%%%%%%%%%%%%%
% \begin{minipage}{0.99\columnwidth} 
% \scalebox{0.99}{
% \centering
\begin{algorithm}[t]
\footnotesize
%\begin{breakablealgorithm}[h]
\caption{Quantile Upper Bound}\label{alg:quantile_estimate}
\textbf{func} \textsc{QUpperBound} \\
\KwIn{noise draws $n$, $\alpha$, $\sigma$, $\eta$, quantile $q$.}
\KwOut{Index of the value in the $q$-th quantile}
$\overline{q} \gets \Phi(\Phi^{-1}(q) + \frac{\eta}{\sigma})$ \\
$\underline{k},\ \overline{k}\ \gets \ \lceil n * \overline{q}\rceil,\ n$ \\
$k^*\ \gets\ 0$ \\
\For{$k\in \{\underline{k}, \underline{k}+1,..., \overline{k}\} $}
{
    \If{$\textsc{Binomial}(n, k, \overline{q}) > 1 - \alpha$}
    {
        $k^*\ \gets k$ \\
    }
    \Else 
    {
        $ Continue $ \\
    }
}
\If{$k^* \neq 0$}
{
    \KwOut{$k^*$}
}
\Else 
{
     \textsc{Abstain}
}
\textbf{func} \textsc{Binomial} \\
\KwIn{Sampling number $n$, $k$, $\overline{q}$.}
$\textsc{Conf}\ \leftarrow\ \sum_{i=1}^{k}\binom{n}{k}(\overline{q})^i(1-\overline{q})^{n-i}$ \\
\KwOut{$\textsc{Conf}$}
\end{algorithm}
% }
% \end{minipage}

%%%%%%%%%%%%%%%%%%%%%%%%%%%%%%%%%%%%%

%%%%%%%%%%%%%%%%%%%%%%%%%%%%%%%%%%%%%
%\vspace{1mm}
% \begin{minipage}{0.99\linewidth} 
% \scalebox{0.99}{
% \centering
\begin{algorithm}[t]
\footnotesize
%\begin{breakablealgorithm}[h]
\caption{$(q, \eta)$-Learnability Certification}\label{alg:cert}
\KwIn{Accuracy function $A_{\sD}$, surrogate weights $\hat{\theta}$, test set $\sD$, $n$, $\sigma$, $q$, $\eta$.}
\KwOut{$(q, \eta)$-Learnability}
Initialize $a$ \\
\For {$i \in 1,...,n$}
{
    $\theta\ \gets \ \hat{\theta}+\epsilon$, $\epsilon\sim\N(0,\sigma^2I)$ \\
    Evaluate $A_{\sD}({\theta})$ on $\sD$ \\
    Append $A_{\sD}({\theta})$ to $a$ \\
}
$a\ \gets \ Sort(a)$ \\
$k\ \gets \ \textsc{QUpperBound}(n, \alpha, \sigma, \eta, q)$ \\
$t\ \gets \ a_{k}$ \\
\KwOut{$t$}
\end{algorithm}
% }
% \end{minipage}
%%%%%%%%%%%%%%%%%%%%%%%%%%%%%%%%%%%%%

%%%%%%%%%%%%%%%%%%%%%%%%%%%%%%%%%%%%%%%%%%%%%%%%%%%%%%%%%%%%%%%%%%%%%%%%%%%%
\subsection{Generalization of the Certification}
% As the surrogate for certification, $f_{\hat{\theta}}$ should exhibit a low classification error on $\sD_s\oplus\delta$, signifying its successful capture of the UE distribution. 
% Second, the recovery attack reveals that the manifold of UEs can be close to that of clean samples. 
% Similar to the proof in Appendix F of the paper by Cohen et al.~\cite{cohen2019certified}, augmenting $f_{\hat{\theta}}$ with parametric noise maximizes the log-likelihood of correct predictions. 
% This augmentation can tighten the certified $(q, \eta)$-Learnability.
%That said, the certification can benefit from a more principled way of selecting $\hat{\theta}$.

%On the other hand, 
% since the test accuracy is empirically calculated based on $\sD$ with finite samples, we would like to have an upper bound on the generalization error such that a high $(q, \eta)$-Learnability also indicates high generalization accuracy on the clean data distribution $\gD$.
% B
In this section, we generalize the certification to the case where the defender has no access to the test dataset $\sD$ and discuss more properties of the generalization certification.

As a solution, the defender can sample a test set $\hat{\sD}$ from the accessible domain $\gD$ and certify a \textit{generalization $(q,\eta)$-Learnability} which bounds the learnability on test sets drawn from $\gD$.
By Hoeffding's Inequality, {\small$A_{\gD}(\hat{\theta}) \leq A_{\hat{\sD}}(\hat{\theta}) +\sqrt{\frac{1}{2N}\log(\frac{2}{\beta})}$} exists with probability at least $1-\beta$.
The right-hand side of the inequality can be used in the QPS function instead of $A_{\sD}(\hat{\theta})$.
Combining with the union bound, a certified generalization $(q,\eta)$-Learnability can be acquired as {\small$l_{(q,\eta)}(\hat{\Theta}; \sD_s\oplus\delta)+\sqrt{\frac{1}{2N}\log(\frac{2n}{\beta})}$}, where $n$ is the number of noise draws.
However, Hoeffding's Inequality does not account for factors that improve generalization learnability.
Therefore, we explore alternative approaches to establish the relationship between generalization learnability and the key components involved in the certification process.
%Moreover, 
Notably, classifiers sampled around a surrogate relate to PAC-Bayesian predictors (\eg Gibbs classifiers) whose parameters are randomly drawn from a distribution~\cite{mcallester2003simplified}. 
We can thus bound the expected generalization accuracy under random weight perturbation through the following corollary. 
%%%%%%%%%%%%%%%%%%%%%%%%%%%%%%%%%%%%%
\begin{restatable}
[\textit{Expected generalization accuracy under parametric smoothing noise}]{corollary}{generr}\label{theorem:generalization_err}
Let $A_{\hat{\sD}}:\Theta\times\gX\times\gY\rightarrow [0,1]$ be an accuracy function of a hypothesis parameterized by $\hat{\theta}\in\Theta$ evaluated on a dataset $\hat{\sD}\sim\gD$. When $\hat{\theta}$ is under a random perturbation $\upsilon$, with probability at least $1-\alpha$, we have:
\begin{equation}\label{eq:param_err_bound}
    %\resizebox{.45\textwidth}{!} 
    %{$
    \E_{\epsilon}[A_{\gD}(\hat{\theta}+\epsilon)] \geq \E_{\epsilon}[{A}_{\hat{\sD}}(\hat{\theta}+\epsilon)] -  \sqrt{\frac{\frac{\|\hat{\theta}\|^2}{\sigma^2} + \ln \frac{N}{\alpha}}{2(N-1)}},
    %$}
\end{equation}
where $N$ is the size of $\hat{\sD}$ and $\epsilon \sim \N(0, \sigma^2I)$.
\end{restatable}
%%%%%%%%%%%%%%%%%%%%%%%%%%%%%%%%%%%%%
\noindent
The corollary indicates that the expected generalization accuracy over the clean data distribution is lower bounded by the expected empirical accuracy minus a function of $\hat{\theta}$, $\sigma$ and $N$.
Although the corollary bounds the expectation rather than the $q$-th quantile, it connotes that a $f_{\hat{\theta}}$ constantly producing correct predictions on $\hat{\sD}$ also has a high generalization accuracy, which evinces the inferiority of the UEs resulting $\hat{\theta}$.
Furthermore, if the expectation of ${A}_{\hat{\sD}}(\hat{\theta}+\epsilon)$ is robust to varying $\sigma$, the expected generalization accuracy would grow with increasing $\sigma$.
For appropriately skewed accuracy distributions, this result implies that an $f_{\hat{\theta}}$ which endures large parametric noise can certify $(q, \eta)$-Learnability that generalizes better to $\gD$. 

%%%%%%%%%%%%%%%%%%%%%%%%%%%%%%%%%%%%%%%%%%%%%%%%%%%%%%%%%%%%%%%%%%%%%%%%%%%%
\begin{mdframed}[backgroundcolor=grey!10,rightline=true,leftline=true,topline=true,bottomline=true,roundcorner=1mm,everyline=false,nobreak=false]
\noindent \textbf{Remark.~}
We defined the QPS function of accuracy scores obtained by a set of classifiers randomized from a surrogate classifier fit to UEs $\sD_s\oplus\delta$.
The QPS function returns the value in the $q$-th quantile of all accuracy scores. 
The output of the QPS function can be further upper bounded by Theorem~\ref{theorem:cert_learnability}.
The upper bound of the accuracy in the $\overline{q}$-th quantile is treated as the $(q, \eta)$-Learnability of $\sD_s\oplus\delta$.
The key message of the certified $(q, \eta)$-Learnability is that --- 
if an unauthorized classifier has its trained parameters in the certified parameter set $\hat{\Theta}:=\{\theta\ |\ \theta\sim\N(\hat{\theta}+\upsilon, \sigma^2I),\ \|\upsilon\| \leq \eta\}$, then its clean test accuracy will not exceed the $(q, \eta)$-Learnability, with probability at least $q$.
Moreover, if $\sD$ is hidden from the defender, a generalization $(q,\eta)$-Learnability can be certified by sampling a test set from $\gD$ by the defender.
\end{mdframed}

%%%%%%%%%%%%%%%%%%%%%%%%%%%%%%%%%%%%%%%%%%%%%%%%%%%%%%%%%%%%%%%%%%%%%%%%%%%%%%%%%%%%%%%%%%%%%%%%%%%%%%%%%%%%%%%%%%%%%%%%%%
\section{Provably Unlearnable Examples}\label{sec:pue}
The core tasks of this section are 1) selecting surrogate classifiers that help certify a tighter $(q, \eta)$-Learnability and 2) generating PUEs to suppress $(q, \eta)$-Learnability against CAs while empirically tackling GAs.

%%%%%%%%%%%%%%%%%%%%%%%%%%%%%%%%%%%%%%%%%%%%%%%%%%%%%%%%%%%%%%%%%%%%%%%%%%%%
\subsection{Desiderata for Surrogates and PUEs}
Recall that we rely on a surrogate classifier fit to the UEs to sample the accuracy scores $A_{\sD}(\hat{\theta}+\epsilon)$.
\textit{The aims of the surrogate are to 1) produce as tight as possible certified $(q, \eta)$-Learnability scores and 2) capture the distributional characteristics of the UEs it is trained on.}
Through parametric randomization, a plausible surrogate should discover as many as possible neighboring classifiers that have higher clean test accuracy.
The best case of unauthorized classifiers can thus have a higher chance of being sampled.
This means the noise added to the surrogate parameters should be large enough such that the probability measure of these potential best cases during sampling can be significant. 
Conversely, an immediate requirement for this type of surrogate classifier is that it should be able to tolerate large parametric noises while making predictions.
Furthermore, the second target can be fulfilled when the surrogate has a low classification error on the set of UEs.

As the surrogate for certification, $f_{\hat{\theta}}$ should exhibit a low classification error on $\sD_s\oplus\delta$, signifying its successful capture of the UE distribution. 
Second, the recovery attack reveals that the manifold of UEs can be close to that of clean samples. 
Similar to the proof in Appendix F of the paper by Cohen et al.~\cite{cohen2019certified}, augmenting $f_{\hat{\theta}}$ with parametric noise maximizes the log-likelihood of correct predictions. 
This augmentation can tighten the certified $(q, \eta)$-Learnability.
% As a conclusion, the principal design goals for the surrogate classifier can be distilled hereby.
% \begin{mdframed}[backgroundcolor=white!10,rightline=true,leftline=true,topline=true,bottomline=true,roundcorner=1mm,everyline=true,nobreak=false]
% \emph{A surrogate fit to the UEs should maintain its utility under random perturbations on its weights such that a reduced certified $(q, \eta)$-Learnability can be attributed to the UEs rather than the parametric smoothing noise.}
% \end{mdframed}
% From the perspective of the generalization error, $1-l_{(q,\eta)}(\hat{\Theta}; \sD_s\oplus\delta)$ can be viewed as a certified error, it is a combination of the train-test generalization error and parametric perturbation error.
% The train-test generalization error originates from the discrepancy in the distributions of training and testing data.
% On the other hand, the parametric perturbation error comes from perturbations in the model at test time.

Suppose there is a tightly certified $(q, \eta)$-Learnability, the task of making PUEs is straightforward --- we will find better perturbations $\delta$ such that the surrogate $\hat{f}_{\theta}$ trained on $\sD_s\oplus\delta$ can produce a lower $l_{(q,\eta)}(\hat{\Theta}; \sD_s\oplus\delta)$.
At a high level, the key intuition behind crafting PUEs that can suppress its $(q, \eta)$-Learnability is described as follows: 
\begin{mdframed}[backgroundcolor=grey!10,rightline=true,leftline=true,topline=true,bottomline=true,roundcorner=1mm,everyline=false,nobreak=false]
A set of plausible UEs should lead to a trained classifier having not only a low clean test accuracy by itself but also low accuracy scores by the population of randomized classifiers in its vicinity.
\end{mdframed}
Thereafter, we will show in the next section that the design goals of the surrogate and UEs are attainable under a unified training framework.

%%%%%%%%%%%%%%%%%%%%%%%%%%%%%%%%%%%%%%%%%%%%%%%%%%%%%%%%%%%%%%%%%%%%%%%%%%%%
\subsection{Random Weight Perturbation}
To meet the first goal, we seek certification surrogates that can generate more correct predictions through randomizations under the parametric noise.
Motivated by the findings in recovery attacks, such surrogates can be trained by using train-time random parametric noise as an augmentation.
There are two possible ways to train surrogates.
First, an offline surrogate can be trained on already made $\sD_s\oplus\delta$.
Otherwise, in an online setting, a surrogate can be trained together with $\delta$.

In the online settings, we can optimize the PAP noise $\delta$ with the surrogate under random weight perturbations.
It is hard to exactly minimize $l_{(q,\eta)}(\hat{\Theta}; \sD_s\oplus\delta)$. 
However, recall that when the ratio of small-valued $A_{\sD}(\theta+\epsilon)$ grows, a lower value can be obtained in the fixed $q$-th quantile of accuracy scores.
Therefore, PAP noises simultaneously updated with a surrogate over random weight perturbations may curtail the $(q, \eta)$-Learnability of the resulting UEs.
Given a dataset $\sD_s$, we train additive error-minimizing noises and a surrogate classifier ${f}_{\hat{\theta}}$ under random weight perturbations. 
The objective of the optimization problem can be formulated as follows:
\begin{equation}\label{eq:minmin_train}
\begin{aligned}
    \min_{\hat{\theta}}&\  \E_{(x,y\sim\sD_s)}\ \min_{\delta'\subseteq\delta}\ \E_{\epsilon\sim\pi(0)}\ \gL({f}_{\hat{\theta}+\epsilon}(x+\delta'), y),\\
    &\ s.t.\ \forall\ \delta'\subseteq\delta,\ \|\delta'\| \leq \xi,
\end{aligned}
\end{equation} 
where $\delta'$ is the corresponding noise added to $x$ and $\pi(0)$ is a train-time noise distribution with a mean of $0$.
In practice, we sample a set of $U$ random weight perturbations with a zero-mean Gaussian distribution, to train $\delta$ and $\hat{\theta}$.
In the offline setting, we omit the inner minimization objective on $\delta$ and only optimize $\hat{\theta}$.
The training on the dataset $\sD_s$ is thus to minimize the following empirical loss:
\begin{equation}\label{eq:empirical_minmin_train}
    \begin{aligned}
        \min_{{\hat{\theta}}} \frac{1}{N}\sum_{j=1}^{N} \min_{\delta_j, \|\delta_j\|\leq\xi} \frac{1}{U}\sum_{i=1}^{U} \gL[f_{\hat{\theta}+\epsilon_i}(x_{j}+\delta_j), y_{j}],
    \end{aligned}
\end{equation}
where $(x_{j}, y_{j})$ is sampled from $\sD_s$ and $\epsilon_i\sim\N(0, {\sigma'}^2I)$ is the $i$-th noise added to $\hat{\theta}$. 
In each step of training, we gradually increase the value of $\sigma'$ with a step size of $s$ to update $\hat{\theta}$ multiple times such that the classifier can better converge.
The optimized perturbations $\delta$ are then superposed on $\sD_s$ to make it unlearnable.
The detailed algorithm for making the PUEs is depicted in Algorithm~\ref{alg:delta_opt}.
More details of the optimization process are in Appendix~\ref{append:algorithm_details}.
%\bo{the format here looks weird. it would be great to summarize the tightness of the certification and refer to the evaluation of it in the exp sec}

In the next section, we will evaluate the certified $(q, \eta)$-learnability and its tightness for PUEs, as well as other state-of-the-art UEs. 
Moreover, we will compare the empirical test accuracy and robustness of PUEs with those of competitors to demonstrate the effectiveness of PUEs against GAs.

%%%%%%%%%%%%%%%%%%%%%%%%%%%%%%%%%%%%%
%\vspace{1mm}
% \begin{minipage}{\columnwidth} 
% \scalebox{0.99}{
% \centering
\begin{algorithm}[t]
\footnotesize
%\begin{breakablealgorithm}[h]
\caption{Training of $f_{\hat{\theta}}$ and $\delta$.}\label{alg:delta_opt}
\KwIn{Surrogate classifier $\hat{f}_{\theta}$ with parameters $\theta$, training dataset $\sD_s$, smoothing distribution $\N(0, \sigma^2I)$, $U_{train}$, $U_{perturb}$, noise STD cap $S$, step size $s$, training batch number $N$, train step number $M$, stop error rate $\tau$, validation set $\sD'_s$.}
\KwOut{PUE noise $\delta$ and surrogate parameters $\hat{\theta}$.}
\textbf{Initialize:} \\
$\delta\ \gets\ [0]^{|\sD_s|\times d}$ \\
$ error \gets 1e^{9}$  \\
\While{$error \geq \tau$}
{
    \For {$i \in 1,...,M$}
    {
        Load a mini-batch $(x_i, y_i)$ from $\sD_s$ \\
        $l\ \gets\ 0$ \\
        $\sigma'\ \gets\ 0$ \\
        \While{$\sigma' \leq S$}
        {   
            $\sigma'\ \gets\ \sigma'+s$ \\
            \For {$j \in 1,...,U_{train}$}
            {
                $\epsilon_j \sim\ \N(0, \sigma'^2I)$ \\
                $l\ +=\ \gL[f_{{\theta}+\epsilon_j}(x_{i}+\delta), y_{i}]$  \\
            }
            $\theta\ \gets\ \textsc{Train\_step}(\theta, l/U_{train})$ \\
        }
    }
    \For{$i \in 1,...,N$}
    {
        Load a mini-batch $(x_i, y_i)$ from $\sD_s$ \\
        $l\ \gets\ 0$ \\
        $\sigma'\ \gets\ 0$ \\
        \While{$\sigma' \leq S$}
        {   
            $\sigma'\ \gets\ \sigma'+s$ \\
            \For{$j \in 1,...,U_{perturb}$}
            {
                $\epsilon_j \sim\ \N(0, \sigma'^2I)$ \\
                $l\ +=\ \gL[f_{{\theta}+\epsilon_j}(x_{i}+\delta), y_{i}]$  \\
            }
            $\delta\ \gets\ \textsc{Opt\_step}(x_{i}, \delta, l/U_{perturb})$ \\
        }
    }
    $error\ \gets\  \textsc{Eval}(\theta, \sD'_s+\delta)$ \\
}
$\hat{\theta}\ \gets\ \theta$\\
\KwOut{$\delta$, $\hat{\theta}$}
\end{algorithm}
% }
% \end{minipage}
%%%%%%%%%%%%%%%%%%%%%%%%%%%%%%%%%%%%%

%%%%%%%%%%%%%%%%%%%%%%%%%%%%%%%%%%%%%%%%%%%%%%%%%%%%%%%%%%%%%%%%%%%%%%%%%%%%
%\vspace{3.5mm}
\begin{mdframed}[backgroundcolor=grey!10,rightline=true,leftline=true,topline=true,bottomline=true,roundcorner=1mm,everyline=false,nobreak=false]  
\noindent \textbf{Remark.~}
PUEs aim to increase the classification error of a set of randomized classifiers. 
This ensures that the PUEs are effective towards not only the surrogate but also other classifiers who are on the support of the smoothing noise distribution. 
In contrast, the surrogate classifier trained on PUEs should stay immune to the weight perturbations caused by parametric smoothing such that the decrease in the certified learnability can be attributed to the effect of PUEs rather than the weight perturbations.
We solve this two-layer problem as a task of bi-level minimization over the data distribution and random weight perturbations. 
\end{mdframed}

%%%%%%%%%%%%%%%%%%%%%%%%%%%%%%%%%%%%%%%%%%%%%%%%%%%%%%%%%%%%%%%%%%%%%%%%%%%%%%%%%%%%%%%%%%%%%%%%%%%%%%%%%%%%%%%%%%%%%%%%%%
\section{Experiments}\label{sec:exp}
We present the experimental results of the certified $(q, \eta)$-Learnability and empirical performance of PUEs against CAs and GAs.
Through the experiments, we aim to answer three questions.
\begin{itemize}[leftmargin=*]
    \item Can random weight perturbations increase the tightness of certified $(q, \eta)$-Learnability? 
    \item Is PUE an effective way of producing unlearnable datasets suppressing certified $(q, \eta)$-Learnability?
    \item How robust is PUE against GAs who escape from the certified parameter set $\hat{\Theta}$?
\end{itemize}
Accordingly, we conduct experiments to measure both the certified $(q, \eta)$-Learnability and the empirical performance of PUE and its baselines. 

%%%%%%%%%%%%%%%%%%%%%%%%%%%%%%%%%%%%%%%%%%%%%%%%%%%%%%%%%%%%%%%%%%%%%%%%%%%%
%\vspace{1mm}
\noindent \textbf{Baselines.~}
We consider both \textit{online surrogates} and \textit{offline surrogates} in the certification against CAs.
We compare PUE with a tailored version of Error-Minimizing Noise (EMN)~\cite{huangunlearnable2021} in their certified $(q, \eta)$-Learnability to show the effectiveness of our online surrogate and the robustness of PUE.
Moreover, we train offline surrogates on off-the-shelf PUE, EMN, and OPS~\cite{wu2023one} to compare their certified $(q, \eta)$-Learnability.
When it comes to GAs, we conduct a series of empirical comparisons with EMN and OPS to show their robustness and clean test accuracy scores in cases where unauthorized classifiers escape from the certified parameter set $\hat{\Theta}$.
Though our method does not restrict the type of noise applied to make data unlearnable (\ie class-wise or sample-wise), we found that PUE surrogates trained with class-wise noises can converge rapidly while those trained by sample-wise noises do not converge very well.
Therefore, in this section, we focus on class-wise noise in all experiments and comparisons.

%%%%%%%%%%%%%%%%%%%%%%%%%%%%%%%%%%%%%%%%%%%%%%%%%%%%%%%%%%%%%%%%%%%%%%%%%%%%
%\vspace{1mm}
\noindent \textbf{Data and models.~} 
In our experiments, we adopt CIFAR-10, CIFAR100, and ImageNet with $100$ randomly selected categories as the datasets in our evaluation.
ResNet-18 is used as the surrogate model to generate UEs in all experiments.
During certification experiments, we use ResNet-18 as the architecture of surrogates.
We certify against $5000$ ImageNet test set samples and the entire test datasets of CIFAR-10/CIFAR100.
In empirical evaluations, we employ ResNet-18, ResNet-50, and DenseNet-121 as the architectures for unauthorized classifiers.
We test the empirical robustness of PUE, EMN, and OPS on ResNet-18 using recovery attacks.
The test accuracy scores of PUE, EMN, and OPS are measured by training ResNet-18, ResNet-50, and DenseNet-121 on them, respectively.

%%%%%%%%%%%%%%%%%%%%%%%%%%%%%%%%%%%%%%%%%%%%%%%%%%%%%%%%%%%%%%%%%%%%%%%%%%%%
%\vspace{1mm}
\noindent \textbf{Training setups.~}
In each iteration of PUE optimization and online/offline surrogate training, the standard deviation (STD) of the noises added to the parameters is gradually increased from $0$ with a step size of $0.05$ until the full noise level is reached. 
Furthermore, we start adding the weight perturbations after some warm-up steps to stabilize the training.
Generally, we start perturbing weights after the classification error on UEs decreases below $50\%$. 
When generating UEs and fitting online surrogates, we set $M=10$ on CIFAR10, $M=20$ on CIFAR100, and $M=100$ on ImageNet.
No random data augmentation is used in the surrogate training processes.
In empirical experiments, we follow the previous training routines~\cite{huangunlearnable2021,wu2023one} and set the epoch to $60$ when training classifiers or launching recovery attacks on CIFAR10. 
For the training and recovery attacks on CIFAR100 and ImageNet, the epoch number is set to $100$. 
The weight decay rate is $5\times 10^{-4}$ on CIFAR10 and is $5\times 10^{-5}$ for CIFAR100/ImageNet.
We apply an SGD with a momentum of $0.9$ and a learning rate of $0.1$ in all classifier/surrogate training runs. 
%More details can be found in Appendix~\ref{append:exp:settings}.

%%%%%%%%%%%%%%%%%%%%%%%%%%%%%%%%%%%%%%%%%%%%%%%%%%%%%%%%%%%%%%%%%%%%%%%%%%%%
%\vspace{1mm}
\noindent\textbf{Certification setups.~}
We select $0.25$ as the STD of the train-time Gaussian noise to obtain the surrogates.
During the certification process, we set $q=0.9$ and $\sigma\in\{0.25, 0.8\}$ for all the surrogates.
There are reasons behind this setting.
First, $\sigma=0.25$ matches with that of the train-time noise, and can certify the highest $(q, \eta)$-Learnability at a fixed $\eta$ (see Section~\ref{subsec:ablation_study}).
Second, $\sigma=0.8$ is sufficient for certifying a parametric radius $\eta=1.0$ within which the recovery attack can restore most of the clean accuracy (refer to Figure~\ref{fig:recovery_attack_common} in Section~\ref{subsec:pue_robustness}).
For each surrogate classifier, we sample $1000$ classifiers around it to compute the confidence interval of the QPS function under a confidence level of $0.99$.
For each surrogate, we certify under various values of $\eta$ and record the certified learnability scores under each $\eta$.
To compute the generalization $(q, \eta)$-Learnability, one can simply adding $\sqrt{\frac{1}{N}\log(\frac{2n}{\beta})}$ to the certified learnability scores given $n$, $N$ and $\beta$.
For instance, for $n=1000$ and $\beta=0.01$, sampling a dataset with size $N=5000$ from the domain $\gD$ as the test set to compute the generalization $(q, \eta)$-Learnability will attract a universal increase of $0.05$ to the certified learnability scores in the tables.
Each run of certification on CIFAR10/CIFAR100 takes around one GPU hour on Nvidia Tesla P100. 
Certifying an ImageNet classifier requires 26.4 GPU hours.
We apply offsets to the certified results to mitigate the impact of accuracy differences in the surrogates.
More details of hyper-parameters and certification settings can be found in Appendix~\ref{append:ex_setup}.

% We compared our PUEs with other four PAP methods, namely Error-Minimizing Noise (EMN)~\cite{huangunlearnable2021}, Robust UEs (RUE)~\cite{furobust2022},  Self-Ensemble Protection (SEP)~\cite{chen2022self}, and One-Pixel Shortcut (OPS)~\cite{wu2023one}.
% UE and RUE generate anti-learning noise with local surrogates.
% SEP generate the noise with the checkpoints of a fully trained model.
% Furthermore, OPS is a surrogate-free method that generates perturbations purely on a dataset.

% \noindent \textbf{Purification defenses.~}
% A series of training techniques, such as data augmentation, adversarial training, and random input smoothing, might be employed in the training phase.
% To demonstrate the robustness of the UEs, we 
% The generated noises are tested against commonly trained deterministic classifiers as well as randomized classifiers produced under two provable defenses, namely Provable Wasserstein Bound (PWB)~\cite{kumar2023provable} and RAB~\cite{weber2022rab}.

%%%%%%%%%%%%%%%%%%%%%%%%%%%%%%%%%%%%%%%%%%%%%%%%%%%%%%%%%%%%%%%%%%%%%%%%%%%%
\subsection{Learnability Certification}\label{subsec:cert}
 In this section, we study the certified $(q, \eta)$-Learnability of UEs for CAs and gauge the tightness of the certification using SAs (\ie recovery attacks).

%%%%%%%%%%%%%%%%%%%%%%%%%%%%%%%%%%%%%
\begin{table}[t]
\caption{Certified $(q, \eta)$-Learnability under Different Training Methods ($\%$, $\sigma=0.25$)}
\label{table:effectiveness_of_training_25}
\centering
\resizebox{.99\linewidth}{!}{%
\begin{tabular}{ccccccccccc}
\toprule
\multirow{2}{*}{Data} & \multirow{2}{*}{Method} & \multicolumn{7}{c}{$\eta\times100$} \\
\cmidrule(r){3-11}
                          &        &  0.1  &  0.5  &  1.0  &  5.0  &  10.0 & 15.0  & 20.0  & 25.0  & 30.0 \\
\midrule
\multirow{2}{*}{CIFAR10}  & PUE-B  & \textbf{10.62} & \textbf{10.67} & \textbf{10.71} & \textbf{11.07} & \textbf{11.86} & \textbf{12.50} & \textbf{13.20} & \textbf{14.67} & \textbf{15.75} \\
                          & EMN    &  5.69 &  5.70 &  5.74 &  5.91 &  6.24 &  6.43 &  6.96 & 8.61 & 10.27 \\
\midrule
\multirow{2}{*}{CIFAR100} & PUE-B  &  \textbf{1.32} &  \textbf{1.32} &  \textbf{1.33} &  \textbf{1.37} &  \textbf{1.41} &  \textbf{1.47} &  \textbf{1.53} &  \textbf{1.59} &  \textbf{1.68} \\
                          & EMN    &  0.43 &  0.43 &  0.44 &  0.47 &  0.52 &  0.59 &  0.70 &  0.77 &  0.89 \\      
\midrule
\multirow{2}{*}{ImageNet} & PUE-B  &  \textbf{1.46} &  \textbf{1.46} &  \textbf{1.48} &  \textbf{1.54} &  \textbf{1.63} &  \textbf{1.79} &  \textbf{1.85} &  \textbf{1.97} &  \textbf{2.13} \\
                          & EMN    &  1.34 &  1.34 &  1.37 &  1.41 &  1.45 &  1.49 &  1.54 &  1.58 &  1.67 \\
\bottomrule
\end{tabular}
}
\end{table}
%%%%%%%%%%%%%%%%%%%%%%%%%%%%%%%%%%%%%

%%%%%%%%%%%%%%%%%%%%%%%%%%%%%%%%%%%%%
\begin{table}[t]
\caption{Certified $(q, \eta)$-Learnability under Different Training Methods ($\%$, $\sigma=0.8$)}
\label{table:effectiveness_of_training_80}
\centering
\resizebox{.99\linewidth}{!}{%
\begin{tabular}{cccccccccccc}
\toprule
\multirow{2}{*}{Data} & \multirow{2}{*}{Method} & \multicolumn{7}{c}{$\eta$} \\
\cmidrule(r){3-12}
                          &        &  0.1  &  0.2  &  0.3  &  0.4  &  0.5  &  0.6  &  0.7  &  0.8  &  0.9  &  1.0 \\
\midrule
\multirow{2}{*}{CIFAR10}  & PUE-B  & \textbf{10.68} & \textbf{10.86} & \textbf{11.18} & \textbf{11.37} & \textbf{11.52} & \textbf{12.00} & \textbf{12.71} & \textbf{12.87} & \textbf{13.72} & \textbf{15.17}\\
                          & EMN    &  5.86 &  6.11 &  6.28 &  6.52 &  7.00 &  7.27 &  7.36 &  7.87 &  8.38 &  9.02\\
\midrule
\multirow{2}{*}{CIFAR100} & PUE-B  &  \textbf{1.13} &  \textbf{1.15} &  \textbf{1.16} &  \textbf{1.20} &  \textbf{1.23} &  \textbf{1.27} &  \textbf{1.31} &  \textbf{1.36} &  \textbf{1.51} &  \textbf{1.56}\\
                          & EMN    &  0.47 &  0.48 &  0.51 &  0.55 &  0.59 &  0.63 &  0.66 &  0.67 &  0.69 &  0.70\\    
\midrule
\multirow{2}{*}{ImageNet} & PUE-B  &  \textbf{1.19} &  \textbf{1.24} &  \textbf{1.26} &  \textbf{1.28} &  \textbf{1.32} &  \textbf{1.37} &  \textbf{1.45} &  \textbf{1.50} &  \textbf{1.61} &  \textbf{1.80}\\
                          & EMN    &  1.12 &  1.14 &  1.16 &  1.20 &  1.25 &  1.29 &  1.40 &  1.40 &  1.48 &  1.48 \\
\bottomrule
\end{tabular}
}
\end{table}
%%%%%%%%%%%%%%%%%%%%%%%%%%%%%%%%%%%%%

%%%%%%%%%%%%%%%%%%%%%%%%%%%%%%%%%%%%%
\begin{table}[t]
\caption{Certified $(q, \eta)$-Learnability under Different PAP Noises ($\%$, $\sigma=0.25$, online)}
\label{table:effectiveness_of_noise_25}
\centering
\resizebox{.99\linewidth}{!}{%
\begin{tabular}{ccccccccccc}
\toprule
\multirow{2}{*}{Data} & \multirow{2}{*}{Method} & \multicolumn{7}{c}{$\eta\times100$} \\
\cmidrule(r){3-11}
                      &               &  0.1  &  0.5  &  1.0  &  5.0  &  10.0 & 15.0  & 20.0  & 25.0  & 30.0 \\
\midrule
\multirow{3}{*}{CIFAR10}  & PUE-10    & \textbf{10.24} & \textbf{10.29} & \textbf{10.31} & \textbf{10.69} & \textbf{11.14} & \textbf{11.82} & \textbf{12.29} & \textbf{13.12} & \textbf{13.59} \\
                          & PUE-1     & 10.86 & 10.97 & 11.04 & 11.62 & 12.12 & 12.64 & 13.35 & 14.22 & 14.66 \\
                          % & EMN-R     & 10.49 & 10.51 & 10.55 & 10.93 & 11.29 & 11.75 & 12.28 & 12.82 & 14.04 \\
                          & PUE-B     & 10.62 & 10.67 & 10.71 & 11.07 & 11.86 & 12.50 & 13.20 & 14.67 & 15.75 \\
\midrule
\multirow{3}{*}{CIFAR100} & PUE-10    &  \textbf{1.29} &  \textbf{1.29} &  \textbf{1.31} &  \textbf{1.35} &  \textbf{1.41} &  \textbf{1.43} &  \textbf{1.46} &  \textbf{1.51} &  \textbf{1.65} \\
                          & PUE-1     &  1.35 &  1.36 &  1.36 &  1.42 &  1.48 &  1.53 &  1.57 &  1.69 &  1.87 \\
                          % & EMN-R     &   &   &   &   &   &   &   &   &   \\
                          & PUE-B     &  1.32 &  1.32 &  1.33 &  1.37 &  \textbf{1.41} &  1.47 &  1.53 &  1.59 &  1.68 \\
\midrule
\multirow{3}{*}{ImageNet} & PUE-10    &  \textbf{1.45} &  \textbf{1.45} &  \textbf{1.45} &  \textbf{1.50} &  \textbf{1.61} &  \textbf{1.68} &  \textbf{1.76} &  \textbf{1.84} &  \textbf{1.95} \\
                          & PUE-1     &  1.58 &  1.58 &  1.58 &  1.67 &  1.73 &  1.81 &  1.95 &  2.00 &  2.19 \\
                          % & EMN-R     &   &   &   &   &   &   &   &   &   \\
                          & PUE-B     &  1.46 &  1.46 &  1.48 &  1.54 &  1.63 &  1.79 &  1.85 &  1.97 &  2.13 \\
\bottomrule
\end{tabular}
}
\end{table}
%%%%%%%%%%%%%%%%%%%%%%%%%%%%%%%%%%%%%

%%%%%%%%%%%%%%%%%%%%%%%%%%%%%%%%%%%%%
\begin{table}[t]
\caption{Certified $(q, \eta)$-Learnability under Different PAP Noises ($\%$, $\sigma=0.8$, online)}
\label{table:effectiveness_of_noise_80}
\centering
\resizebox{.99\linewidth}{!}{%
\begin{tabular}{cccccccccccc}
\toprule
\multirow{2}{*}{Data} & \multirow{2}{*}{Method} & \multicolumn{7}{c}{$\eta$} \\
\cmidrule(r){3-12}
                          &           &  0.1  &  0.2  &  0.3  &  0.4  &  0.5  &  0.6  &  0.7  &  0.8  &  0.9  &  1.0 \\
\midrule
\multirow{3}{*}{CIFAR10}  & PUE-10    & \textbf{10.10} & \textbf{10.24} & \textbf{10.42} & \textbf{10.80} & \textbf{11.10} & \textbf{11.57} & \textbf{11.89} & \textbf{12.45} & \textbf{12.81} & 13.72 \\
                          & PUE-1     & 10.60 & 10.73 & 11.07 & 11.21 & 11.38 & 11.67 & 12.04 & 12.71 & 13.00 & \textbf{13.59} \\
                          % & EMN-R     &   &    &   &   &    &    &    &    &    &   \\
                          & PUE-B     & 10.68 & 10.86 & 11.18 & 11.37 & 11.52 & 12.00 & 12.71 & 12.87 & 13.72 & 15.17\\
\midrule
\multirow{3}{*}{CIFAR100} & PUE-10    &  \textbf{1.11} &  \textbf{1.14} &  \textbf{1.16} &  \textbf{1.20} &  \textbf{1.23} &  \textbf{1.25} &  \textbf{1.27} &  \textbf{1.35} &  1.48 &  \textbf{1.52} \\
                          & PUE-1     &  1.13 &  1.15 &  1.18 &  1.21 &  1.24 &  1.27 &  1.31 &  \textbf{1.35} &  \textbf{1.44} &  1.60\\
                          % & EMN-R     &   &    &   &   &    &    &    &    &    &   \\
                          & PUE-B     &  1.13 &  1.15 &  \textbf{1.16} &  \textbf{1.20} &  1.23 &  1.27 &  1.31 &  1.36 &  1.51 &  1.56\\
\midrule
\multirow{3}{*}{ImageNet} & PUE-10    &  \textbf{1.17} &  \textbf{1.19} &  \textbf{1.24} &  \textbf{1.26} &  \textbf{1.28} &  \textbf{1.32} &  \textbf{1.37} &  \textbf{1.45} &  \textbf{1.54} &  \textbf{1.61}\\
                          & PUE-1     &  1.22 &  1.24 &  1.30 &  1.35 &  1.37 &  1.41 &  1.45 &  1.52 &  \textbf{1.54} &  1.68\\
                          % & EMN-R     &   &    &   &   &    &    &    &    &    &   \\
                          & PUE-B     &  1.19 &  1.24 &  1.26 &  1.28 &  1.32 &  1.37 &  1.45 &  1.50 &  1.61 &  1.80\\
\bottomrule
\end{tabular}
}
\end{table}
%%%%%%%%%%%%%%%%%%%%%%%%%%%%%%%%%%%%%

%%%%%%%%%%%%%%%%%%%%%%%%%%%%%%%%%%%%%
\begin{table}[t]
\caption{Certified $(q, \eta)$-Learnability under Different PAP Noises ($\%$, $\sigma=0.25$, offline)}
\label{table:effectiveness_of_noise_25_offline}
\centering
\resizebox{.99\linewidth}{!}{%
\begin{tabular}{ccccccccccc}
\toprule
\multirow{2}{*}{Data} & \multirow{2}{*}{Method} & \multicolumn{7}{c}{$\eta\times100$} \\
\cmidrule(r){3-11}
                      &               &  0.1  &  0.5  &  1.0  &  5.0  &  10.0 & 15.0  & 20.0  & 25.0  & 30.0  \\
\midrule
\multirow{3}{*}{CIFAR10}  & PUE-10    & \textbf{10.57} & \textbf{10.58} & \textbf{10.58} & \textbf{10.83} & \textbf{11.10} & \textbf{11.39} & \textbf{11.68} & \textbf{12.22} & \textbf{13.20} \\
                          & EMN       & 11.20 & 11.26 & 11.28 & 11.48 & 11.76 & 12.23 & 12.55 & 13.01 & 13.43 \\
                          & OPS       & 10.59 & 10.61 & 10.64 & 10.96 & 11.37 & 11.73 & 11.98 & 12.54 & 13.38 \\
\midrule
\multirow{2}{*}{CIFAR100} & PUE-10    &  \textbf{1.13} &  \textbf{1.13} &  \textbf{1.13} &  \textbf{1.16} &  \textbf{1.21} &  \textbf{1.27} &  \textbf{1.31} &  \textbf{1.35} &  \textbf{1.47} \\
                          & EMN       &  1.16 &  1.16 &  1.17 &  1.20 &  1.25 &  1.31 &  1.37 &  1.43 &  1.60 \\
\bottomrule
\end{tabular}
}
\end{table}
%%%%%%%%%%%%%%%%%%%%%%%%%%%%%%%%%%%%%

%%%%%%%%%%%%%%%%%%%%%%%%%%%%%%%%%%%%%
\begin{table}[t]
\caption{Certified $(q, \eta)$-Learnability under Different PAP Noises ($\%$, $\sigma=0.8$, offline)}
\label{table:effectiveness_of_noise_80_offline}
\centering
\resizebox{.99\linewidth}{!}{%
\begin{tabular}{cccccccccccc}
\toprule
\multirow{2}{*}{Data}     & \multirow{2}{*}{Method} & \multicolumn{7}{c}{$\eta$} \\
\cmidrule(r){3-12}
                          &           &  0.1  &  0.2  &  0.3  &  0.4  &  0.5  &  0.6  &  0.7  &  0.8  &  0.9  &  1.0  \\
\midrule
\multirow{3}{*}{CIFAR10}  & PUE-10    & \textbf{10.65} & 10.86 & 11.06 & \textbf{11.33} & \textbf{11.52} & \textbf{11.64} & \textbf{11.97} & \textbf{12.17} & \textbf{12.69} & \textbf{13.43} \\
                          & EMN       & 11.56 & 11.74 & 11.96 & 12.36 & 12.48 & 12.90 & 13.15 & 13.58 & 14.02 & 14.33 \\
                          & OPS       & \textbf{10.65} & \textbf{10.83} & \textbf{10.98} & 11.36 & 11.55 & 11.72 & 12.02 & 12.24 & 13.59 & 13.65 \\
\midrule
\multirow{2}{*}{CIFAR100} & PUE-10    &  \textbf{1.13} &  \textbf{1.15} &  \textbf{1.17} &  \textbf{1.20} &  \textbf{1.25} &  \textbf{1.30} &  1.36 &  \textbf{1.38} &  \textbf{1.39} &  \textbf{1.42} \\
                          & EMN       &  1.14 &  1.17 &  1.22 &  1.26 &  1.29 &  1.34 &  \textbf{1.34} &  \textbf{1.38} &  1.43 &  1.46 \\
\bottomrule
\end{tabular}
}
\end{table}
%%%%%%%%%%%%%%%%%%%%%%%%%%%%%%%%%%%%%

%%%%%%%%%%%%%%%%%%%%%%%%%%%%%%%%%%%%%
\begin{figure}[t]
     \centering
     \begin{subfigure}
         \centering
         \includegraphics[width=.51\columnwidth]{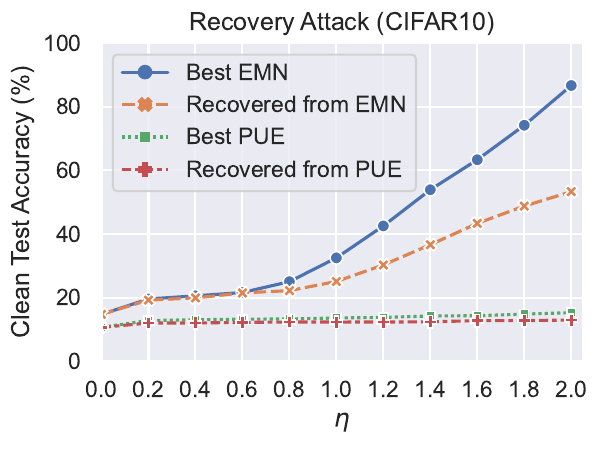}    
     \end{subfigure}\hspace{-3mm}%
     \begin{subfigure}
         \centering
         \includegraphics[width=0.51\columnwidth]{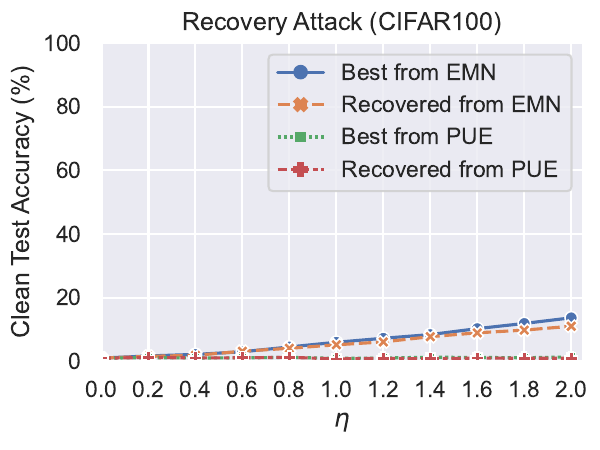}
     \end{subfigure}
\caption{The clean test accuracy scores of surrogate classifiers gauged on clean CIFAR10 (left) and CIFAR100 (right) test set after recovery attacks. The best accuracy scores approximate the True Learnability of parameters inside the hypersphere centered at $\hat{\theta}$ with a radius of $\eta$.}
\label{fig:recovery_attack}
\end{figure}
%%%%%%%%%%%%%%%%%%%%%%%%%%%%%%%%%%%%%

%%%%%%%%%%%%%%%%%%%%%%%%%%%%%%%%%%%%%%%%%%%%%%%%%%%%%%%%%%%%%%%%%%%%%%%%%%%%
%\vspace{1mm}
\noindent\textbf{Certification results.~}
We compare certified learnability scores from two aspects.
We first compare the certified learnability obtained based on surrogates trained on the same datasets but with different training strategies to show the effectiveness of random weight perturbations.
Second, we compare the certified learnability of different PAP noises by using the same random weight perturbing strategy to train surrogates.
This comparison aims to reveal the performance of PUE in reducing the certified learnability.

To make the comparisons, we introduce a baseline surrogate, PUE-B, which employs random weight perturbations in the training of the surrogate parameters but does not perturb the weights when optimizing the PAP noise. 
Therefore, PUE-B differs from EMN only in the training method of the surrogate and can be used as the baseline in the evaluation of PUEs.
On the other hand, the surrogate trained without random weight perturbation is denoted as EMN.
The comparison is made under two different levels of smoothing noise ($\sigma=0.25$ and $\sigma=0.8$).
The results are presented in Table~\ref{table:effectiveness_of_training_25} and Table~\ref{table:effectiveness_of_training_80}.
According to Table~\ref{table:effectiveness_of_training_25}, compared to EMN, PUE-B can be certified at a higher $(0.9, \eta)$-Learnability across all datasets. 
It can also be observed that the gain in the $(0.9, \eta)$-Learnability of PUE-B increases when the certified radius rises. 
On CIFAR100 and ImageNet, the gain becomes more obvious.
The observation supports the effectiveness of the random weight perturbation strategy in reducing the gap between the $(q, \eta)$-Learnability and the True Learnability.

Next, we compare the certified learnability of online surrogates trained by random weight perturbations on different UEs.
We generated two versions of PUEs (\ie PUE-1 and PUE-10) by setting the value of $U_{perturb}$ to $1$ and $10$, respectively.
The certification results based on the corresponding online surrogates are recorded in Table~\ref{table:effectiveness_of_noise_25} and Table~\ref{table:effectiveness_of_noise_80}.
Generally, in all the tables, PUE-10 achieves the lowest $(0.9, \eta)$-Learnability.
When the smoothing noise level matches the noise level of random weight perturbation, PUE-10 can diminish the $(0.9, \eta)$-Learnability on all datasets.
However, when the parametric smoothing noise gets larger to $0.8$, PUE-10 has the best results at most of the certified radii across the three datasets but not for some cases at large $\eta$ values.
This should be the consequence of the large smoothing noise used, which will be verified by the ablation study in Section~\ref{subsec:ablation_study}.
Saliently, all of the results on PUEs outperform that of PUE-B, meaning that optimizing the PAP noise over random weight perturbations can effectively suppress the $(0.9, \eta)$-Learnability. 

At last, we compare certification results based on offline surrogates trained on CIFAR10 and CIFAR100 with PAP noises.
We find surrogates trained on CIFAR100 OPS cannot be augmented properly within the same training epochs used for PUE and EMN.
We thus only compare with OPS on CIFAR10.
The certified learnability scores are presented in Table~\ref{table:effectiveness_of_noise_25_offline} and Table~\ref{table:effectiveness_of_noise_80_offline}. 
We have similar observations on the certification results based on online surrogates and offline surrogates. 
According to the results, PUEs certified based on offline surrogates outperform both EMN and OPS when $\sigma=0.25$.
When $\sigma=0.8$, OPS performs better on CIFAR10 at small $\eta$.
Similarly, EMN obtains slightly lower certified learnability scores at $\eta=0.7$.
We suspect it is a consequence of the large smoothing noise, which leads to a loose certification.
Additionally, the advantage of OPS can be a result of insufficient steps of training since classifiers usually require more epochs to converge on OPS.
However, PUE still outperforms OPS when $\eta$ exceeds $0.4$.

% In addition, we find that classifiers trained on off-the-shelf UEs of EMN, OPS, TAP~\cite{fowl2021adversarial} and NTGA~\cite{yuan2021neural} cannot be properly augmented by random weight perturbations on CIFAR100 and ImageNet.
% Consequently, they produce lower but less convincing $(q, \eta)$-Learnability scores.
% The classifiers trained with fixed PAP noises converge in a few epochs of training when using the same stop criteria for training PUE (\ie classification error less than $0.1$), suggesting more training steps should be used. 
% Inasmuch as it is impossible to keep the comparison under the same setting, we do not compare with these PAPs.

%%%%%%%%%%%%%%%%%%%%%%%%%%%%%%%%%%%%%%%%%%%%%%%%%%%%%%%%%%%%%%%%%%%%%%%%%%%%
%\vspace{1mm}
\noindent\textbf{Tightness of the certification.~}
%\noindent\textbf{Hardness results on recovery attacks.~}
It is important to make sense of the gap between certified learnability and True Learnability.  
Since True Learnability is impossible to calculate, we gauge it by launching recovery attacks against surrogates used in the certification.
Specifically, we use projected SGD to finetune the online surrogates of PUE-10 and EMN, respectively, on clean samples and clip their weights to make the after-trained weights stay within an $\ell_2$ norm of $\eta$ from the original weights.
We use $20\%$ of the CIFAR10/CIFAR100 training set in the fine-tuning to obtain the generalized clean test accuracy and use the CIFAR10/CIFAR100 test set to approximate the best accuracy. 
The learning rate of projected SGD is set to $0.01$ in all attacks.
The clean test accuracy scores recovered by the attack under different $\eta$ are plotted in Figure~\ref{fig:recovery_attack}.
On the CIFAR10 EMN surrogate, the generalized test accuracy can be recovered from below $15\%$ to near $60\%$ within an $\eta$ of $2.0$ while the best possible accuracy is over $85\%$.
In contrast, the recovery attack against the CIFAR10 PUE surrogate can only restore the accuracy from $10.68\%$ to $12.98\%$ and, in the best case, $15.31\%$.
The best accuracy is close to the certified learnability scores, meaning tighter learnability is certified using the PUE surrogate.

Compared to CIFAR10, CIFAR100 surrogates demonstrate lower best accuracy and recovered accuracy, which could be imputed to two possible reasons.
First, CIFAR100 is by nature more difficult to learn than CIFAR10.
Second, the CIFAR100 surrogate is trained by more iterations on the UEs and is harder to recover.
However, it can be seen that the PUE surrogate still outperforms the EMN one on the tightness of certification.
The EMN surrogate has the best accuracy ranging from $1.03\%$ to $13.69\%$ when $\eta$ increases, while the accuracy of the PUE surrogate can only be recovered to $1.40\%$ in the best case.
The comparison shows the superiority of the PUE surrogate in certifying tighter $(q, \eta)$-Learnability.
We also examine the impact of surrogate architecture towards the tightness of the certified $(q, \eta)$-Learnability.
Please refer to Appendix~\ref{append:ex_setup} for details. 
In addition, we check the validity of the certification under the stochasticity of different training runs to show that classifiers trained by adversaries can still fall inside the surrogate-certified $\hat{\Theta}$.
% According to the results in Table~\ref{table:cert_with_diff_arc_0.25}, the surrogate with fewer parameters produces tighter certifications.
% Moreover, tighter $(q, \eta)$-Learnability can be certified when the certification surrogate has the same architecture as the noise generator in the offline setting.

% Probably, True Learnability on CIFAR100 can be better approximated since the gauged best accuracy is even lower than some certified $(q, \eta)$-Learnability scores.
% Nevertheless, gauging True Learnability with recovery attacks seems to be our current best bet.

% From another perspective, This comparison also shows the hardness results of recovery attacks against hypotheses with particular parameters.
% The surrogate trained on PUE demonstrates higher collective robustness over parametric changes in a radius of $\eta$ and is barely affected by the recovery attack.
% The results are based on the surrogate trained with special techniques (\ie random weight perturbation), which does not apply to the training process of the attackers.
% However, the results point out that proper UEs allow learning algorithms to select desired hypotheses more robustly from a specific region of the parameter space.

%%%%%%%%%%%%%%%%%%%%%%%%%%%%%%%%%%%%%
\begin{figure}[t]
     \centering
     \begin{subfigure}
         \centering
         \includegraphics[width=.51\columnwidth]{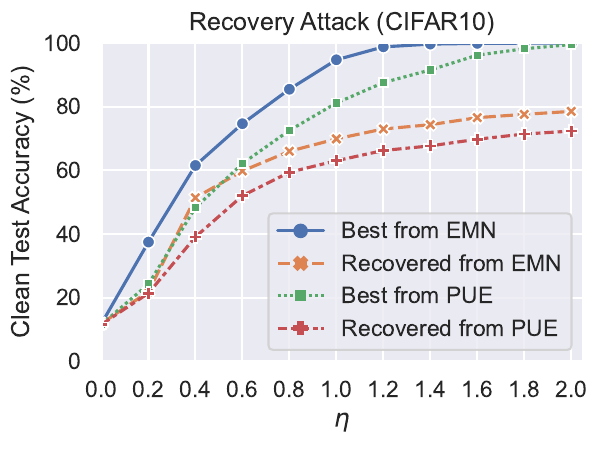}    
     \end{subfigure}\hspace{-3mm}%
     \begin{subfigure}
         \centering
         \includegraphics[width=0.51\columnwidth]{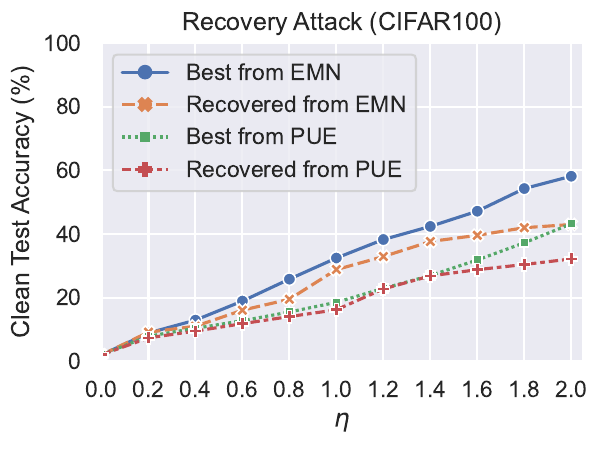}
     \end{subfigure}
\caption{The robustness of PUE and EMN against recovery attacks.}
\label{fig:recovery_attack_common}
\end{figure}
%%%%%%%%%%%%%%%%%%%%%%%%%%%%%%%%%%%%%

%%%%%%%%%%%%%%%%%%%%%%%%%%%%%%%%%%%%%%%%%%%%%%%%%%%%%%%%%%%%%%%%%%%%%%%%%%%%
%\vspace{1mm}
\noindent\textbf{Validity of the certification.~}
Recall that the validity of certified $(q, \eta)$-learnability is characterized by a certified parameter set. 
However, the training of classifiers faces stochasticity in weight initialization, mini-batch loading, stochastic gradient descent, etc. 
This implies that training on PUEs could lead to dissimilar model weights across different runs, so the PUEs that were certified by the surrogate classifier may not remain valid in subsequent training runs conducted by adversaries.
To investigate the validity of PUEs for adversaries using mainstream training techniques, we train 10 ResNet-18 models on CIFAR10 PUEs with random weight initialization, batch loading orders, and data augmentations to check how diverse their trained weights can be. 
After training, we draw ten layers (nine convolution layers and one linear layer) and visualize the distribution of the converged weights. 
The visualization can be found along with our source code at \href{https://github.com/Provably-Unlearnable-Examples/PUE}{https://github.com/Provably-Unlearnable-Examples/PUE}.
Furthermore, we calculate the pairwise difference for each parameter across the 10 classifiers and record the mean and STD of all the differences. 
According to the results, the weights from the 10 runs share similarities in their distributions, with a mean parameter difference of $-4.92\times 10^{-6}$ and an STD of $0.01$.
Note that the certified parameter set $\hat{\Theta}$ is actually an infinite Gaussian mixture.
For example, with $\sigma=0.25$ and $\eta=1.0$, $\hat{\Theta}$ has sufficiently large probability mass within $[\hat{\theta}-0.5, \hat{\theta}+0.5]$ for $\hat{\theta}$ with a dimensionality of $11$M (\ie the rough parameter count of ResNet-18) and beyond, which suggests that the certified parameter set can cover classifier weights from different training runs with stochasticity.
Therefore, $\hat{\Theta}$ certified by a surrogate $\hat{\theta}$ can effectively capture classifiers trained separately by adversaries following certain standard training procedures.

%%%%%%%%%%%%%%%%%%%%%%%%%%%%%%%%%%%%%%%%%%%%%%%%%%%%%%%%%%%%%%%%%%%%%%%%%%%%
\subsection{Robustness of PUE against GAs}\label{subsec:pue_robustness} 
In this section, we aim to understand how robust PUEs are when the classifier trained on them is outside of the certified parameter set $\hat{\Theta}$ or evades the certified learnability scores with a probability of $1-q$.
%distinct from the surrogate used in the certification (\ie out of the certified parameter set $\hat{\Theta}$).
This setting aligns with real-world cases in which unauthorized classifiers are distinct from the surrogates used in the certification.

We particularly focus on the robustness of UEs by showing some hardness results of recovery attacks against classifiers trained on PUE, EMN, and OPS, following common training routines.
Furthermore, we evaluate the performance of PUE in defending against various training approaches to verify its practicality.
Specifically, we compare the empirical clean test accuracy of unauthorized classifiers trained on PUE, EMN and OPS using four representative training strategies (\ie MixUp, CutOut, Fast Autoaugment, and Adversarial Training).
The results are deferred to Appendix~\ref{append:exp:emprirical_results}.

%%%%%%%%%%%%%%%%%%%%%%%%%%%%%%%%%%%%%
\begin{figure}[t]
     \centering
     \begin{subfigure}
         \centering
         \includegraphics[width=.51\columnwidth]{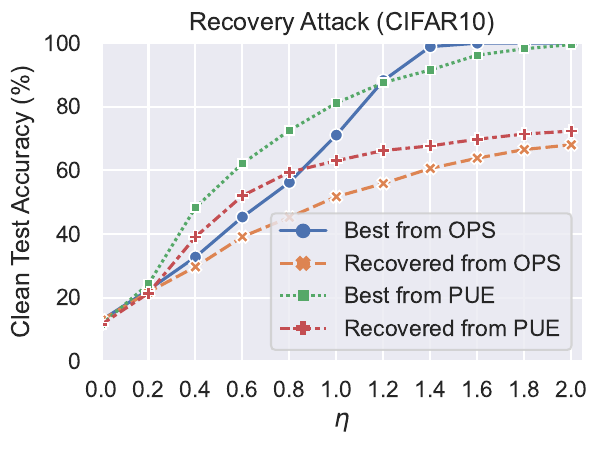}    
     \end{subfigure}\hspace{-3mm}%
     \begin{subfigure}
         \centering
         \includegraphics[width=0.51\columnwidth]{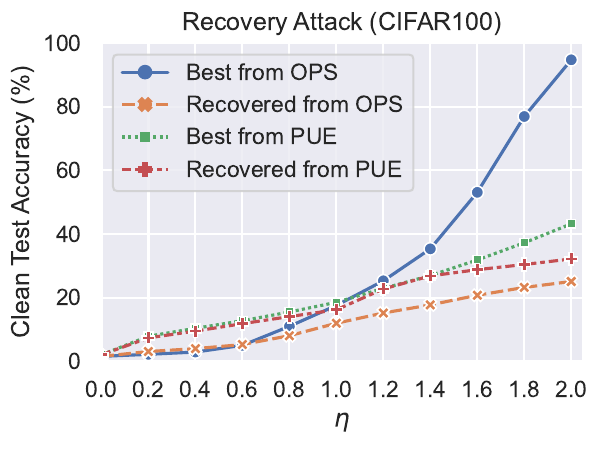}
     \end{subfigure}
\caption{The robustness of PUE and OPS against recovery attacks.}
\label{fig:recovery_attack_common_ops}
\end{figure}
%%%%%%%%%%%%%%%%%%%%%%%%%%%%%%%%%%%%%

%%%%%%%%%%%%%%%%%%%%%%%%%%%%%%%%%%%%%
\begin{figure*}[t]
     \centering
     \begin{subfigure}
         \centering
         \includegraphics[width=0.28\textwidth]{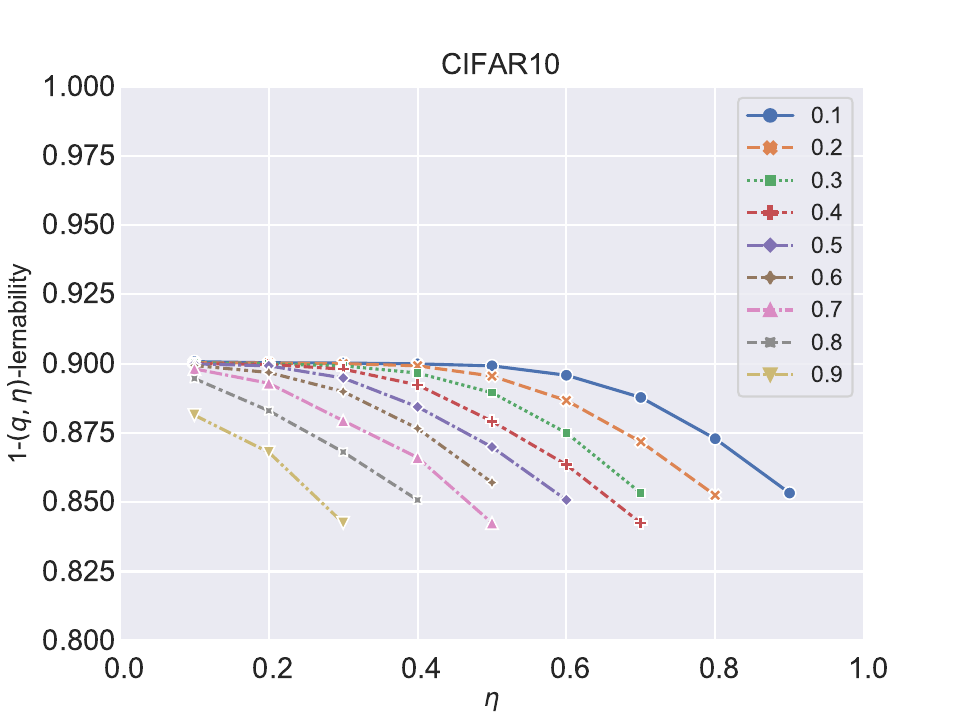}    
     \end{subfigure}
     \begin{subfigure}
         \centering
         \includegraphics[width=0.28\textwidth]{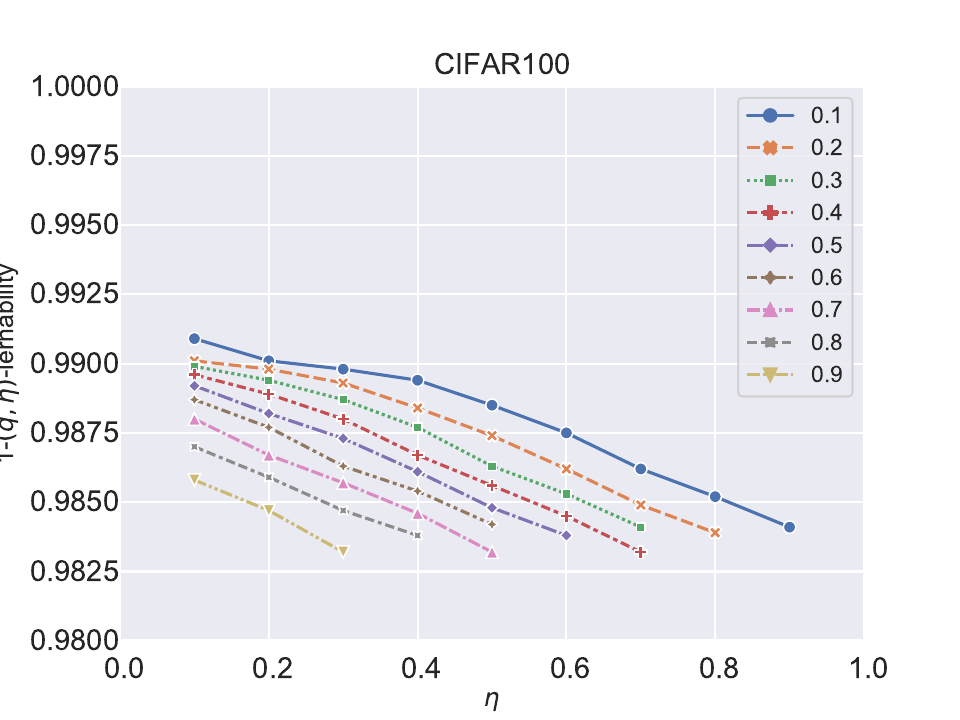}
     \end{subfigure}
     \begin{subfigure}
         \centering
         \includegraphics[width=0.28\textwidth]{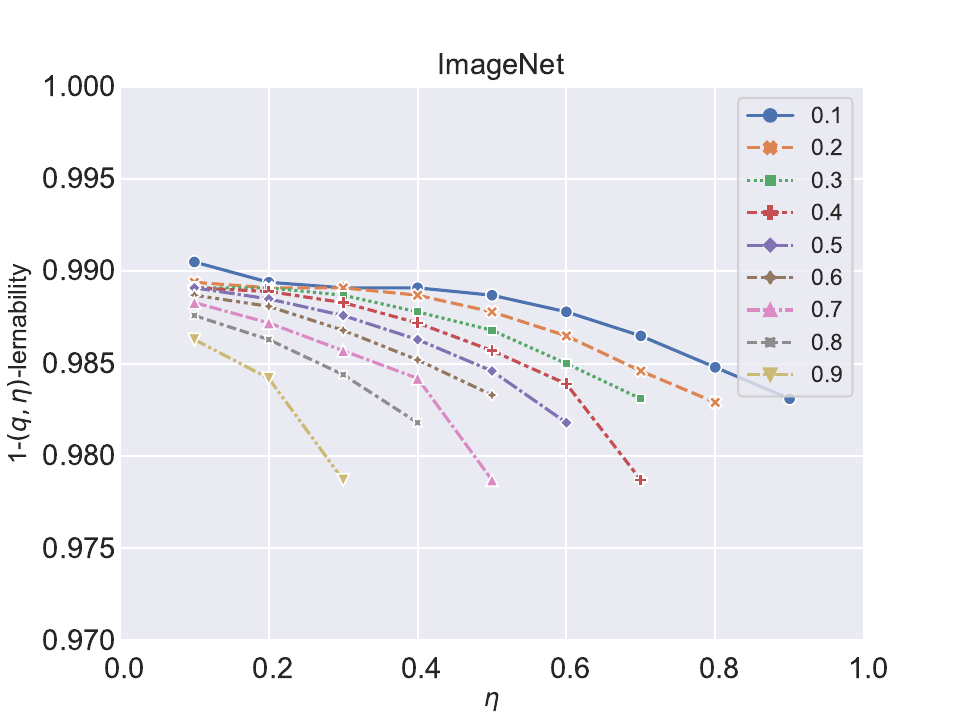}	 
     \end{subfigure}
\caption{The $(q, \eta)$-Learnability scores under different values of $q$ are plotted on three data distributions (\ie CIFAR10, CIFAR100, and ImageNet). At the same $\eta$, a greater probability $q$ can certify a higher $(q, \eta)$-Learnability.}
\label{fig:learnability_under_diff_p}
\end{figure*}
%%%%%%%%%%%%%%%%%%%%%%%%%%%%%%%%%%%%%

%%%%%%%%%%%%%%%%%%%%%%%%%%%%%%%%%%%%%%%%%%%%%%%%%%%%%%%%%%%%%%%%%%%%%%%%%%%%
%\vspace{1mm}
\noindent\textbf{Hardness results of recovery attacks.~}
Based on the same rules of gauging the surrogates using recovery attacks, we measure the performance of recovery attacks on classifiers trained on PUE, EMN, and OPS.
The classifiers are trained on CIFAR10/CIFAR100 with PAP noises. 
We recover the clean accuracy using $20\%$ of clean training data and calculate the best possible accuracy by training the poisoned classifiers on the clean test set. 
Note that adversaries possessing $20\%$ of clean data have a strong capability to restore the performance of the poisoned classifier, evaluating PAP noises under this setting can better reveal their robustness.
Moreover, data augmentations (\ie random crop, flip, and rotation) are employed in all recovery attacks to align with training strategies in real-world scenarios.

The comparison between EMN and PUE is presented in Figure~\ref{fig:recovery_attack_common}.
Compared to the results in Figure~\ref{fig:recovery_attack}, the clean test accuracy of the normally trained classifiers can be better restored.
On CIFAR10, the poisoned classifiers of PUE and EMN have initial accuracy scores of $11.62\%$ and $11.89\%$, respectively.
After the recovery, the EMN classifier has $78.55\%$ accuracy and the PUE one has $72.37\%$ accuracy.
On the side of the best possible accuracy, EMN can achieve $100\%$ and PUE reaches $99.55\%$.
Similarly, the EMN classifier can be restored to $43.01\%$ accuracy on CIFAR100, while the PUE classifier accuracy achieves $32.19\%$.
However, the classifier trained on PUE demonstrates better robustness than the EMN classifier in both the best possible case and the recovery attack case.
Most accuracy can be recovered within an $\eta$ of $1.0$, especially for CIFAR10.
This is also why we certify at most an $\eta$ of $1.0$ in Section~\ref{subsec:cert}.

Aside from the above comparison, we notice several intriguing phenomena of PUE and OPS in Figure~\ref{fig:recovery_attack_common_ops}.
OPS is a surrogate-free PAP noise. 
It searches for the pixel leading to the most significant change in the input data distribution, for each class in the dataset.
Therefore, OPS is supposed to have higher robustness since it is model-agnostic and can be generalized to different classifiers based on experiments. 
It can be observed from the figure that, though OPS has a lower empirically recovered accuracy than PUE, its best recoverable accuracy gets higher than that of PUE rapidly when $\eta$ grows. 
This implies that OPS has remarkable local robustness, possibly due to the effective manipulation of the input data distribution, which makes it hard to find proper clean training samples to reduce the clean generalization error of the trained classifier. 
However, when a set of plausible training samples is selected (\eg the Best from OPS case) for the recovery attack, OPS can be more vulnerable than other PAPs.
Specifically, PUEs reduce at most $54.4\%$ of the accuracy on CIFAR100.
The weakness becomes more obvious when the perturbations/uncertainty in the weights increase beyond a certain threshold.

The comparisons suggest that training GAs on PUE can lead to more robust poisoning results, especially against best-case recovery attacks.
We attribute this to two reasons.
First, theoretically, $\hat{\Theta}$ occupies a non-zero probability measure on every infinitesimal region on $\R^d$.
Therefore, these empirically trained classifiers have non-zero probabilities of being included in $\hat{\Theta}$.
In practice, a larger certified parametric radius $\eta$ helps increase the probability measure at classifiers far from the surrogate.
Second, the loss function \ref{eq:minmin_train} encourages PUE to reduce the accuracy of a collection of classifiers with Gaussian-modeled uncertainty in their parameters.
Consequently, the collective robustness of PUE towards classifiers showing uncertainties in their parameters is enhanced.

%\vspace{-1.5mm}
%%%%%%%%%%%%%%%%%%%%%%%%%%%%%%%%%%%%%%%%%%%%%%%%%%%%%%%%%%%%%%%%%%%%%%%%%%%%
\subsection{Ablation Study}\label{subsec:ablation_study}
%%%%%%%%%%%%%%%%%%%%%%%%%%%%%%%%%%%%%%%%%%%%%%%%%%%%%%%%%%%%%%%%%%%%%%%%%%%%
%\vspace{1mm}
%\noindent\textbf{Properties of the certification.~}
We conduct an ablation study on the trade-off between the STD $\sigma$ of the smoothing noise and the learnability scores over different factors. 
First, we evaluate how the train-time noise during PUE surrogate optimization affects the trade-off.
We train PUE-B surrogates on the CIFAR10/CIFAR100 datasets under different train-time noises and measure their certified learnability under smoothing noises with various $\sigma$ values.
We fix the value of $q$ at $0.9$ in all cases.
The results on CIFAR10 and CIFAR100 are in Table~\ref{table:sigma_learnability_tradoff_cifar10} and Table~\ref{table:sigma_learnability_tradoff_cifar100} of Appendix~\ref{append:ex_setup}, respectively.
The certification returns \textsc{Abstain} when a valid binomial confidence interval cannot be found. 
The abstentions are marked by ``$-$'' in the tables.
The highest certified learnability under each $\eta$ is highlighted.

The numbers in Table~\ref{table:sigma_learnability_tradoff_cifar10} provide some intriguing observations about the trade-off.
First, the largest certifiable parametric radius $\eta$ positively correlates with $\sigma$.
However, large $\sigma$ values result in decreases in the $(q, \eta)$-Learnability scores at small certified radii, which indicates that the negative impact of the smoothing noise towards the tightness gap is significant in these cases.
On the train-time noise side, generally, greater train-time noises make the resulting surrogates more endurable to large smoothing noises in the certification.
Thus, the certification performs better in large-$\sigma$ cases and produces better average certified learnability over different $\eta$ values.
Generally, when the train-time noise matches the $\sigma$ value, the highest certified learnability scores can be found.
For example, certifying using $\sigma=0.25$ obtained the highest $(q, \eta)$-Learnability score when the STD of the train-time noise is $0.25$.
We observe similar trends in the CIFAR100 certification.
This phenomenon suggests that the train-time noise should match the parametric smoothing noise to reduce the tightness gap.

Next, we evaluated the $(q, \eta)$-Learnability under different $q$ values to see how the certified learnability can be affected. 
We select $0.25$ as the STD for the train time noise and the parametric smoothing noise in this study.
The learnability scores are illustrated in Figure~\ref{fig:learnability_under_diff_p}.
From the figure, certifying at a higher probability $q$ compromises the maximal certifiable parametric radius across all datasets.  
However, the learnability scores certified at the same radius with larger $q$ values are usually larger than those certified with small $q$ values, which reveals tighter guarantees on the best cases unauthorized classifiers can achieve.
At smaller parametric radii, the certified learnability is less distinguishable among different $q$ values. 
Nevertheless, the discrepancy becomes obvious when $\eta$ grows. 
In this paper, to ensure the $(q, \eta)$-Learnability can be certified at reasonable $\eta$ values and be held with a sound probability, we select $q=0.9$ in the certification experiments.
%We also include additional experiments on the computational overhead and surrogate architectures in Appendix~\ref{append:ex_setup}.

% We also tested the impact of the number $U$ of randomization during training the UEs.
% The results are presented in Table~\ref{table:ensemble_learnability_tradoff}.
% \begin{table}[htb]
% \caption{Impact of Train-time Randomization Times on Learnability}
% \label{table:ensemble_learnability_tradoff}
% \centering
% \resizebox{.85\linewidth}{!}{%
% \begin{tabular}{ccccccc}
% \toprule
% \multirow{2}{*}{Randomization Times ($U$)} & \multicolumn{6}{c}{$\ell_2$-Radius ($\eta$)}\\
%     & 0.2 & 0.4 & 0.6 & 0.8 & 1.0 & 1.2 \\
% \midrule
% 20   &  &  &  &  &  & \\
% 60  &  &  &  &  &  &  \\
% 100  &  &  &  &  &  &  \\
% \bottomrule
% \end{tabular}
% }
% \end{table}

% \noindent \textbf{Effectiveness under purification.~}
% \begin{table}[htb]
% \caption{Anti-learning Performance under Different Training Strategies}
% \label{table:cert_acc_cifar10}
% \centering
% \resizebox{.7\linewidth}{!}{%
% \begin{tabular}{cccccc}
% \toprule
% Defense & PUE &  EMN & RUE & SEP & OPS\\
% \midrule
% Mixup &  &  &  &  & \\
% Cutout &  &  &  &  & \\
% RandAugment &  &  &  &  & \\
% $\ell_{\infty}$ Adversarial Training  &  &  &  &  & \\
% ISS &  &  &  &  & \\
% \bottomrule
% \end{tabular}
% }
% \end{table}

%%%%%%%%%%%%%%%%%%%%%%%%%%%%%%%%%%%%%%%%%%%%%%%%%%%%%%%%%%%%%%%%%%%%%%%%%%%%%%%%%%%%%%%%%%%%%%%%%%%%%%%%%%%%%%%%%%%%%%%%%%
\section{Discussion and Limitation}
%The certified $(q, \eta)$-Learnability, as a guaranteed upper bound of the clean-task accuracy, can be used as a measurement of the protection level offered by UEs.
% The two core tasks in this paper are 1) certifying tighter $(q, \eta)$-Learnability and 2) designing PAP noises to minimize the certified $(q, \eta)$-Learnability.
%Obviously, when the gap between the $(q, \eta)$-Learnability and the True Learnability shrinks, the certification becomes tighter and a more accurate protection measurement for UEs can be guaranteed therefrom.
In the following part, we will discuss the limitations of our methods and explore potential avenues for improvement.

%%%%%%%%%%%%%%%%%%%%%%%%%%%%%%%%%%%%%%%%%%%%%%%%%%%%%%%%%%%%%%%%%%%%%%%%%%%%
%\vspace{1mm}
\noindent \textbf{Certified parameter set.~}
A crucial step for improving the certificate is to increase the certified parametric radius $\eta$ such that the certificate can cover unauthorized classifiers with more diverse parameters.
Based on the theory and the experiments, we can find a negative correlation between $q$ and the maximum of certifiable $\eta$.
This trade-off suggests that confident certifications can only be made at small $\eta$.
While thoroughly breaking this trade-off seems impossible, it is possible to improve $\eta$ by meticulously designing the parametric smoothing noise and the surrogate training method.
In Theorem~\ref{theorem:cert_learnability}, the certification abstains if $\overline{q}$ is too large such that a valid Binomial confidence interval cannot be found.
To avoid the abstention at large $\eta$, we should use a substantial $\sigma$ since $\eta$ scales with $\sigma$. 
Next, to make such a substantial $\sigma$ functioning, we should have a surrogate classifier that can tolerate a significant level of smoothing noise while not degrading the tightness of the certification.
Through experiments, we verify that larger parametric smoothing noise indeed leads to a broader certifiable parametric radius. 
Nevertheless, increasing the parametric radius $\eta$ cannot be solely resorted to enlarging the smoothing noise since larger noise can also compromise the tightness of the $(q, \eta)$-Learnability.
Specifically, too-large noises negatively impact the testing accuracy of the classifiers randomized from an insufficiently trained surrogate and make them fail to obtain plausible clean test accuracy scores.
We propose a baseline solution in this paper to alleviate the negative effects of the parametric noise, but we reckon further endeavors can be made to increase $\eta$ at high values of $q$.

%%%%%%%%%%%%%%%%%%%%%%%%%%%%%%%%%%%%%%%%%%%%%%%%%%%%%%%%%%%%%%%%%%%%%%%%%%%%
%\vspace{1mm}
\noindent \textbf{Tightness gap of the certification.~}
%The certification can benefit from more principled surrogate selection methods.
Given the noticeable connection between certified learnability and certified robustness, a possible way to further improve the tightness of the certification might be using a combination of adversarial training techniques, such as AWP~\cite{wu2020adversarial} and SmoothAdv\cite{salman2019provably}, to find the most critical perturbation direction for the weights while training the surrogate.
In contrast, optimizing the PAP noise against these carefully augmented surrogates may lead to more robust UEs.
Another possible pathway is to select better parametric smoothing noise.
Though this simple uni-variance Gaussian noise used in this paper endorses a tidy form of $(q, \eta)$-Learnability, we suspect the tightness gap can be further reduced by introducing better smoothing noises.
First, the probability density of Gaussian might not be optimal in discovering the best cases of classifiers.
Intuitively, the optimal parametric smoothing noise should have a large probability of drawing parameters that can lead to the highest clean test accuracy.
Second, the weights in different layers of a classifier may respond differently against the noise added to them, rendering uni-variance noises less efficient in surrogate augmentation and certification. 
A better option could be applying various levels of train-time noise in different layers to cope with the layer-wise noise sensitivities and using corresponding parametric smoothing noise in the certification.
%In addition, a rigid relationship between the input data dimensionality and the maximal size of the certifiable parametric set should be revealed to help understand the tightness gap.

%%%%%%%%%%%%%%%%%%%%%%%%%%%%%%%%%%%%%%%%%%%%%%%%%%%%%%%%%%%%%%%%%%%%%%%%%%%%
%\vspace{1mm}
\noindent \textbf{Real-world practicality.~}
Certified $(q, \eta)$-Learnability is a versatile framework that can be applied to different data modalities and tasks, as the certification holds in the parameter space and $A_{\sD}(\hat{\theta})$ can represent any function that maps model parameters to a scalar value.
As a proof of concept, we certify the unlearnable SST2 training set with Error-min-0 modifications~\cite{li2023make} via an LSTM surrogate.
A $(0.9, \eta)$-learnability of $0.61$ can be certified at $\eta=1$ while the accuracy at $\eta=0$ is $0.52$.
Furthermore, the certification framework and PUE generation process may introduce additional computational overhead. 
We benchmark this overhead across various model architectures and datasets, as shown in Tables~\ref{table:online_train_cost} and \ref{table:offline_train_cost}.
The overhead is confined to the certification and PUE generation stages and remains manageable.
To reduce the cost, techniques such as perturbing only specific layers during PUE generation can be explored. 
Importantly, the application of PUE does not impact real-time data streaming, as PAP noise for each data category can be pre-generated and applied to corresponding streaming data.
Finally, PUE requires a high perturbation rate, which means that a small proportion of PUEs in a clean dataset will not hinder the training process for authorized users. 
Legitimate users can also be provided with the PUE perturbations to restore data learnability. 
These features suggest that PUE is compatible with other defenses, such as model watermarking and unlearning, as the functionality of the watermark or retained set can be maintained.

%%%%%%%%%%%%%%%%%%%%%%%%%%%%%%%%%%%%%%%%%%%%%%%%%%%%%%%%%%%%%%%%%%%%%%%%%%%%%%%%%%%%%%%%%%%%%%%%%%%%%%%%%%%%%%%%%%%%%%%%%%
\section{Related work}
%%%%%%%%%%%%%%%%%%%%%%%%%%%%%%%%%%%%%%%%%%%%%%%%%%%%%%%%%%%%%%%%%%%%%%%%%%%%
\noindent\textbf{Perturbative availability poison and UEs.}
PAPs emerged as a response to unauthorized models trained to infer private datasets. 
There were early attempts that generate perturbations dynamically on surrogate models, which is computationally intensive~\cite{shen2019tensorclog,feng2019learning}.
To alleviate the cost, error-minimizing perturbations were introduced to produce unlearnable data samples~\cite{huangunlearnable2021}. 
Subsequently, the training cost of the perturbation was further reduced by using pretrained surrogate classifiers or generalized neural tangent kernels~\cite{tao2021better,fowl2021adversarial,yuan2021neural,chen2022self}.
A series of works also resort to surrogate-free perturbations~\cite{evtimov2021disrupting,yu2022availability,sandoval2022autoregressive,wu2023one}.
In addition, enhancements were proposed to make transferable PAP noises or robust UEs against adversarial training ~\cite{furobust2022,wen2022adversarial,ren2022transferable}.
We also notice some work making efforts to address the uncertainty of UEs~\cite{lu2023exploring,he2024sharpness}.
Particularly, the sharpness-aware method SAPA~\cite{he2024sharpness} leverages weight perturbations to approximate the worst-case model for the poison attack.
However, the perturbation considered in SAPA is model- and data-dependent and may not be effective with variations in the type of loss function. 
Moreover, all the current methods fall short of offering a provable guarantee on the robustness of UEs against uncertainties in learning algorithms and adaptive adversaries in the field.

%%%%%%%%%%%%%%%%%%%%%%%%%%%%%%%%%%%%%%%%%%%%%%%%%%%%%%%%%%%%%%%%%%%%%%%%%%%%
%\vspace{1mm}
\noindent\textbf{Certified robustness via randomized smoothing.~} 
Following previous certified robustness works using differential privacy and R\'enyi divergence~\cite{lecuyer2019certified,li2019certified}, randomized smoothing was first proposed by Cohen et al. as a tool for supplying a tight $\ell_2$ robustness certificate for black-box functions, based on Neyman-Pearson Lemma (NPL)~\cite{cohen2019certified}.
Though randomized smoothing places no restriction on its based classifiers, a line of work aimed at improving the base classifiers to obtain a better trade-off between the certified accuracy and the certified radius~\cite{salman2019provably,zhai2020macer,jeong2020consistency,jeong2021smoothmix,jeong2023confidence}. 
On the other hand, some trials have been made to design more capable smoothing noises, smoothing pipelines, and certification algorithms beyond NPL~\cite{fischer2020certified,li2021tss, hao2022gsmooth,sukenik2022intriguing,cullen2022double,dvijotham2020framework,zhang2020black,salman2020denoised}.
Intriguingly, a series of works discovered that the maximum of certifiable $\ell_p$ ($p>2$) radius shrinks with the increasing input dimensionality~\cite{kumar20bcurse,yang2020randomized,li2022double,shu2024effects}.
Moreover, randomized smoothing has been extended to poisoning defenses\cite{zhang2022bagflip,weber2022rab,kumar2023provable,zhang2023pecan}, object detection~\cite{chiang2020detection}, and watermark verification~\cite{bansal2022certified}.
The works in this area focus on the certified robustness based on worst-case models rather than examining the certified learnability of datasets through the analysis of the best-case models trained on them.

%%%%%%%%%%%%%%%%%%%%%%%%%%%%%%%%%%%%%%%%%%%%%%%%%%%%%%%%%%%%%%%%%%%%%%%%%%%%%%%%%%%%%%%%%%%%%%%%%%%%%%%%%%%%%%%%%%%%%%%%%%
\section{Conclusion}
In this paper, we propose a certification mechanism for deriving the possibly best clean data utility that unauthorized classifiers trained on UEs can achieve.
The certified $(q, \eta)$-Learnability indicates that, with probability at least $q$, the unauthorized classifiers whose weights are from a specific parameter subspace have a guaranteed upper bound on their clean test accuracy.
Moreover, we propose a way to craft provably unlearnable examples that obtain smaller $(q, \eta)$-Learnability scores, compared with the ones crafted by existing methods. 
Our certification mechanism takes the first step towards certifiably robust and provably effective UEs, which can provide a rigid guarantee on the protection level of PAP noises for data availability control.
However, there exists a tightness gap between the certified $(q, \eta)$-Learnability and True Learnability.
Such a tightness gap can be better alleviated in the future to make the certification more tenable.
Crucially, such alleviation may also expand the space of certifiable parameters to cover greater parameter variance resulting from training techniques such as adversarial training.  
Finally, we reckon that future attempts can be made to extend PUE to sample-wise noises and reduce the training cost of making PUEs.

%%%%%%%%%%%%%%%%%%%%%%%%%%%%%%%%%%%%%%%%%%%%%%%%%%%%%%%%%%%%%%%%%%%%%%%%%%%%%%%%%%%%%%%%%%%%%%%%%%%%%%%%%%%%%%%%%%%%%%%%%%%%%%%%%%%%%%%%%%%%
\section*{Acknowledgments}
This work has been supported by the Cyber Security Research Centre Limited whose activities are partially funded by the Australian Government's Cooperative Research Centres Programme. Minhui Xue is supported in part by Australian Research Council (ARC) DP240103068 and in part by CSIRO -- National Science Foundation (US) AI Research Collaboration Program.

%%%%%%%%%%%%%%%%%%%%%%%%%%%%%%%%%%%%%%%%%%%%%%%%%%%%%%%%%%%%%%%%%%%%%%%%%%%%%%%%%%%%%%%%%%%%%%%%%%%%%%%%%%%%%%%%%%%%%%%%%%%%%%%%%%%%%%%%%%%%
%\clearpage
\bibliographystyle{IEEEtran}
\bibliography{advref_abbrev}

%%%%%%%%%%%%%%%%%%%%%%%%%%%%%%%%%%%%%%%%%%%%%%%%%%%%%%%%%%%%%%%%%%%%%%%%%%%%%%%
%%%%%%%%%%%%%%%%%%%%%%%%%%%%%%%%%%%%%%%%%%%%%%%%%%%%%%%%%%%%%%%%%%%%%%%%%%%%%%%
% APPENDIX
%%%%%%%%%%%%%%%%%%%%%%%%%%%%%%%%%%%%%%%%%%%%%%%%%%%%%%%%%%%%%%%%%%%%%%%%%%%%%%%
%%%%%%%%%%%%%%%%%%%%%%%%%%%%%%%%%%%%%%%%%%%%%%%%%%%%%%%%%%%%%%%%%%%%%%%%%%%%%%%

\appendix
% \setcounter{section}{0}
%\renewcommand{\appendixname}{Appendix~\Alph{section}}
%%%%%%%%%%%%%%%%%%%%%%%%%%%%%%%%%%%%%%%%%%%%%%%%%%%%%%%%%%%%%%%%%%%%%%%%%%%%%%%%%%%%%%%%%%%%%%%%%%%%%%%%%%%%%%%%%%%%%%%%%%%%%%%%%%%%%%%%%%%%

%%%%%%%%%%%%%%%%%%%%%%%%%%%%%%%%%%%%%%%%%%%%%%%%%%%%%%%%%%%%%%%%%%%%%%%%%%%%%%%%%%
\subsection{More Algorithmic Details}\label{append:algorithm_details}
%%%%%%%%%%%%%%%%%%%%%%%%%%%%%%%%%%%%%%%%%%%%%%%%%%%%%%%%%%%%%%%%%%%%%%%%%%%%
\noindent\textbf{Computing $(q, \eta)$-Learnability.~} 
We introduce the detailed steps of computing $(q, \eta)$-Learnability.
Please refer to Algorithm~\ref{alg:quantile_estimate} and Algorithm~\ref{alg:cert} for the pseudocode.

First, according to Theorem~\ref{theorem:cert_learnability}, $\overline{q} = \Phi( \Phi^{-1}(q) + \frac{\eta}{\sigma})$ can be theoretically computed by supplying the values of $q$, $\eta$, and $\sigma$.
Next, we sample a finite number of $A_{\sD}(\hat{\theta}+\epsilon)$ using Monte Carlo.
For simplicity, we denote the sampled and sorted $A_{\sD}(\hat{\theta}+\epsilon)$ as a set $a = \{a_1, a_2, ..., a_n\}$, where $a_1\leq a_2 \cdots\leq a_n$.
Based on the order statistics $a$, the smallest accuracy $a_t\in a$ satisfying $\Pr_{a_i\in a}[a_i \leq a_t]\geq \overline{q}$ can be found as the accuracy in the empirical $\overline{q}$-th quantile (\ie $a_t$ is the empirically calculated value of ${h}_{\overline{q}}(\hat{\theta})$).
Note the $(q, \eta)$-Learnability should be calculated as the confidence interval upper bound of $a_t$.
Calculating a one-sided confidence interval upper bound resorts to finding the $k$-th accuracy in $a$, such that by giving a confidence level $1-\alpha$, there is $\Pr[a_t \leq a_k]\geq 1-\alpha$.
%Similar to the previous method~\cite{chiang2020detection}, there is
There is
\begin{equation*}
%\resizebox{.6\columnwidth}{!} 
    %{
\begin{aligned}
    \Pr[a_t \leq a_k] &= \sum_{i=1}^{k} \Pr[a_{i-1} \leq a_t \leq a_{i}] \\
    &= \sum_{i=1}^k \binom{n}{i} (\overline{q})^i (1-\overline{q})^{n-i}.
\end{aligned}
    %}
\end{equation*}
Thereafter, we just find the smallest $k$ letting 
\begin{equation*}
    \sum_{i=1}^k \binom{n}{i} (\overline{q})^i (1-\overline{q})^{n-i} \geq 1-\alpha. 
\end{equation*}
The $k$-th accuracy $a_k$ is returned as the certified $(q, \eta)$-Learnability $l_{(q,\eta)}(\hat{\Theta}; \sD_s\oplus\delta)$.

%%%%%%%%%%%%%%%%%%%%%%%%%%%%%%%%%%%%%%%%%%%%%%%%%%%%%%%%%%%%%%%%%%%%%%%%%%%%
%\vspace{1mm}
\noindent\textbf{Algorithms of surrogate training and PAP noise optimization.~} 
We attach the details of the surrogate training step and PAP noise optimization step in Algorithm~\ref{alg:train_step} and Algorithm~\ref{alg:opt_step}, respectively. 
The surrogate is updated by using gradient descent with a learning rate of $r$ given $\frac{l}{U_{train}}$ from Algorithm~\ref{alg:delta_opt}.
During optimizing the PAP noise $\delta$, if using class-wise noise, the noise of each class is independently updated by optimizing using examples from the corresponding class.
Specifically, suppose there are $K$ classes of data in the mini-batch, for the $i$-th class, its corresponding noise $\delta[i]$ is updated as 
\begin{equation*}
%\resizebox{1\columnwidth}{!} 
    %{
    \begin{aligned}
         \delta[i] \gets Clip_{\pm\xi}\left( x[i] - r_p \cdot \sign (\frac{1}{U_{perturb}} \bigtriangledown_{\delta[i]} l[i])\right) - x[i],
    \end{aligned}
    %}
\end{equation*}
where $x[i]$ are input samples belonging to the $i$-th class and $l[i]=\sum_{j=1}^{S/s}\ \ L[f_{\hat{\theta}+\epsilon_j}(x[i]), i]$ is the accumulated loss calculated from samples of the $i$-class in the mini-batch.

%%%%%%%%%%%%%%%%%%%%%%%%%%%%%%%%%%%%%
\begin{algorithm}[htb]
\footnotesize
\caption{Training Step of Surrogate}\label{alg:train_step}
\textbf{func} \textsc{Train\_step}\\
\KwIn{$\theta$, $l/U_{train}$, learning rate $r$}
    $\theta \gets \theta - \frac{r}{U_{train}} \cdot \bigtriangledown_{\theta} l$\\
\KwOut{$\theta$}
\end{algorithm}
%%%%%%%%%%%%%%%%%%%%%%%%%%%%%%%%%%%%%

%%%%%%%%%%%%%%%%%%%%%%%%%%%%%%%%%%%%%
\begin{algorithm}[htb]
\footnotesize
\caption{Optimization Step of PUE}\label{alg:opt_step}
\textbf{func} \textsc{Opt\_step}\\
\KwIn{$\delta$, $x$, $l/U_{perturb}$, learning rate $r_p$}
    $\delta \gets Clip_{\pm\xi}\left(x - r_p \cdot \sign (\frac{1}{U_{perturb}} \bigtriangledown_{\delta} l)\right) - x$\\
\KwOut{$\delta$}
\end{algorithm}
%%%%%%%%%%%%%%%%%%%%%%%%%%%%%%%%%%%%%

%%%%%%%%%%%%%%%%%%%%%%%%%%%%%%%%%%%%%
\begin{table*}[t]
\caption{Trade-off between Train-Time/Certification Noise Level and Learnability ($\%$) on CIFAR10}
\label{table:sigma_learnability_tradoff_cifar10}
\centering
\resizebox{.85\textwidth}{!}{%
\begin{tabular}{ccccccccccccccccccccccccc}
%\begin{tabularx}{0.99\linewidth}{XXXXXXXXXXXXXXXXXXXXXXXXX}
\toprule
\multirow{2}{*}{Train Noise} & \multirow{2}{*}{Certification Noise ($\sigma$)} & \multicolumn{23}{c}{$\eta \times 100$} \\
\cmidrule(r){3-25} 
                             &                                      &  0.1  &  0.5  &  1.0  &  5.0  &  10.0 & 15.0  & 20.0  & 25.0  & 30.0  & 35.0  & 40.0  & 45.0  & 50.0  &  55.0 & 60.0  & 65.0  & 70.0  & 75.0  & 80.0  & 85.0  & 90.0  & 95.0  & 100.0\\
\midrule
\multirow{10}{*}{0.10}       &  0.10                                & \textbf{14.57} & \textbf{14.68} & \textbf{14.76} & \textbf{16.00} & \textbf{17.26} &   -   &   -   &   -   &   -   &   -   &   -   &   -   &   -   &   -   &   -   &   -   &   -   &   -   &   -   &   -   &   -   &   -   &   -  \\
                             &  0.20                                & 10.41 & 10.42 & 10.49 & 11.01 & 11.36 & \textbf{12.00} & \textbf{12.66} & \textbf{13.57} &   -   &   -   &   -   &   -   &   -   &   -   &   -   &   -   &   -   &   -   &   -   &   -   &   -   &   -   &   -  \\
                             &  0.30                                & 10.49 & 10.52 & 10.53 & 10.68 & 11.00 & 11.23 & 11.63 & 11.75 & 12.16 & \textbf{12.73} &   -   &   -   &   -   &   -   &   -   &   -   &   -   &   -   &   -   &   -   &   -   &   -   &   -  \\
                             &  0.40                                & 10.51 & 10.52 & 10.52 & 10.72 & 10.92 & 11.14 & 11.50 & 11.70 & 12.04 & 12.23 & \textbf{12.77} & \textbf{12.88} & \textbf{13.49} &   -   &   -   &   -   &   -   &   -   &   -   &   -   &   -   &   -   &   -  \\
                             &  0.50                                & 10.67 & 10.69 & 10.69 & 10.81 & 11.12 & 11.24 & 11.62 & 12.01 & \textbf{12.28} & 12.44 & 12.60 & 12.75 & 12.89 & 13.07 & 13.11 & 13.11 &   -   &   -   &   -   &   -   &   -   &   -   &   -  \\
                             &  0.60                                & 10.66 & 10.67 & 10.68 & 10.86 & 10.95 & 11.03 & 11.14 & 11.38 & 11.65 & 12.08 & 12.38 & 12.63 & 12.98 & \textbf{13.07} & \textbf{13.45} & \textbf{13.46} & \textbf{13.75} & \textbf{14.80} &   -   &   -   &   -   &   -   &   -  \\
                             &  0.70                                & 10.51 & 10.51 & 10.53 & 10.60 & 10.74 & 10.82 & 10.95 & 11.12 & 11.19 & 11.32 & 11.45 & 11.63 & 11.70 & 11.80 & 11.99 & 12.28 & 12.57 & 12.77 & \textbf{13.67} & \textbf{13.70} & \textbf{13.70} &   -   &   -  \\
                             &  0.80                                & 10.49 & 10.49 & 10.50 & 10.54 & 10.62 & 10.69 & 10.75 & 10.90 & 10.99 & 11.06 & 11.24 & 11.34 & 11.41 & 11.47 & 11.66 & 11.76 & 11.88 & 12.00 & 12.43 & 12.54 & 12.95 & \textbf{12.95} & \textbf{13.18}\\
                             &  0.90                                & 10.50 & 10.50 & 10.53 & 10.56 & 10.62 & 10.70 & 10.78 & 10.91 & 10.98 & 11.06 & 11.18 & 11.29 & 11.35 & 11.56 & 11.77 & 11.94 & 12.06 & 12.29 & 12.33 & 12.58 & 12.58 & 12.81 & 13.01\\
                             &  1.00                                & 10.49 & 10.49 & 10.50 & 10.54 & 10.61 & 10.67 & 10.74 & 10.82 & 10.93 & 11.03 & 11.12 & 11.28 & 11.43 & 11.71 & 11.80 & 11.96 & 12.12 & 12.15 & 12.17 & 12.22 & 12.36 & 12.50 & 12.60\\
\midrule
\multirow{11}{*}{0.15}       &  0.10                                & \textbf{13.75} & \textbf{13.85} & \textbf{14.04} & \textbf{15.17} & \textbf{16.59} &   -   &   -   &   -   &   -   &   -   &   -   &   -   &   -   &   -   &   -   &   -   &   -   &   -   &   -   &   -   &   -   &   -   &   -  \\
                             &  0.15                                & 10.53 & 10.60 & 10.62 & 11.26 & 11.99 & \textbf{13.67} &   -   &   -   &   -   &   -   &   -   &   -   &   -   &   -   &   -   &   -   &   -   &   -   &   -   &   -   &   -   &   -   &   -  \\
                             &  0.20                                & 10.57 & 10.59 & 10.65 & 10.99 & 11.54 & 12.07 & \textbf{13.14} & \textbf{15.20} &   -   &   -   &   -   &   -   &   -   &   -   &   -   &   -   &   -   &   -   &   -   &   -   &   -   &   -   &   -  \\
                             &  0.30                                & 10.49 & 10.49 & 10.54 & 10.73 & 11.03 & 11.39 & 11.82 & 13.00 & \textbf{13.43} & \textbf{13.67} &   -   &   -   &   -   &   -   &   -   &   -   &   -   &   -   &   -   &   -   &   -   &   -   &   -  \\
                             &  0.40                                & 10.47 & 10.50 & 10.51 & 10.69 & 10.87 & 11.01 & 11.26 & 11.48 & 11.87 & 12.36 & \textbf{12.55} & \textbf{13.10} & \textbf{13.92} &   -   &   -   &   -   &   -   &   -   &   -   &   -   &   -   &   -   &   -  \\
                             &  0.50                                & 10.44 & 10.44 & 10.47 & 10.58 & 10.69 & 10.82 & 11.09 & 11.39 & 11.61 & 12.28 & 12.34 & 12.61 & 12.78 & \textbf{13.12} & \textbf{13.72} & \textbf{13.72} &   -   &   -   &   -   &   -   &   -   &   -   &   -  \\
                             &  0.60                                & 10.45 & 10.50 & 10.53 & 10.62 & 10.74 & 10.87 & 11.04 & 11.19 & 11.46 & 11.51 & 11,65 & 11.80 & 12.11 & 12.30 & 12.60 & 12.63 & \textbf{13.55} & \textbf{14.10} &   -   &   -   &   -   &   -   &   -  \\
                             &  0.70                                & 10.40 & 10.40 & 10.42 & 10.49 & 10.57 & 10.66 & 10.78 & 10.96 & 11.05 & 11.20 & 11.32 & 11.43 & 11.63 & 11.78 & 12.04 & 12.39 & 12.58 & 12.72 & \textbf{13.03} & \textbf{13.47} & \textbf{13.47} &   -   &   -  \\
                             &  0.80                                & 10.38 & 10.38 & 10.38 & 10.42 & 10.53 & 10.62 & 10.68 & 10.79 & 10.92 & 11.03 & 11.20 & 11.36 & 11.52 & 11.74 & 12.28 & 12.33 & 12.48 & 12.49 & 12.70 & 13.16 & 13.16 & \textbf{13.16} & \textbf{13.74}\\
                             &  0.90                                & 10.42 & 10.42 & 10.43 & 10.51 & 10.56 & 10.62 & 10.71 & 10.78 & 10.92 & 10.97 & 11.03 & 11.16 & 11.19 & 11.25 & 11.39 & 11.67 & 11.81 & 12.06 & 12.12 & 12.17 & 12.20 & 12.38 & 12.78\\
                             &  1.00                                & 10.35 & 10.35 & 10.35 & 10.42 & 10.50 & 10.64 & 10.73 & 10.76 & 10.87 & 10.97 & 11.04 & 11.17 & 11.24 & 11.28 & 11.38 & 11.48 & 11.53 & 11.63 & 11.75 & 12.14 & 12.38 & 12.74 & 12.80\\
\midrule
\multirow{10}{*}{0.20}       &  0.10                                & 10.58 & 10.60 & 10.63 & 10.96 & \textbf{13.70} &   -   &   -   &   -   &   -   &   -   &   -   &   -   &   -   &   -   &   -   &   -   &   -   &   -   &   -   &   -   &   -   &   -   &   -  \\
                             &  0.20                                & \textbf{11.70} & \textbf{11.89} & \textbf{11.98} & \textbf{12.58} & 13.38 & \textbf{14.21} & \textbf{15.02} & \textbf{15.29} &   -   &   -   &   -   &   -   &   -   &   -   &   -   &   -   &   -   &   -   &   -   &   -   &   -   &   -   &   -  \\
                             &  0.30                                & 10.57 & 10.57 & 10.61 & 10.84 & 11.15 & 11.61 & 11.81 & 12.24 & \textbf{12.77} & \textbf{12.95} &   -   &   -   &   -   &   -   &   -   &   -   &   -   &   -   &   -   &   -   &   -   &   -   &   -  \\
                             &  0.40                                & 10.42 & 10.42 & 10.44 & 10.65 & 10.90 & 11.10 & 11.22 & 11.31 & 11.81 & 11.96 & \textbf{12.18} & \textbf{13.07} & \textbf{13.18} &   -   &   -   &   -   &   -   &   -   &   -   &   -   &   -   &   -   &   -  \\
                             &  0.50                                & 10.46 & 10.46 & 10.46 & 10.57 & 10.71 & 10.86 & 11.07 & 11.27 & 11.31 & 11.40 & 11.51 & 11.73 & 12.42 & \textbf{12.59} & \textbf{14.54} & \textbf{14.54} &   -   &   -   &   -   &   -   &   -   &   -   &   -  \\
                             &  0.60                                & 10.42 & 10.44 & 10.46 & 10.55 & 10.64 & 10.77 & 10.89 & 10.96 & 11.09 & 11.18 & 11.34 & 11.46 & 11.68 & 12.02 & 12.33 & 12.38 & \textbf{12.96} & \textbf{13.11} &   -   &   -   &   -   &   -   &   -  \\
                             &  0.70                                & 10.34 & 10.36 & 10.37 & 10.46 & 10.57 & 10.61 & 10.68 & 10.80 & 10.96 & 11.13 & 11.28 & 11.50 & 11.61 & 11.70 & 11.87 & 12.06 & 12.36 & 12.70 & \textbf{13.00} & \textbf{14.18} & \textbf{14.18} &   -   &   -  \\
                             &  0.80                                & 10.40 & 10.40 & 10.44 & 10.47 & 10.61 & 10.73 & 10.76 & 10.91 & 10.97 & 11.08 & 11.22 & 11.30 & 11.47 & 11.61 & 12.04 & 12.17 & 12.20 & 12.57 & 12.74 & 13.08 & 13.19 & 13.19 & 13.23\\
                             &  0.90                                & 10.45 & 10.45 & 10.46 & 10.55 & 10.58 & 10.62 & 10.72 & 10.80 & 10.90 & 11.00 & 11.05 & 11.18 & 11.31 & 11.40 & 11.64 & 11.72 & 11.99 & 12.08 & 12.19 & 12.49 & 12.73 & \textbf{13.30} & \textbf{13.48}\\
                             &  1.00                                & 10.43 & 10.43 & 10.44 & 10.49 & 10.54 & 10.61 & 10.67 & 10.78 & 10.92 & 10.99 & 11.04 & 11.08 & 11.23 & 11.34 & 11.36 & 11.44 & 11.54 & 11.73 & 11.81 & 11.99 & 12.44 & 12.61 & 12.76\\
\midrule
\multirow{11}{*}{0.25}       &  0.10                                &  9.96 &  9.96 &  9.96 & 10.00 & 10.98 &   -   &   -   &   -   &   -   &   -   &   -   &   -   &   -   &   -   &   -   &   -   &   -   &   -   &   -   &   -   &   -   &   -   &   -  \\
                             &  0.20                                & 10.02 & 10.24 & 10.27 & 10.58 & 10.98 & 11.85 & 12.69 & 14.49 &   -   &   -   &   -   &   -   &   -   &   -   &   -   &   -   &   -   &   -   &   -   &   -   &   -   &   -   &   -  \\
                             &  0.25                                & \textbf{10.62} & \textbf{10.67} & \textbf{10.71} & \textbf{11.07} & \textbf{11.86} & \textbf{12.05} & \textbf{13.20} & \textbf{14.67} & \textbf{15.75} &   -   &   -   &   -   &   -   &   -   &   -   &   -   &   -   &   -   &   -   &   -   &   -   &   -   &   -  \\
                             &  0.30                                & 10.35 & 10.35 & 10.40 & 10.69 & 10.98 & 11.32 & 11.98 & 12.57 & 12.95 & \textbf{13.02} &   -   &   -   &   -   &   -   &   -   &   -   &   -   &   -   &   -   &   -   &   -   &   -   &   -  \\
                             &  0.40                                & 10.41 & 10.43 & 10.44 & 10.63 & 10.86 & 11.15 & 11.28 & 11.57 & 12.58 & 12.82 & \textbf{13.03} & \textbf{13.18} & \textbf{14.21} &   -   &   -   &   -   &   -   &   -   &   -   &   -   &   -   &   -   &   -  \\
                             &  0.50                                & 10.40 & 10.41 & 10.42 & 10.53 & 10.73 & 10.88 & 11.17 & 11.45 & 11.60 & 11.76 & 11.92 & 12.22 & 12.73 & \textbf{12.76} & \textbf{14.00} & \textbf{14.00} &   -   &   -   &   -   &   -   &   -   &   -   &   -  \\
                             &  0.60                                & 10.40 & 10.40 & 10.41 & 10.50 & 10.61 & 10.73 & 10.83 & 10.98 & 11.15 & 11.41 & 11.51 & 11.76 & 11.93 & 11.99 & 12.49 & 12.76 & \textbf{13.08} & \textbf{15.25} &   -   &   -   &   -   &   -   &   -  \\
                             &  0.70                                & 10.47 & 10.50 & 10.52 & 10.57 & 10.72 & 10.83 & 10.96 & 11.02 & 11.14 & 11.19 & 11.31 & 11.55 & 11.77 & 11.90 & 12.13 & 12.19 & 12.38 & 12.55 & \textbf{12.95} & \textbf{15.70} & \textbf{15.70} &   -   &   -  \\
                             &  0.80                                & 10.53 & 10.53 & 10.53 & 10.59 & 10.68 & 10.77 & 10.86 & 10.95 & 11.18 & 11.28 & 11.37 & 11.49 & 11.52 & 11.83 & 12.00 & 12.12 & 12.71 & 12.72 & 12.87 & 12.91 & 13.72 & \textbf{13.72} & \textbf{15.17}\\
                             &  0.90                                & 10.52 & 10.52 & 10.53 & 10.59 & 10.66 & 10.77 & 10.88 & 10.96 & 11.09 & 11.27 & 11.40 & 11.47 & 11.51 & 11.52 & 11.85 & 12.21 & 12.23 & 12.51 & 12.90 & 13.10 & 13.36 & 13.62 & 13.96\\
                             &  1.00                                & 10.50 & 10.50 & 10.51 & 10.56 & 10.66 & 10.75 & 10.78 & 10.82 & 10.96 & 11.07 & 11.17 & 11.26 & 11.51 & 11.57 & 11.62 & 11.71 & 11.82 & 12.12 & 12.24 & 12.42 & 12.67 & 13.10 & 13.18\\
\bottomrule
%\end{tabularx}
\end{tabular}
}
\end{table*}
%%%%%%%%%%%%%%%%%%%%%%%%%%%%%%%%%%%%%

%%%%%%%%%%%%%%%%%%%%%%%%%%%%%%%%%%%%%
\begin{table*}[t]
\caption{Trade-off between Train-Time/Certification Noise Level and Learnability ($\%$) on CIFAR100}
\label{table:sigma_learnability_tradoff_cifar100}
\centering
\resizebox{.8\linewidth}{!}{%
\begin{tabular}{ccccccccccccccccccccccccc}
\toprule
\multirow{2}{*}{Train Noise} & \multirow{2}{*}{Certification Noise ($\sigma$)} & \multicolumn{23}{c}{ $\eta \times 100$} \\
\cmidrule(r){3-25} 
                             &                                      &  0.1  &  0.5  &  1.0  &  5.0  &  10.0 & 15.0  & 20.0  & 25.0  & 30.0  & 35.0  & 40.0  & 45.0  & 50.0  & 55.0  & 60.0  & 65.0  & 70.0  & 75.0  & 80.0  & 85.0  & 90.0  & 95.0  & 100.0\\
\midrule
\multirow{10}{*}{0.10}       &  0.10                                &  \textbf{1.81} &  \textbf{1.82} &  \textbf{1.84} &  \textbf{1.97} &  \textbf{2.27} &   -   &   -   &   -   &   -   &   -   &   -   &   -   &   -   &   -   &   -   &   -   &   -   &   -   &   -   &   -   &   -   &   -   &   -  \\
                             &  0.20                                &  1.17 &  1.17 &  1.17 &  1.21 &  1.32 &  \textbf{1.40} &  \textbf{1.49} &  \textbf{1.60} &   -   &   -   &   -   &   -   &   -   &   -   &   -   &   -   &   -   &   -   &   -   &   -   &   -   &   -   &   -  \\
                             &  0.30                                &  1.13 &  1.13 &  1.13 &  1.17 &  1.22 &  1.25 &  1.33 &  1.40 &  \textbf{1.44} &  \textbf{1.49} &   -   &   -   &   -   &   -   &   -   &   -   &   -   &   -   &   -   &   -   &   -   &   -   &   -  \\
                             &  0.40                                &  1.13 &  1.13 &  1.13 &  1.15 &  1.17 &  1.19 &  1.23 &  1.28 &  1.35 &  1.41 &  \textbf{1.50} &  \textbf{1.52} &  \textbf{1.60} &   -   &   -   &   -   &   -   &   -   &   -   &   -   &   -   &   -   &   -  \\
                             &  0.50                                &  1.12 &  1.13 &  1.13 &  1.14 &  1.16 &  1.18 &  1.20 &  1.23 &  1.25 &  1.33 &  1.34 &  1.35 &  1.37 &  1.37 &  \textbf{1.57} &  1.57 &   -   &   -   &   -   &   -   &   -   &   -   &   -  \\
                             &  0.60                                &  1.12 &  1.12 &  1.12 &  1.14 &  1.16 &  1.19 &  1.21 &  1.24 &  1.26 &  1.30 &  1.32 &  1.38 &  1.40 &  \textbf{1.46} &  1.48 &  \textbf{1.61} &  \textbf{1.70} &  \textbf{1.74} &   -   &   -   &   -   &   -   &   -  \\
                             &  0.70                                &  1.12 &  1.12 &  1.12 &  1.13 &  1.15 &  1.16 &  1.18 &  1.19 &  1.21 &  1.24 &  1.25 &  1.28 &  1.30 &  1.32 &  1.34 &  1.35 &  1.42 &  1.43 &  \textbf{1.58} &  \textbf{1.66} &  \textbf{1.66} &   -   &   -  \\
                             &  0.80                                &  1.12 &  1.12 &  1.12 &  1.13 &  1.16 &  1.16 &  1.18 &  1.19 &  1.20 &  1.22 &  1.25 &  1.27 &  1.28 &  1.33 &  1.34 &  1.35 &  1.38 &  1.38 &  1.42 &  1.42 &  1.48 &  \textbf{1.48} &  \textbf{1.60}\\
                             &  0.90                                &  1.11 &  1.11 &  1.11 &  1.12 &  1.12 &  1.13 &  1.14 &  1.15 &  1.17 &  1.17 &  1.19 &  1.20 &  1.21 &  1.22 &  1.22 &  1.25 &  1.25 &  1.30 &  1.33 &  1.39 &  1.41 &  1.44 &  1.56\\
                             &  1.00                                &  1.11 &  1.11 &  1.11 &  1.11 &  1.12 &  1.13 &  1.14 &  1.15 &  1.16 &  1.17 &  1.17 &  1.19 &  1.19 &  1.21 &  1.23 &  1.23 &  1.26 &  1.28 &  1.31 &  1.32 &  1.34 &  1.36 &  1.36\\
\midrule
\multirow{10}{*}{0.15}       &  0.10                                &  \textbf{1.85} &  \textbf{1.86} &  \textbf{1.88} &  \textbf{2.01} &  \textbf{2.33} &   -   &   -   &   -   &   -   &   -   &   -   &   -   &   -   &   -   &   -   &   -   &   -   &   -   &   -   &   -   &   -   &   -   &   -  \\
                             &  0.15                                &  1.61 &  1.18 &  1.19 &  1.27 &  1.39 &  \textbf{1.51} &   -   &   -   &   -   &   -   &   -   &   -   &   -   &   -   &   -   &   -   &   -   &   -   &   -   &   -   &   -   &   -   &   -  \\
                             &  0.20                                &  1.13 &  1.13 &  1.14 &  1.19 &  1.25 &  1.37 &  \textbf{1.45} &  \textbf{1.66} &   -   &   -   &   -   &   -   &   -   &   -   &   -   &   -   &   -   &   -   &   -   &   -   &   -   &   -   &   -  \\
                             &  0.30                                &  1.12 &  1.13 &  1.13 &  1.15 &  1.21 &  1.25 &  1.33 &  1.40 &  \textbf{1.48} &  \textbf{1.50} &   -   &   -   &   -   &   -   &   -   &   -   &   -   &   -   &   -   &   -   &   -   &   -   &   -  \\
                             &  0.40                                &  1.13 &  1.13 &  1.13 &  1.16 &  1.19 &  1.24 &  1.30 &  1.32 &  1.40 &  1.43 &  \textbf{1.46} &  \textbf{1.49} &  \textbf{1.52} &   -   &   -   &   -   &   -   &   -   &   -   &   -   &   -   &   -   &   -  \\
                             &  0.50                                &  1.13 &  1.13 &  1.13 &  1.14 &  1.16 &  1.18 &  1.19 &  1.21 &  1.24 &  1.26 &  1.28 &  1.30 &  1.37 &  \textbf{1.42} &  \textbf{1.59} &  \textbf{1.59} &   -   &   -   &   -   &   -   &   -   &   -   &   -  \\
                             &  0.60                                &  1.12 &  1.12 &  1.12 &  1.14 &  1.15 &  1.15 &  1.19 &  1.20 &  1.22 &  1.25 &  1.27 &  1.30 &  1.33 &  1.35 &  1.38 &  1.43 &  1.44 &  \textbf{1.52} &   -   &   -   &   -   &   -   &   -  \\
                             &  0.70                                &  1.12 &  1.12 &  1.12 &  1.13 &  1.14 &  1.15 &  1.17 &  1.18 &  1.21 &  1.25 &  1.26 &  1.28 &  1.35 &  1.40 &  1.41 &  1.43 &  \textbf{1.45} &  1.49 &  \textbf{1.54} &  \textbf{1.74} &  \textbf{1.74} &   -   &   -  \\
                             &  0.80                                &  1.12 &  1.12 &  1.12 &  1.13 &  1.14 &  1.15 &  1.16 &  1.17 &  1.19 &  1.20 &  1.22 &  1.25 &  1.26 &  1.27 &  1.30 &  1.31 &  1.31 &  1.32 &  1.35 &  1.39 &  1.41 &  \textbf{1.41} &  \textbf{1.48}\\
                             &  0.90                                &  1.11 &  1.11 &  1.11 &  1.11 &  1.12 &  1.13 &  1.15 &  1.16 &  1.18 &  1.19 &  1.21 &  1.22 &  1.24 &  1.25 &  1.27 &  1.31 &  1.32 &  1.34 &  1.34 &  1.37 &  1.39 &  1.39 &  1.46\\
                             &  1.00                                &  1.11 &  1.11 &  1.11 &  1.11 &  1.12 &  1.13 &  1.14 &  1.14 &  1.15 &  1.16 &  1.18 &  1.20 &  1.21 &  1.22 &  1.24 &  1.26 &  1.28 &  1.32 &  1.32 &  1.33 &  1.34 &  1.35 &  1.41\\
\midrule
\multirow{10}{*}{0.20}       &  0.10                                &  1.24 &  1.24 &  1.25 &  1.33 &  1.41 &   -   &   -   &   -   &   -   &   -   &   -   &   -   &   -   &   -   &   -   &   -   &   -   &   -   &   -   &   -   &   -   &   -   &   -  \\
                             &  0.20                                &  \textbf{1.51} &  \textbf{1.52} &  \textbf{1.53} &  \textbf{1.59} &  \textbf{1.70} &  \textbf{1.75} &  \textbf{1.81} &  \textbf{1.86} &   -   &   -   &   -   &   -   &   -   &   -   &   -   &   -   &   -   &   -   &   -   &   -   &   -   &   -   &   -  \\
                             &  0.30                                &  1.14 &  1.14 &  1.15 &  1.18 &  1.21 &  1.24 &  1.28 &  1.38 &  \textbf{1.45} &  \textbf{1.51} &   -   &   -   &   -   &   -   &   -   &   -   &   -   &   -   &   -   &   -   &   -   &   -   &   -  \\
                             &  0.40                                &  1.14 &  1.14 &  1.15 &  1.18 &  1.19 &  1.23 &  1.25 &  1.30 &  1.37 &  1.43 &  \textbf{1.49} &  \textbf{1.54} &  \textbf{1.68} &   -   &   -   &   -   &   -   &   -   &   -   &   -   &   -   &   -   &   -  \\
                             &  0.50                                &  1.12 &  1.12 &  1.13 &  1.14 &  1.16 &  1.20 &  1.23 &  1.27 &  1.28 &  1.31 &  1.32 &  1.39 &  1.44 &  \textbf{1.50} &  \textbf{1.70} &  \textbf{1.70} &   -   &   -   &   -   &   -   &   -   &   -   &   -  \\
                             &  0.60                                &  1.12 &  1.12 &  1.12 &  1.14 &  1.16 &  1.17 &  1.19 &  1.21 &  1.24 &  1.25 &  1.30 &  1.30 &  1.31 &  1.36 &  1.39 &  1.47 &  \textbf{1.55} &  \textbf{1.56} &   -   &   -   &   -   &   -   &   -  \\
                             &  0.70                                &  1.11 &  1.11 &  1.11 &  1.12 &  1.14 &  1.16 &  1.17 &  1.18 &  1.20 &  1.23 &  1.24 &  1.26 &  1.28 &  1.32 &  1.37 &  1.38 &  1.40 &  1.43 &  \textbf{1.46} &  \textbf{1.49} &  \textbf{1.49} &   -   &   -  \\
                             &  0.80                                &  1.12 &  1.12 &  1.12 &  1.12 &  1.14 &  1.16 &  1.18 &  1.20 &  1.21 &  1.23 &  1.24 &  1.27 &  1.30 &  1.31 &  1.33 &  1.33 &  1.34 &  1.36 &  1.38 &  1.39 &  1.39 &  1.39 &  \textbf{1.47}\\
                             &  0.90                                &  1.13 &  1.13 &  1.13 &  1.13 &  1.15 &  1.16 &  1.17 &  1.18 &  1.19 &  1.19 &  1.20 &  1.21 &  1.22 &  1.25 &  1.25 &  1.28 &  1.28 &  1.32 &  1.32 &  1.35 &  1.40 &  \textbf{1.43} &  1.44\\
                             &  1.00                                &  1.12 &  1.12 &  1.12 &  1.13 &  1.14 &  1.16 &  1.16 &  1.17 &  1.17 &  1.18 &  1.19 &  1.20 &  1.22 &  1.23 &  1.25 &  1.26 &  1.28 &  1.30 &  1.31 &  1.35 &  1.37 &  1.40 &  1.40\\
\midrule
\multirow{11}{*}{0.25}       &  0.10                                &  1.12 &  1.12 &  1.13 &  1.19 &  1.31 &   -   &   -   &   -   &   -   &   -   &   -   &   -   &   -   &   -   &   -   &   -   &   -   &   -   &   -   &   -   &   -   &   -   &   -  \\
                             &  0.20                                &  1.27 &  1.27 &  1.28 &  1.31 &  1.39 &  1.47 &  \textbf{1.60} &  \textbf{1.69} &   -   &   -   &   -   &   -   &   -   &   -   &   -   &   -   &   -   &   -   &   -   &   -   &   -   &   -   &   -  \\
                             &  0.25                                &  \textbf{1.32} &  \textbf{1.32} &  \textbf{1.33} &  \textbf{1.37} &  \textbf{1.42} &  \textbf{1.47} &  1.53 &  1.59 &  \textbf{1.68} &   -   &   -   &   -   &   -   &   -   &   -   &   -   &   -   &   -   &   -   &   -   &   -   &   -   &   -  \\
                             &  0.30                                &  1.18 &  1.18 &  1.19 &  1.23 &  1.28 &  1.32 &  1.35 &  1.52 &  1.56 &  \textbf{1.61} &   -   &   -   &   -   &   -   &   -   &   -   &   -   &   -   &   -   &   -   &   -   &   -   &   -  \\
                             &  0.40                                &  1.12 &  1.12 &  1.12 &  1.14 &  1.16 &  1.19 &  1.22 &  1.30 &  1.34 &  1.40 &  \textbf{1.44} &  \textbf{1.47} &  \textbf{1.58} &   -   &   -   &   -   &   -   &   -   &   -   &   -   &   -   &   -   &   -  \\
                             &  0.50                                &  1.12 &  1.12 &  1.12 &  1.15 &  1.17 &  1.19 &  1.21 &  1.23 &  1.24 &  1.25 &  1.28 &  1.35 &  1.41 &  \textbf{1.42} &  \textbf{1.56} &  \textbf{1.56} &   -   &   -   &   -   &   -   &   -   &   -   &   -  \\
                             &  0.60                                &  1.10 &  1.10 &  1.10 &  1.11 &  1.13 &  1.16 &  1.18 &  1.21 &  1.22 &  1.23 &  1.27 &  1.33 &  1.36 &  1.38 &  1.44 &  1.49 &  \textbf{1.51} &  \textbf{1.52} &   -   &   -   &   -   &   -   &   -  \\
                             &  0.70                                &  1.10 &  1.11 &  1.11 &  1.11 &  1.13 &  1.15 &  1.17 &  1.19 &  1.20 &  1.21 &  1.23 &  1.25 &  1.27 &  1.28 &  1.31 &  1.31 &  1.38 &  1.40 &  \textbf{1.50} &  \textbf{1.52} &  \textbf{1.52} &   -   &   -  \\
                             &  0.80                                &  1.10 &  1.10 &  1.10 &  1.12 &  1.13 &  1.13 &  1.15 &  1.15 &  1.16 &  1.18 &  1.20 &  1.22 &  1.23 &  1.25 &  1.27 &  1.28 &  1.31 &  1.34 &  1.36 &  1.38 &  1.51 &  \textbf{1.51} &  \textbf{1.56}\\
                             &  0.90                                &  1.10 &  1.10 &  1.10 &  1.11 &  1.13 &  1.14 &  1.15 &  1.16 &  1.17 &  1.18 &  1.19 &  1.20 &  1.21 &  1.21 &  1.23 &  1.24 &  1.25 &  1.26 &  1.28 &  1.30 &  1.30 &  1.34 &  1.34\\
                             &  1.00                                &  1.10 &  1.10 &  1.11 &  1.11 &  1.13 &  1.13 &  1.14 &  1.16 &  1.19 &  1.20 &  1.21 &  1.22 &  1.23 &  1.25 &  1.26 &  1.27 &  1.30 &  1.31 &  1.33 &  1.33 &  1.34 &  1.37 &  1.38\\
\bottomrule
\end{tabular}
}
\end{table*}
%%%%%%%%%%%%%%%%%%%%%%%%%%%%%%%%%%%%%

%%%%%%%%%%%%%%%%%%%%%%%%%%%%%%%%%%%%%
\begin{table}[t]
\caption{Certification Offset}
\label{table:acc_corrections}
\centering
\resizebox{.95\linewidth}{!}{%
\begin{tabular}{ccccccccc}
\toprule
\multirow{2}{*}{Data} & Accuracy (Online) &    \multicolumn{3}{c}{$C$ (Online)}     & Accuracy (Offline) &  \multicolumn{2}{c}{$C$ (Offline)} \\
\cmidrule{2-8}                                                           
                      &  {PUE-B}          &  PUE-10  &  PUE-1  &  EMN    &     PUE-10         &    EMN     &     OPS   \\
\midrule
      CIFAR10         &   9.96            &   -0.68  &  -0.12  & -4.72   &      10.32         &    +0.94   &    -0.28  \\
      CIFAR100        &   1.00            &   -0.02  &  -0.10  & -0.68   &       0.97         &    +0.01   &    N/A    \\
      ImageNet        &   1.24            &   +0.15  &  +0.15  & +0.15   &       N/A          &     N/A    &    N/A    \\
\bottomrule
\end{tabular}
}
\end{table}
%%%%%%%%%%%%%%%%%%%%%%%%%%%%%%%%%%%%%

%%%%%%%%%%%%%%%%%%%%%%%%%%%%%%%%%%%%%
\begin{table}[t]
\caption{Certified $(q,\eta)$-Learnability based on Different Surrogate Architectures ($\%$, $\sigma=0.25$)}
\label{table:cert_with_diff_arc_0.25}
\centering
\resizebox{.95\linewidth}{!}{%
\begin{tabular}{cccccccccc}
\toprule
\multirow{2}{*}{Architecture}  & \multicolumn{9}{c}{$\eta\times100$}\\
\cmidrule(r){2-10}
                          &  0.1  &  0.5  &  1.0  &  5.0  &  10.0 &  15.0 &  20.0 &  25.0 &  30.0  \\
\midrule
ResNet-18                 & \textbf{10.57} & \textbf{10.58} & \textbf{10.58} & \textbf{10.83} & \textbf{11.10} & \textbf{11.39} & \textbf{11.68} & \textbf{12.22} & \textbf{13.20} \\
ResNet-50                 &  7.86 &  7.87 &  7.87 &  8.00 &  8.15 &  8.48 &  8.91 &  9.95 & 10.97 \\
\bottomrule
\end{tabular}
}
\end{table}
%%%%%%%%%%%%%%%%%%%%%%%%%%%%%%%%%%%%%

%\onecolumn
%%%%%%%%%%%%%%%%%%%%%%%%%%%%%%%%%%%%%%%%%%%%%%%%%%%%%%%%%%%%%%%%%%%%%%%%%%%%%%%%%%%%%%%%%%%%%%%%%%%%%%%%%%%%%%%%%%%%%%%%%%
\subsection{Proofs}\label{append:proof}
\lcert*
\begin{proof}
The proof is similar to that of Lemma 2 in the paper of Chiang et al.~\cite{chiang2020detection}.
For simplicity, notate the RHS of the inequality by ${h}_{\overline{q}}(\hat{\theta})$.
Let 
\begin{equation*}
    \lambda(\hat{\theta}) = \E_{\epsilon\sim\N(0,\sigma^2I)} \mathds{1}[A_{\sD}(\hat{\theta}+\epsilon)\leq {h}_{\overline{q}}(\hat{\theta})],
\end{equation*}
there is $\lambda(\hat{\theta}) \in [0,1]$.
It is obvious that
\begin{equation*}
    \lambda(\hat{\theta}) = \Pr[A_{\sD}(\hat{\theta}+\epsilon)\leq {h}_{\overline{q}}(\hat{\theta})].
\end{equation*}
The following function
\begin{equation*}
\begin{aligned}
     \Lambda(\hat{\theta}) &= \sigma \Phi^{-1}(\lambda(\hat{\theta})) \\
     &= \sigma \Phi^{-1}(\Pr[A_{\sD}(\hat{\theta}+\epsilon)\leq {h}_{\overline{q}}(\hat{\theta})])
\end{aligned}
\end{equation*}
is 1-Lipschiz due to Lemma 2 of Salman et al.~\cite{salman2019provably}.
Therefore,
\begin{equation}\label{eq:ineq1}
\begin{aligned}
         & \Phi^{-1}(\Pr[A_{\sD}(\hat{\theta}+\upsilon+\epsilon) \leq {h}_{\overline{q}}(\hat{\theta})]) \\
    \geq & \Phi^{-1}(\Pr[A_{\sD}(\hat{\theta}+\epsilon) \leq {h}_{\overline{q}}(\hat{\theta})]) - \frac{\|\upsilon\|}{\sigma}.
\end{aligned}
\end{equation}
If and only if when $\|\upsilon\| \leq \eta$, there is
\begin{equation}\label{eq:ineq2}
\begin{aligned}
    & \Phi^{-1}(\Pr[A_{\sD}(\hat{\theta}+\epsilon)\leq {h}_{\overline{q}}(\hat{\theta})]) - \frac{\|\upsilon\|}{\sigma} \\
    \geq & \Phi^{-1}(\Pr[A_{\sD}(\hat{\theta}+\epsilon)\leq {h}_{\overline{q}}(\hat{\theta})]) - \frac{\eta}{\sigma} \\
       = & \Phi^{-1}(\overline{q}) - \frac{\eta}{\sigma} \\
       = & \Phi^{-1}({q}).
\end{aligned}
\end{equation}
Since the inverse CDF is monotonically increasing, combining Inequality~\ref{eq:ineq1} and Inequality~\ref{eq:ineq2} yields
\begin{equation*}
    \Pr[A_{\sD}(\hat{\theta}+\upsilon+\epsilon)\leq {h}_{\overline{q}}(\hat{\theta})] \geq q.
\end{equation*}
Note that, by definition, 
\begin{equation*}
{h}_{{q}}(\hat{\theta}+\upsilon) = \inf\ \{t\ |\ \Pr[A_{\sD}(\hat{\theta}+\upsilon+\epsilon) \leq t] \geq {q}\},
\end{equation*}
Then
\begin{equation*}
\begin{aligned}
    {h}_{{q}}(\hat{\theta}+\upsilon) &\leq {h}_{\overline{q}}(\hat{\theta}) \\
 \iff   {h}_{{q}}(\hat{\theta}+\upsilon) &\leq \inf\ \{t\ |\ \Pr[A_{\sD}(\hat{\theta}+\epsilon) \leq t] \geq \overline{q}\}.
\end{aligned}
\end{equation*}
The proof is concluded.
\end{proof}

\generr*
\begin{proof}
Given the input space $\gX:=\{ x\in\R^d\ |\ \sum_{i=1}^{d} x^2_i \leq \tau \}$, a classifier $f_{\hat{\theta}}:\R^d\rightarrow \R^k$, an accuracy $A_{\gD}(\cdot)$ over data distribution $\gD$, and an empirical accuracy $A_{\hat{\sD}}(\cdot)$ calculated from a test dataset $\hat{\sD}$ of size $N$.
Notice that the generalization error of $f_{\hat{\theta}}$ is
\begin{equation}
    err_{\gD}(\hat{\theta}) = \E_{(x,y)\sim\gD} \mathds{1}[f_{\hat{\theta}}^{(y)}(x) \leq \max_{i:i\neq y} f_{\hat{\theta}}^{(i)}(x)],
\end{equation}
where $f_{\hat{\theta}}^{(i)}(x)$ is the $i$-th classification score of $f_{\hat{\theta}}(x)$. 
Similarly, the empirical error on $\hat{\sD}$ is
\begin{equation}
    err_{\hat{\sD}}(\hat{\theta}) = \frac{1}{N} \sum_{j=1}^{N}\mathds{1}[f_{\hat{\theta}}^{(y_j)}(x_j) \leq \max_{i:i\neq y_j} f_{\hat{\theta}}^{(i)}(x_j)].
\end{equation}
When the classifier's weights are randomly perturbed by $\epsilon\sim\pi(0)$, according to Langford \& Shawe-Taylor, and McAllester~\cite{mcallester2003simplified}, for any $\alpha > 0$:
\begin{equation}
    \resizebox{1\columnwidth}{!} 
    {%
    $
    \begin{aligned}
        & \Pr\left( \E_{\epsilon}[err_{\gD}(\hat{\theta}+\epsilon)] \leq \E_{\epsilon}[err_{\hat{\sD}}(\hat{\theta}+\epsilon)] + \sqrt{\frac{D_{KL}(\pi(\hat{\theta}+\epsilon)||P) + \log\frac{N}{\alpha}}{2(N-1)}}\right) \\
    & \geq 1 - \alpha,
    \end{aligned}
    $
    }
\end{equation}
where $P$ is a prior distribution independent from training data, $\pi(\hat{\theta}+\upsilon)$ is the distribution of $\hat{\theta}+\upsilon$, and $D_{KL}$ measures the Kullback–Leibler (KL) Divergence.
Therefore, with probability at least $1-\alpha$,
\begin{equation}\label{eq:mc_bound}
%\resizebox{0.8\columnwidth}{!} 
    %{
    \begin{aligned}
              & \E_{\epsilon} [{err}_{\gD}(\hat{\theta}+\epsilon)] \\
         \leq & \E_{\epsilon} [err_{\hat{\sD}}(\hat{\theta}+\epsilon)] + \sqrt{\frac{D_{KL}(\pi(\hat{\theta}+\epsilon) || P) + \log \frac{N}{\alpha}}{2(N-1)}}.
    \end{aligned}
    %}
\end{equation}
Let $P$ be $\N(0, \sigma^2I)$, and by knowing $\pi(\hat{\theta}+\epsilon) = \N(\hat{\theta}, \sigma^2I)$, the KL Divergence can be computed as
\begin{equation}\label{eq:kld_upper_bound}
%\resizebox{0.8\columnwidth}{!} 
    %{
    \begin{aligned}
        D_{KL}(\pi(\hat{\theta}+\epsilon) || P) &= \frac{\|\hat{\theta}\|^2}{2\sigma^2} + \frac{|\hat{\theta}|}{2} \left[ \frac{\sigma^2}{\sigma^2} + \log (\frac{\sigma^2}{\sigma^2}) - 1 \right] \\
        & = \frac{\|\hat{\theta}\|^2}{2\sigma^2},
    \end{aligned}
    %}
\end{equation}
where $|\hat{\theta}|$ is the size of $\hat{\theta}$.
Plugging Equation~\ref{eq:kld_upper_bound} and 
\begin{equation*}
\begin{aligned}
    & A_{\gD}(\hat{\theta}+\epsilon) = 1 - err_{\gD}(\hat{\theta}+\epsilon) \\
    & A_{\hat{\sD}}(\hat{\theta}+\epsilon) = 1 - err_{\hat{\sD}}(\hat{\theta}+\epsilon) 
\end{aligned}
\end{equation*}
into Inequality~\ref{eq:mc_bound} concludes the proof.
\end{proof}

%%%%%%%%%%%%%%%%%%%%%%%%%%%%%%%%%%%%%%%%%%%%%%%%%%%%%%%%%%%%%%%%%%%%%%%%%%%%%%%%%%
\subsection{More Experimental Settings and Results}\label{append:ex_setup}
%%%%%%%%%%%%%%%%%%%%%%%%%%%%%%%%%%%%%%%%%%%%%%%%%%%%%%%%%%%%%%%%%%%%%%%%%%%%
%\subsubsection{Experimental Settings}
\noindent\textbf{Additional details of hyper-parameters.~}
For all experiments, we set the value $\xi=8/255$ for CIFAR10 and CIFAR100. 
On ImageNet, we set $\xi=16/255$.
The training batch size is $128$ for all classifiers except DenseNet121, where it is set to $96$.
The step size of gradient descent during optimizing the noise $\delta$ is set to $\xi/10$.
We set $U_{train}=100$ for all training with random weight perturbations.
In all experiments of crafting PUE/PUE-B/EMN, the value of $M$ is set to $10$ for CIFAR10, $20$ for CIFAR100, and $100$ for ImageNet.
The stop criteria for the PUE/PUE-B/EMN optimization is that the validation error drops below $10\%$.
We use $20\%$ of the training set to optimize the PAP noise and validate the error of the online surrogate. 
The online surrogate is trained on the full training set with a poisoning rate of $1.0$.
In offline surrogate training, we use the fully poisoned training datasets for training and validation.
The offline surrogate training is stopped when the training error is less than $10\%$ and the surrogate has gone through at least one epoch of training with random weight perturbations.

%%%%%%%%%%%%%%%%%%%%%%%%%%%%%%%%%%%%%%%%%%%%%%%%%%%%%%%%%%%%%%%%%%%%%%%%%%%%
%\vspace{1mm}
\noindent\textbf{Comparison of $(q, \eta)$-Learnability.~}
Given a surrogate with parameters $\hat{\theta}$ and a corresponding $\hat{\Theta}$, we are interested in the accuracy gain $l_{(q,\eta)}(\hat{\Theta}; \sD_s\oplus\delta) - A_{\sD}(\hat{\theta})$ resulting from perturbing $\hat{\theta}$. 
When comparing two surrogates trained by different methods, the one inducing a higher accuracy gain through certification can locate a certified parameter set on which a tighter certification can be made. 
Regarding surrogates trained on different PAP noises using the same training methods, their accuracy gains reflect the robustness of the noises towards recovery attacks or uncertainties in parameters. 
Therefore, we adjust the learnability scores in the comparisons to better demonstrate the differences in accuracy gains.
In the online setting, let the accuracy of the PUE-B surrogate be $A_{\sD}(\theta_{PUE-B})$, we offset the certified learnability of $X$ $\in\{$PUE-10, PUE-1, EMN$\}$ by adding a constant $C = A_{\sD}(\theta_{PUE-B}) - A_{\sD}(\theta_{X})$.
In the offline setting, we use PUE-10 as the baseline and accordingly offset certification results from EMN and OPS.
The offsets are listed in Table~\ref{table:acc_corrections}.

%%%%%%%%%%%%%%%%%%%%%%%%%%%%%%%%%%%%%%%%%%%%%%%%%%%%%%%%%%%%%%%%%%%%%%%%%%%%
%\subsubsection{More Ablation Study}\label{append:exp:ablation}
%%%%%%%%%%%%%%%%%%%%%%%%%%%%%%%%%%%%%%%%%%%%%%%%%%%%%%%%%%%%%%%%%%%%%%%%%%%%
%\vspace{1mm}
\noindent\textbf{Ablation study of certification.~}
The ablation study on CIFAR10 and CIFAR100 is presented in Table~\ref{table:sigma_learnability_tradoff_cifar10} and Table~\ref{table:sigma_learnability_tradoff_cifar100}, respectively.
We can observe similar trends in the results on the two datasets.
Generally, when the train-time noise matches the parametric smoothing noise, higher $(q, \eta)$-Learnability can be certified.
A train-time noise level of $0.25$ can obtain higher $(q, \eta)$-Learnability certified at farther $\eta$.

%%%%%%%%%%%%%%%%%%%%%%%%%%%%%%%%%%%%%%%%%%%%%%%%%%%%%%%%%%%%%%%%%%%%%%%%%%%%
%\vspace{1mm}
\noindent\textbf{Cross-architecture certification.~}
We train offline surrogates based on ResNet-18 and ResNet-50 architectures on the PUE version of CIFAR10, respectively.
The CIFAR10 PUEs are crafted based on ResNet-18.
Therefore, the ResNet-50 surrogate examines a cross-architecture certification performance.
We compare the certification results from these two architectures in Table~\ref{table:cert_with_diff_arc_0.25}.
We set $\sigma=0.25$ in the certification.
Similar to the previous comparisons, we use ResNet-18 as the baseline and apply offsets (-2.26) to the $(q, \eta)$-Learnability scores certified by ResNet-50.

The results show that ResNet-18 certifies tighter $(q, \eta)$-Learnability scores.
The results suggest that the surrogate used in certification should match the surrogate used for generating UEs.
Note that ResNet-50 has more parameters, therefore the certifications are carried out on parameter spaces with different dimensionality.
On top of that, we also notice that ResNet-50 is less capable of finding better classifiers through parametric randomization.
Better surrogate training methods could be proposed in the future to efficiently augment surrogates with more parameters.

%%%%%%%%%%%%%%%%%%%%%%%%%%%%%%%%%%%%%
\begin{table}[t]
\caption{Empirical Top-1 Clean Testing Accuracy ($\%$) of Classifiers Trained on UEs}
\label{table:emp_acc}
\centering
\resizebox{.65\linewidth}{!}{%
\begin{tabular}{cccccc}
\toprule
Data &      Model   &  PUE  &   EMN  &  OPS  &  Clean \\
\midrule
\multirow{3}{*}{CIFAR10} 
     & ResNet-18    & \textbf{10.62} &  11.89 & 12.95 &  94.90 \\
     & ResNet-50    & \textbf{10.00} &  12.37 & 10.96 &  94.17 \\
     & DenseNet-121 & \textbf{10.31} &  11.86 & 13.80 &  95.16 \\
\midrule
\multirow{3}{*}{CIFAR100} 
     & ResNet-18    &  2.62 &   2.49 &  \textbf{1.57} &  70.78 \\
     & ResNet-50    &  2.03 &   \textbf{1.61} &  1.78 &  70.89 \\
     & DenseNet-121 &  3.01 &   2.56 &  \textbf{2.32} &  95.16 \\
\midrule
\multirow{3}{*}{ImageNet} 
     & ResNet-18    &  \textbf{4.56} &   4.84 & 22.32 &  79.86 \\
     & ResNet-50    &  \textbf{3.34} &   4.00 & 26.62 &  82.36 \\
     & DenseNet-121 &  \textbf{4.60} &   6.23 & 22.60 &  84.04 \\
\bottomrule
\end{tabular}
}
\end{table}
%%%%%%%%%%%%%%%%%%%%%%%%%%%%%%%%%%%%%

%%%%%%%%%%%%%%%%%%%%%%%%%%%%%%%%%%%%%
\begin{table}[htb]
\caption{Impact of Training Strategies}
\label{table:emp_acc_with_stratagies}
\centering
\resizebox{.55\linewidth}{!}{%
\begin{tabular}{ccccc}
\toprule
\multirow{2}{*}{Strategy}           & \multicolumn{4}{c}{Data} \\
\cmidrule{2-5}
                                    &  PUE  &  EMN  &  OPS  &  Clean \\
\midrule
MixUp                               & 24.66 & 17.79 & 32.99 &  95.21 \\
CutOut                              & 12.05 & 14.42 & 57.69 &  95.18 \\
Fast Autoaugment                    & 35.33 & 22.71 & 65.61 &  95.14 \\
$\ell_{\infty}$ AT                  & 85.03 & 85.06 & 10.45 &  84.97 \\
\bottomrule
\end{tabular}
}
\end{table}
%%%%%%%%%%%%%%%%%%%%%%%%%%%%%%%%%%%%%

%%%%%%%%%%%%%%%%%%%%%%%%%%%%%%%%%%%%%%%%%%%%%%%%%%%%%%%%%%%%%%%%%%%%%%%%%%%%
\noindent\textbf{Computational overhead.~}
We also tested the effective noise level for training the surrogate and PUEs. 
We found that the surrogate and PUEs could converge rapidly on a Gaussian noise whose $\sigma' \leq 0.25$. 
However, when the noise further increases to $\sigma' \geq 0.3$, the training fails to converge in a limited time.
We record the GPU hours of training PUEs and surrogates under different noise levels, on Nvidia Tesla P100 GPUs (Table~\ref{table:online_train_cost} and Table~\ref{table:offline_train_cost}).
The training cost increases with the scale of the train-time noise. 
Each training run of PUEs and an online surrogate takes around $40$ GPU hours when selecting $0.25$ as the scale of the train-time noise.
Training an offline surrogate with a noise scale of $0.25$ takes around 2.16 GPU hours.
The training cost is within an acceptable range.

%%%%%%%%%%%%%%%%%%%%%%%%%%%%%%%%%%%%%
\begin{table}[t]
\caption{Training Cost (GPU hour) for PUEs and Online Surrogates}
\label{table:online_train_cost}
\centering
\resizebox{.55\linewidth}{!}{%
\begin{tabular}{cccccc}
\toprule
\multirow{2}{*}{Dataset}  & \multicolumn{5}{c}{Noise Level}\\
\cmidrule(r){2-6}
                          &  0.05 & 0.10 & 0.15 & 0.20  & 0.25  \\
\midrule
CIFAR10                   &  0.24 & 1.20 & 1.20 & 19.68 & 39.60 \\
CIFAR100                  &  0.24 & 2.16 & 2.16 & 12.96 & 38.88 \\
ImageNet                  &  0.40 & 3.84 & 5.04 & 23.28 & 45.90 \\
\bottomrule
\end{tabular}
}
\end{table}
%%%%%%%%%%%%%%%%%%%%%%%%%%%%%%%%%%%%%

%%%%%%%%%%%%%%%%%%%%%%%%%%%%%%%%%%%%%
\begin{table}[t]
\caption{Training Cost (GPU hour) for Offline Surrogates}
\label{table:offline_train_cost}
\centering
\resizebox{.4\linewidth}{!}{%
\begin{tabular}{cccc}
\toprule
\multirow{2}{*}{Dataset}  & \multicolumn{3}{c}{PAP Noise}\\
\cmidrule(r){2-4}
                          &  PUE  &  EMN  &  OPS  \\
\midrule
CIFAR10                   &  2.16 &  2.16 &  3.60 \\
CIFAR100                  &  2.16 &  2.16 &  N/A  \\
\bottomrule
\end{tabular}
}
\end{table}
%%%%%%%%%%%%%%%%%%%%%%%%%%%%%%%%%%%%%

%%%%%%%%%%%%%%%%%%%%%%%%%%%%%%%%%%%%%%%%%%%%%%%%%%%%%%%%%%%%%%%%%%%%%%%%%%%%
\subsection{Empirical Evaluations}\label{append:exp:emprirical_results}
In this section, we investigate the empirical utility of UEs.
We conduct empirical evaluations of the PUEs by comparing the clean test accuracy of classifiers trained on PUE, EMN~\cite{huangunlearnable2021}, and OPS~\cite{wu2023one}.
We also employ train-time strategies in our experiments to show the robustness of PUE against such strategies.
However, as discussed in this paper, we want to emphasize that empirical results could be less instructive due to the uncertainty brought by the randomness in training.

%%%%%%%%%%%%%%%%%%%%%%%%%%%%%%%%%%%%%%%%%%%%%%%%%%%%%%%%%%%%%%%%%%%%%%%%%%%%
%\vspace{1mm}
\noindent\textbf{Vanilla training.~}
In this part, we measure the empirical Top-1 accuracy of classifiers trained on PUE-10, EMN, and OPS, respectively.
There is no particular training strategy incorporated at this stage and we stick to the training details mentioned at the beginning of the Experiments section.
We use the default settings from the corresponding papers of EMN and OPS in the comparison, except for ImageNet OPS.
We use Random Resized Crop instead of Center Crop to process ImageNet data to $224\times224$ for PUE, EMN, and OPS.
The results are presented in Table~\ref{table:emp_acc}.
PUE outperforms all competitors on CIFAR10 and ImageNet, and it is on par with EMN and OPS on CIFAR100.
We observe that OPS performs poorly on ImageNet.
The possible reason is that Random Resized Crop diversifies the pixel distribution in each class, rendering the task of finding the most impactful pixel arduous for OPS.

%%%%%%%%%%%%%%%%%%%%%%%%%%%%%%%%%%%%%%%%%%%%%%%%%%%%%%%%%%%%%%%%%%%%%%%%%%%%
%\vspace{1mm}
\noindent\textbf{Training with strategies.~}
Training techniques such as data augmentation and adversarial training are prevalently placed into model training in real-world cases. 
Though the certified learnability considers only the parameters of the trained model and ignores the training techniques, it is also of interest to test how PUEs react to these strategies when the trained classifier is out of the certifiable parameter set. 
In this part, we apply different techniques in the training and compare the robustness of PUE, EMN and OPS towards these techniques.  
Specifically, we tested PUEs along with its competitors on adversarial training (AT)~\cite{madry2017towards}, Mixup~\cite{zhang2018mixup}, Cutout~\cite{devries2017improved}, and Fast Autoaugment~\cite{lim2019fast}.
We use an $\ell_{\infty}$ norm of $8/255$ for PGD in the adversarial training.
The results are in Table~\ref{table:emp_acc_with_stratagies}.
PUE outperforms both EMN and OPS when training using CutOut, and better survives from MixUp and Fast Autoaugment than OPS.
In contrast, OPS is barely affected by adversarial training while PUE and EMN can be invalidated.  

\end{document}